\documentclass[12pt,a4paper]{article}
\usepackage[utf8]{inputenc}
\usepackage[whole]{bxcjkjatype}

\usepackage{authblk}
\usepackage[top=25truemm,bottom=25truemm,left=25truemm,right=25truemm,footskip=10truemm]{geometry}
\usepackage{booktabs}
\usepackage{amsmath}
\usepackage{amsthm}
\usepackage{float}
\usepackage{amssymb}
\usepackage{amsfonts}
\usepackage{ascmac}
\usepackage{scalefnt}
\usepackage{indentfirst}
\usepackage{setspace}
\usepackage{subcaption}
\usepackage{lscape}
\usepackage{natbib}
\usepackage{url}
\usepackage{color}
\usepackage[dvipsnames]{xcolor}
\usepackage{mathtools}
\usepackage{hyperref}
\usepackage{algorithm}
\usepackage{algorithmic}
\usepackage{enumitem}

\theoremstyle{plain}
\newtheorem{theorem}{Theorem}
\newtheorem{proposition}[theorem]{Proposition}
\newtheorem{corollary}[theorem]{Corollary}
\newtheorem{lemma}[theorem]{Lemma}
\newtheorem{assumption}{Assumption}

\newtheorem{remark}{Remark}

\DeclareMathOperator*{\argmax}{arg\,max}

\newcommand{\RR}{\mathbb{R}}
\newcommand{\EE}{\mathbb{E}}
\newcommand{\PP}{\mathbb{P}}
\newcommand{\Var}{\operatorname{Var}}

\newcommand{\sgn}{\operatorname{sgn}}
\newcommand{\KL}{\operatorname{KL}}
\newcommand{\kl}{\operatorname{kl}}

\newcommand{\1}{\boldsymbol{1}}

\newcommand{\cA}{\mathcal{A}}
\newcommand{\cB}{\mathcal{B}}
\newcommand{\cD}{\mathcal{D}}
\newcommand{\cE}{\mathcal{E}}
\newcommand{\cG}{\mathcal{G}}
\newcommand{\cH}{\mathcal{H}}

\newcommand{\cR}{\mathcal{R}}
\newcommand{\cT}{\mathcal{T}}
\newcommand{\cX}{\mathcal{X}}
\newcommand{\uR}{\underline{R}}

\newcommand{\indep}{\mathop{\perp\!\!\!\perp}}

\newcommand{\dd}{\mathop{}\!\mathrm{d}}

\title{
PAC-Bayesian Reward-Certified\\[0.4\baselineskip]
Outcome Weighted Learning
}

\author[1,2,$\dag$]{Yuya Ishikawa}
\author[3,4,*]{Shu Tamano}

\affil[1]{\small{Department of Clinical Biostatistics, Graduate School of Medical and Dental Sciences, Institute of Science Tokyo, 1-5-45 Yushima, Bunkyo-ku, Tokyo, 113-8510, Japan}}
\affil[2]{\small{Biostatistics Division, Center for Research Administration and Support, National Cancer Center, 5-1-1 Tsukiji, Chuo-ku, Tokyo 104-0045, Japan}}

\affil[3]{\small{Department Multidisciplinary Sciences, Graduate School of Arts and Sciences, The University of
Tokyo, 3-8-1 Komaba, Meguro-Ku, Tokyo 153-8902, Japan}}
\affil[4]{\small{Department of Epidemiology, National Institute of Infectious Diseases, Japan Institute for Health Security, 1-23-1 Toyama, Shinjuku-Ku, Tokyo 162-0052, Japan}}

\affil[$\dag$]{\small{E-mail: yuishik2.ncc.go.jp}}

\affil[*]{\small{E-mail: tamano-shu212@g.ecc.u-tokyo.ac.jp}}

\date{}

\begin{document}

\maketitle

\begin{abstract}
    Estimating optimal individualized treatment rules (ITRs) via outcome weighted learning (OWL) often relies on observed rewards that are noisy or optimistic proxies for the true latent utility.
    Ignoring this reward uncertainty leads to the selection of policies with inflated apparent performance, yet existing OWL frameworks lack the finite-sample guarantees required to systematically embed such uncertainty into the learning objective.
    To address this issue, we propose PAC-Bayesian Reward-Certified Outcome Weighted Learning (PROWL).
    Given a one-sided uncertainty certificate, PROWL constructs a conservative reward and a strictly policy-dependent lower bound on the true expected value.
    Theoretically, we prove an exact certified reduction that transforms robust policy learning into a unified, split-free cost-sensitive classification task.
    This formulation enables the derivation of a nonasymptotic PAC-Bayes lower bound for randomized ITRs, where we establish that the optimal posterior maximizing this bound is exactly characterized by a general Bayes update.
    To overcome the learning-rate selection problem inherent in generalized Bayesian inference, we introduce a fully automated, bounds-based calibration procedure, coupled with a Fisher-consistent certified hinge surrogate for efficient optimization.
    Our experiments demonstrate that PROWL achieves improvements in estimating robust, high-value treatment regimes under severe reward uncertainty compared to standard methods for ITR estimation.
\end{abstract}

\noindent%

\section{Introduction}
\label{sec:introduction}

\subsection{Background}
\label{subsec:background}

Individualized treatment rules (ITRs), also called optimal treatment regimes or statistical treatment rules, provide a formal framework for personalized medicine by mapping patient covariates to treatment assignments.
They are particularly useful when treatment effects are heterogeneous and tailored decisions can improve clinical outcomes.
The modern literature combines the causal formulation of treatment regimes developed in biostatistics \citep{murphy2003optimal,robins2004optimal,moodie2007demystifying} with decision-theoretic treatment-rule analysis in econometrics \citep{manski2004statistical,hirano2009asymptotics}.
As emphasized by \citet{kosorok2019precision}, estimating optimal ITRs is often more crucial than merely assessing population average treatment effects whenever interventions are tailored to covariate information.
As reviewed by \citet{lipkovich2024modern}, modern treatment effect heterogeneity analysis is increasingly oriented toward individualized treatment recommendations, subgroup identification, and treatment regime estimation in clinical trials and observational studies.
Under standard identification conditions rooted in the potential-outcome framework \citep{robins1986new,robins2000marginal}, a primary goal is to estimate a rule that maximizes its value, namely the expected utility or outcome that would be observed if the rule were deployed in the target population.

A broad distinction in this literature is between indirect and direct approaches.
Indirect methods estimate conditional mean outcomes or treatment contrasts and then derive a treatment rule by plug-in optimization;
representative examples include Q-learning, A-learning, and related regression-based procedures \citep{qian2011performance,zhang2012robust,schulte2015q}.
Direct methods, in contrast, target the value criterion more explicitly, often by reducing policy learning to classification or empirical welfare maximization \citep{manski2004statistical,zhang2012estimating,Zhao2012,kitagawa2018should,athey2021policy}.
From this perspective, single-stage ITR estimation can be viewed as a special case of offline policy learning or contextual bandits, with treatments as actions and clinical outcomes as rewards \citep{hirano2009asymptotics,sakhi2023pac,gouverneur2025refined}.

Among direct methods, outcome weighted learning (OWL) is particularly attractive because it converts value maximization into a weighted classification problem \citep{Zhao2012}.
Under standard causal assumptions, the value of an ITR admits an inverse-probability-weighted representation, and OWL exploits this representation by assigning classification penalties to treatment assignments associated with better observed outcomes.
This reduction enables the use of large-margin classifiers such as support vector machines for policy learning.
Since it directly targets a clinically relevant value criterion while remaining algorithmically compatible with statistical learning tools, OWL has become one of the most visible templates for direct treatment-regime estimation.
The original framework has since been extended in several directions, including sparse penalization \citep{Song2015}, residual weighting \citep{zhou2017residual}, augmentation and multistage learning \citep{zhao2015new,liu2018augmented}, robust loss design \citep{fu2019robust}, personalized dose finding \citep{chen2016personalized}, treatment selection under composite outcomes such as effectiveness and cost \citep{xu2022estimation}, tree-based rules \citep{zhu2017greedy}, ordinal and multicategory treatments \citep{chen2018estimating,zhang2020multicategory,ma2023learning}, and Bayesian formulations that quantify uncertainty in the learned rule \citep{YazzourhFreeman2024}.
These developments confirm OWL as a versatile and durable foundation for personalized treatment learning.

In practice, however, the observed reward is often only a proxy for the utility that decision makers truly care about.
This discrepancy may arise from measurement error, surrogate endpoints, delayed outcomes, or preference-based aggregation of multiple outcomes into a single scalar reward.
The issue is not merely that the reward is noisy; rather, the reward may be systematically optimistic relative to the latent utility that should define the policy objective.
For example, the choice of utility construction can materially affect the estimated rule in cost-effectiveness analysis \citep{xu2022estimation}, and patient preferences can alter the optimal treatment even when the observed clinical outcomes are identical \citep{butler2018incorporating}.
If such reward uncertainty is ignored, direct optimization of the observed reward can select policies whose apparent performance is inflated relative to their true value.
Moreover, this distortion is inherently policy-dependent: two rules with similar nominal values may differ substantially in the extent to which they place treatment recommendations in regions where the observed reward is a poor proxy for the latent utility.
While existing OWL variants improve robustness to outliers, efficiency, or computational stability, they still largely treat the observed reward as the absolute target quantity to be optimized \citep{zhou2017residual,liu2018augmented,fu2019robust,YazzourhFreeman2024}.
Despite its practical and conceptual importance, the OWL literature contains little finite-sample theory for handling reward uncertainty in a way that directly informs how the rule should be learned.
Consequently, reward uncertainty should not be treated as a nuisance affecting only point estimation; it must enter the learning problem itself.

To address this issue, we propose PAC-Bayesian Reward-Certified Outcome Weighted Learning (PROWL).
We introduce a latent true reward $R^\ast$ and assume that the data provide an observable one-sided uncertainty certificate $U$ such that $R^\ast \ge R-U$ almost surely.
This induces a conservative reward $\underline{R}=(R-U)_+$ and, consequently, a policy-dependent lower bound on the true value.
Within this framework, our contributions are fourfold:
(i) we derive an exact certified reduction that converts reward-certified policy learning into a weighted classification problem;
(ii) we establish a PAC-Bayes finite-sample lower bound on the unobserved target value over a joint policy--nuisance space, enabling split-free learning;
(iii) we prove that the posterior exactly maximizing this bound takes the form of a general Bayes update;
and (iv) we develop a fully data-driven criterion for tuning the learning rate, along with a practical hinge-based surrogate posterior.
By incorporating the certificate directly into the learning objective, PROWL robustly targets the rule whose latent value can be most strongly certified from finite samples.

\subsection{Related Work}
\label{subsec:related-work}

Our proposed framework builds upon and connects two major theoretical domains:
PAC-Bayes theory and general Bayes inference.

PAC-Bayes theory provides nonasymptotic guarantees for randomized predictors by balancing empirical performance with a Kullback--Leibler complexity penalty relative to a prior \citep{mcallester1998some,seeger2002pac,mcallester2003pac,maurer2004note,catoni2007pac,rodriguez2024more}.
Originally developed for supervised learning, PAC-Bayes tools have recently proven useful for offline decision problems such as contextual bandits and offline policy learning \citep{sakhi2023pac,gouverneur2025refined}.
However, the existing PAC-Bayes policy-learning literature does not address the setting in which the reward entering the value function is itself uncertain, nor does it exploit the exact weighted-classification reduction that makes OWL attractive for ITR estimation.

Furthermore, our theoretical results intersect with the literature on general Bayes updating.
The framework of \citet{bissiri2016general} shows that coherent posterior updating can be based directly on a loss function, yielding a Gibbs or generalized Bayesian posterior.
This perspective is especially appealing when the policy objective is defined through a value or risk rather than through a fully specified likelihood, and it has motivated a growing literature on model-robust Bayesian learning and misspecification-aware updating \citep{holmes2017assigning,grunwald2017inconsistency,lyddon2019general,miller2019robust}.
A persistent challenge in that literature is the choice of learning rate, which controls the trade-off between prior information and empirical loss (see also \citet{syring2019calibrating,wu2023comparison,altamirano2023robust,tamano2025location}).
By contrast, our analysis shows that, in reward-certified OWL, the generalized posterior arises as the exact optimizer of a PAC-Bayes lower bound on the latent policy value.
This provides a finite-sample, decision-theoretic derivation of a general Bayes update rather than a purely axiomatic or heuristic one.

\subsection{Organization of the Article}
\label{subsec:organization}

The rest of this paper is organized as follows.
In Section \ref{sec:preliminaries}, we introduce the causal framework for ITRs, formalize reward uncertainty using one-sided certificates, and define the surrogate objective.
In Section \ref{sec:proposed-method}, we derive the PAC-Bayes lower bound, establish its equivalence to a general Bayes update, and present a practical surrogate implementation with data-driven learning rate calibration.
In Section \ref{sec:numerical-experiments}, we examine numerical experiments.
In Section \ref{sec:actual-data-experiments}, we apply the proposed method to actual data.
In Section \ref{sec:discussion}, we conclude this paper with a discussion of implications, limitations, and future work.
Finally, the appendices contain detailed technical proofs and extended configurations for the experiments.
\section{Preliminaries}
\label{sec:preliminaries}

\subsection{Problem Setup}
\label{subsec:problem-setup}

Let $(\Omega,\mathcal{F},\PP)$ be a probability space.
We observe the data $Z=(X,A,R)$, where $X\in\cX\subset\RR^p$ denotes the covariates, $A\in\{-1,1\}$ is the assigned treatment, and $R\in[0,1]$ is a proxy reward.
The target reward $R^\ast\in[0,1]$ is assumed to be unobserved.
The policy-learning sample $Z_{1:n}=\{(X_i,A_i,R_i)\}_{i=1}^n$ consists of independent and identically distributed (i.i.d.) draws, which are embedded in an underlying i.i.d.\ sequence $\{(X_i,A_i,R_i,R_i^\ast)\}_{i=1}^n$.
For $a\in\{-1,1\}$, let $R^a$ and $(R^\ast)^a$ denote the potential proxy and target rewards, respectively.
For a measurable ITR $d:\cX\to\{-1,1\}$, we define the counterfactual rewards under the regime as $R^d\coloneq R^{d(X)}$ and $(R^\ast)^d\coloneq (R^\ast)^{d(X)}$, with their corresponding expected values given by $V(d)\coloneq \EE[R^d]$ and $V^\ast(d)\coloneq \EE[(R^\ast)^d]$.

\begin{assumption}[Identification and positivity]
\label{ass:id}
    For each $a\in\{-1,1\}$, $R=R^a$ and $R^\ast=(R^\ast)^a$ almost surely on
    $\{A=a\}$.
    Furthermore, we assume conditional ignorability and strict overlap:
    \begin{equation*}
        \{R^1,R^{-1},(R^\ast)^1,(R^\ast)^{-1}\}\indep A\mid X
        ,
        \quad
        \pi(a\mid x)\coloneq \PP(A=a\mid X=x)\ge \varepsilon
    \end{equation*}
    for all $(x,a)\in\cX\times\{-1,1\}$ and some constant $\varepsilon\in(0,1/2]$.
\end{assumption}

Under Assumption~\ref{ass:id}, standard inverse-probability weighting yields the identification results:
\begin{equation*}
    V(d)=\EE\Biggl[
        \frac{R\1\{A=d(X)\}}{\pi(A\mid X)}
    \Biggr]
    ,
    \quad
    V^\ast(d)=\EE\Biggl[
        \frac{R^\ast\1\{A=d(X)\}}{\pi(A\mid X)}
    \Biggr]
    .
\end{equation*}
We utilize these identities throughout this paper and defer their standard proof to Appendix~\ref{app:identification}.

\subsection{Reward Uncertainty}
\label{subsec:reward-uncertainty}

\begin{assumption}[One-sided reward certificate]
\label{ass:certificate}
    There exists a measurable function
    $U:\cX\times\{-1,1\}\times[0,1]\to[0,1]$ such that $R^\ast \ge R-U(X,A,R)$ almost surely.
\end{assumption}

We define the certified lower reward as $\uR\coloneq (R-U)_+$ and its corresponding certified value as
\begin{equation*}
    V_{\uR}(d)
    \coloneq
    \EE\Biggl[
        \frac{\uR\1\{A=d(X)\}}{\pi(A\mid X)}
    \Biggr]
    .
\end{equation*}
Since $R^\ast\ge \uR$ almost surely, the true value of any policy is strictly bounded below by its certified value, i.e., $V^\ast(d)\ge V_{\uR}(d)$.
If $U\le R$ almost surely, then $\uR=R-U$ and the relationship simplifies to
\begin{equation*}
    V^\ast(d)
    \ge
    V(d)-\EE\Biggl[
        \frac{U\1\{A=d(X)\}}{\pi(A\mid X)}
    \Biggr]
    .
\end{equation*}
Consequently, reward uncertainty alters the ranking of policies only when the induced penalty is policy-dependent; this occurs, for instance, through treatment-specific certificates or active clipping in $\uR=(R-U)_+$.

\subsection{Policy Class}
\label{subsec:policy-class}

Let $(\cA,\cH)$ be a measurable nuisance-parameter space, and let $\nu_a:\cX\times\cA\to[0,1]$ be a jointly measurable function for each $a\in\{-1,1\}$.
For any $\alpha\in\cA$, we denote $\nu_{a,\alpha}(x)\coloneq \nu_a(x,\alpha)$.

Similarly, let $(\cB,\cG)$ be a measurable policy-parameter space, and let $f:\cX\times\cB\to\RR$ be jointly measurable.
For $\beta\in\cB$, we define $f_\beta(x)\coloneq f(x;\beta)$, $d_\beta(x)\coloneq \sgn\{f_\beta(x)\}$, adopting the convention $\sgn(0)=1$.
For a generic measurable score $g$, we write $d_g(x)\coloneq \sgn\{g(x)\}$.
We collect both the policy and nuisance parameters into a joint parameter space $\Theta\coloneq \cB\times\cA$.

A key structural advantage of our theoretical framework is its split-free nature: the nuisance parameter $\alpha$ is learned jointly with the policy parameter $\beta$ using the exact same policy-learning sample.
This property is fundamentally structural rather than algorithmic:
because the certified value identity introduced below is exact for every measurable $\alpha$, PAC-Bayes bounds can be applied directly on the joint space $\Theta$, provided that the prior is specified independently of $Z_{1:n}$.
\section{Proposed Method}
\label{sec:proposed-method}

In this section, we introduce our primary methodological framework and its corresponding estimator, which is PAC-Bayesian Reward-Certified Outcome Weighted Learning (PROWL).
To clearly distinguish between our structural theory and the practical estimator, we develop PROWL progressively.
First, we establish an exact certified value representation that translates the robust policy learning problem into a unified, cost-sensitive classification task (Section~\ref{subsec:certified-value-representation}).
Leveraging this split-free reduction, we construct a nonasymptotic PAC-Bayes lower bound on the unobserved target value.
We demonstrate that the optimal posterior distribution maximizing this bound inherently takes the form of a general Bayes update; this posterior, and its induced Gibbs policy, constitutes the foundational PROWL estimator (Section~\ref{subsec:pac-general-bayes}).
Finally, we bridge the gap to practical implementation by providing an automated, bound-based procedure for learning-rate selection and introducing a computationally efficient surrogate posterior (Section~\ref{subsec:learning-rate}).

\subsection{Certified Value Representation}
\label{subsec:certified-value-representation}

For $a\in\{-1,1\}$ and $\alpha\in\cA$, we define the doubly robust-style score
\begin{equation}
    \label{eq:Gamma-a}
    \Gamma_a^\alpha(Z)
    =
    \nu_{a,\alpha}(X)
    +
    \frac{\1\{A=a\}}{\pi(a\mid X)}\{\uR-\nu_{a,\alpha}(X)\}
    ,
\end{equation}
and, for a measurable ITR $d$, we define
\begin{equation}
    \label{eq:Gamma-d}
    \Gamma_d^\alpha(Z)
    =
    \sum_{a\in\{-1,1\}}\1\{d(X)=a\}\Gamma_a^\alpha(Z)
    .
\end{equation}
We then construct the certified advantage, the induced pseudo-label, and the absolute weight as follows:
\begin{equation}
\label{eq:DYW}
    D_\alpha(Z)\coloneq \Gamma_1^\alpha(Z)-\Gamma_{-1}^\alpha(Z)
    ,
    \quad
    Y_\alpha(Z)\coloneq \sgn\{D_\alpha(Z)\}
    ,
    \quad
    W_\alpha(Z)\coloneq |D_\alpha(Z)|
    .
\end{equation}
For each fixed $\alpha\in\cA$, let $\cR_{01}^\alpha(d)\coloneq \EE[W_\alpha(Z)\1\{Y_\alpha(Z)\neq d(X)\}]$ be the weighted $0$-$1$ risk, and define $C_\alpha^\sharp\coloneq 2^{-1}\EE[\Gamma_1^\alpha(Z)+\Gamma_{-1}^\alpha(Z)+W_\alpha(Z)]$.
We also define $\mu_a(x)\coloneq \EE[\uR\mid X=x,A=a]$ for $a\in\{-1,1\}$.

\begin{theorem}[Exact certified reduction]
\label{thm:exact-certified-reduction}
    For every measurable ITR $d$ and every $\alpha\in\cA$,
    \begin{equation*}
        V_{\uR}(d)
        =
        \EE[\Gamma_d^\alpha(Z)]
        =
        C_\alpha^\sharp-\cR_{01}^\alpha(d)
        .
    \end{equation*}
    Consequently, for each fixed $\alpha$, maximizing $V_{\uR}(d)$ over a measurable $d$ is strictly equivalent to minimizing $\cR_{01}^\alpha(d)$.
    Moreover, any Bayes rule for $V_{\uR}$ takes the form
    \begin{equation*}
        d_{\uR}^\ast(x)\in \argmax_{a\in\{-1,1\}}\mu_a(x)
        .
    \end{equation*}
\end{theorem}
See Appendix~\ref{app:pf-thm-exact-certified-reduction} for the proof.
Theorem~\ref{thm:exact-certified-reduction} demonstrates that maximizing the certified lower bound is structurally equivalent to solving a cost-sensitive classification problem, which allows nuisance parameter learning to be seamlessly absorbed into the joint parameter optimization without the need for computationally expensive sample splitting.

The first equality in Theorem~\ref{thm:exact-certified-reduction} represents the split-free identity crucial to our PAC-Bayes development; it holds exactly for every $\alpha$, allowing nuisance learning to be subsumed within the joint parameter $\theta=(\beta,\alpha)$.
The second equality provides the exact signed classification reduction for a fixed choice of nuisance parameter.

While we defer the discussion of special cases (certified OWL, certified RWL, and arm-specific augmentation) to Appendix~\ref{app:special-cases}, we explicitly establish the conditional variance calculations that theoretically motivate our optimal nuisance design below.
For $a\in\{-1,1\}$, let $\sigma_a^2(x)\coloneq \Var(\uR\mid X=x,A=a)$ and $p_a(x)\coloneq \pi(a\mid x)$.

\begin{proposition}[Conditional moments and variance-optimal nuisance choices]
\label{prop:variance-optimality}
    Fix $\alpha\in\cA$.
    The following statements hold for $P_X$-almost every $x\in\cX$.
    
    \begin{enumerate}[label=(\roman*)]
        \item For each $a\in\{-1,1\}$,
        \begin{equation}
        \label{eq:Gamma-mean-var}
            \begin{split}
                \EE[\Gamma_a^\alpha(Z)\mid X=x]
                &=
                \mu_a(x)
                ,
                \\
                \Var(\Gamma_a^\alpha(Z)\mid X=x)
                &=
                \frac{\sigma_a^2(x)}{p_a(x)}
                +
                \frac{1-p_a(x)}{p_a(x)}
                \{\mu_a(x)-\nu_{a,\alpha}(x)\}^2
                .
            \end{split}
        \end{equation}
        Consequently, the conditional variance of $\Gamma_a^\alpha(Z)$ is uniquely minimized at $x$ by selecting $\nu_{a,\alpha}(x)=\mu_a(x)$.
    
        \item Writing $p(x)\coloneq \pi(1\mid x)$ and $q(x)\coloneq \pi(-1\mid x)=1-p(x)$, we have
        \begin{equation}
        \label{eq:D-mean}
            \EE[D_\alpha(Z)\mid X=x]
            =
            \mu_1(x)-\mu_{-1}(x)
            ,
        \end{equation}
        and
        \begin{equation}
        \label{eq:D-var}
            \begin{split}
                &\Var(D_\alpha(Z)\mid X=x)
                \\
                &=
                \frac{\sigma_1^2(x)}{p(x)}
                +
                \frac{\sigma_{-1}^2(x)}{q(x)}
                +
                \frac{
                    \bigl[
                        q(x)\{\mu_1(x)-\nu_{1,\alpha}(x)\}
                        +
                        p(x)\{\mu_{-1}(x)-\nu_{-1,\alpha}(x)\}
                    \bigr]^2
                }{
                    p(x)q(x)
                }
                .
            \end{split}
        \end{equation}
        Therefore, the conditional variance of $D_\alpha(Z)$ is minimized at $x$ by any choice of nuisance pair satisfying the linear constraint
        \begin{equation}
            \label{eq:variance-min-condition}
            q(x)\nu_{1,\alpha}(x)+p(x)\nu_{-1,\alpha}(x)
            =
            q(x)\mu_1(x)+p(x)\mu_{-1}(x)
            .
        \end{equation}
    
        \item Restricted to the treatment-free subclass $\nu_{1,\alpha}=\nu_{-1,\alpha}=m_\alpha$, the unique minimizer of the conditional variance is
        \begin{equation}
            \label{eq:treatment-free-optimal}
            m^\dagger(x)
            =
            q(x)\mu_1(x)+p(x)\mu_{-1}(x)
            .
        \end{equation}
        In the context of balanced randomization, where $p(x)=q(x)=1/2$, this expression elegantly simplifies to
        \begin{equation}
            \label{eq:balanced-treatment-free-optimal}
            m^\dagger(x)
            =
            \EE[\uR\mid X=x]
            .
        \end{equation}
    \end{enumerate}
\end{proposition}
See Appendix~\ref{app:pf-prop-variance-optimality} for the proof.
Proposition~\ref{prop:variance-optimality} provides the theoretical foundation for optimal nuisance parameter design, showing that setting the nuisance functions to the conditional expected certified rewards uniquely minimizes the conditional variance of the certified score, thereby maximizing statistical efficiency.

\subsection{PAC-Bayes Bound and General Bayes Update}
\label{subsec:pac-general-bayes}

Let $\Pi_0\in\mathcal P(\Theta)$ be a prior distribution that is independent of the policy-learning sample $Z_{1:n}$.
For any $\theta=(\beta,\alpha)\in\Theta$, define the empirical certified value as $\hat{V}_{\uR,n}(d_\theta) \coloneq n^{-1}\sum_{i=1}^n \Gamma_{d_\beta}^\alpha(Z_i)$.
For a posterior $Q\ll\Pi_0$, let $d_Q$ denote the Gibbs policy that draws $\tilde{\theta}=(\tilde{\beta},\tilde{\alpha})\sim Q$ and deploys the rule $d_{\tilde{\beta}}$.
Its expected target value and empirical certified value are given by $V^\ast(d_Q)\coloneq \EE_{\theta\sim Q}[V^\ast(d_\beta)]$ and $\hat{V}_{\uR,n}(d_Q)\coloneq \EE_{\theta\sim Q}[\hat{V}_{\uR,n}(d_\theta)]$, respectively.
By Theorem~\ref{thm:exact-certified-reduction}, $\hat{V}_{\uR,n}(d_\theta)$ is an unbiased estimator of $V_{\uR}(d_\beta)$ for every fixed $\theta$.

Set $c_\varepsilon\coloneq \varepsilon^{-1}$, $K_\varepsilon\coloneq 2\varepsilon^{-1}-1$, and define the certified value loss as $\ell_{\mathrm{val}}(\theta;Z) \coloneq (c_\varepsilon-\Gamma_{d_\beta}^\alpha(Z))/K_\varepsilon$.
Since $\Gamma_{d_\beta}^\alpha(Z)\in[1-\varepsilon^{-1},\varepsilon^{-1}]$,
the loss $\ell_{\mathrm{val}}(\theta;Z)$ naturally takes values within the range $[0,1]$.

\begin{theorem}[PAC-Bayes lower bound and exact general Bayes update]
\label{thm:exact-pac-bayes}
    Fix $\delta\in(0,1)$ and $\eta>0$.
    With probability at least $1-\delta$, the inequality
    \begin{equation*}
        V^\ast(d_Q)
        \ge
        \hat{V}_{\uR,n}(d_Q)
        -
        K_\varepsilon
        \Biggl\{
            \frac{\KL(Q\|\Pi_0)+\log(1/\delta)}{\eta n}
            +
            \frac{\eta}{8}
        \Biggr\}
    \end{equation*}
    holds simultaneously for all posteriors $Q\ll\Pi_0$.
    
    For a fixed learning rate $\eta$, the right-hand side of this bound is uniquely maximized by the posterior
    \begin{equation*}
        \hat{Q}_\eta^{\mathrm{PROWL,val}}(\dd\theta)
        \propto
        \exp\Biggl\{
            -\eta\sum_{i=1}^n \ell_{\mathrm{val}}(\theta;Z_i)
        \Biggr\}\Pi_0(\dd\theta)
        \propto
        \exp\Biggl\{
            \frac{\eta n}{K_\varepsilon}\hat{V}_{\uR,n}(d_\theta)
        \Biggr\}\Pi_0(\dd\theta)
        .
    \end{equation*}
\end{theorem}
See Appendix~\ref{app:pf-thm-exact-pac-bayes} for the proof.
Theorem~\ref{thm:exact-pac-bayes} reveals that rigorously bounding the target reward uncertainty strictly yields a general Bayes update, providing strong theoretical justification that optimizing the nonasymptotic PAC-Bayes lower bound intrinsically recovers the outcome-weighted learning objective.

Theorem~\ref{thm:exact-pac-bayes} establishes an exact general Bayes update \citep{bissiri2016general} over the joint space $\Theta$.
This forms our central conceptual result: maximizing a nonasymptotic PAC-Bayes lower bound on the unobserved target value exactly yields a mathematically principled general Bayes posterior.

Furthermore, Theorem~\ref{thm:exact-certified-reduction} demonstrates that, conditional on a fixed nuisance parameter $\alpha$, the $\beta$-dependent component of the empirical loss is exactly $K_\varepsilon^{-1}\sum_{i=1}^n W_{\alpha,i}\1\{Y_{\alpha,i}\neq d_\beta(X_i)\}$, up to an $\alpha$-dependent additive constant.
Consequently, the induced update for the policy parameter exactly mirrors the certified advantage-weighted OWL general Bayes update, conditional on $\alpha$.

\subsection{Learning-rate Selection and Practical Computation}
\label{subsec:learning-rate}

To optimally tune the learning rate with respect to the value function, we define
\begin{equation*}
    \hat{L}_{\mathrm{val},n}(Q)
    \coloneq
    \frac{c_\varepsilon-\hat{V}_{\uR,n}(d_Q)}{K_\varepsilon}
    ,
    \quad
    \xi(n)\coloneq \exp\Bigl(\frac{1}{12n}\Bigr)\sqrt{\frac{\pi n}{2}}+2
    ,
\end{equation*}
and $c_{n,\delta}(Q) \coloneq (\KL(Q\|\Pi_0)+\log\{\xi(n)/\delta\})/ n$.
For a temperature scale $\gamma>0$, let
\begin{equation*}
    \mathrm{LCB}_{n,\delta}^{\mathrm{val}}(Q;\gamma)
    \coloneq
    c_\varepsilon
    -
    K_\varepsilon
    \frac{
    1-\exp\{-c_{n,\delta}(Q)-\gamma \hat{L}_{\mathrm{val},n}(Q)\}
    }{
    1-\exp(-\gamma)
    }
    .
\end{equation*}

\begin{proposition}[Exact-value certification and temperature selection]
\label{prop:temperature-selection}
    With probability at least $1-\delta$, the bound
    \begin{equation*}
        V^\ast(d_Q)\ge \mathrm{LCB}_{n,\delta}^{\mathrm{val}}(Q;\gamma)
    \end{equation*}
    holds simultaneously for all posteriors $Q\ll\Pi_0$ and all $\gamma>0$.
    For a fixed $\gamma$, the unique maximizer of the mapping $Q\mapsto \mathrm{LCB}_{n,\delta}^{\mathrm{val}}(Q;\gamma)$ is $\hat{Q}_\gamma^{\mathrm{PROWL,val}}$.
    Hence, for any finite grid $\cT_n\subset(0,\infty)$,
    \begin{equation*}
        \hat{\gamma}\in
        \argmax_{\gamma\in\cT_n}
        \mathrm{LCB}_{n,\delta}^{\mathrm{val}}
        \bigl(
            \hat{Q}_\gamma^{\mathrm{PROWL,val}};\gamma
        \bigr)
    \end{equation*}
    guarantees that
    \begin{equation*}
        V^\ast(d_{\hat{Q}_{\hat{\gamma}}^{\mathrm{PROWL,val}}})
        \ge
        \max_{\gamma\in\cT_n}
        \mathrm{LCB}_{n,\delta}^{\mathrm{val}}
        \bigl(
            \hat{Q}_\gamma^{\mathrm{PROWL,val}};\gamma
        \bigr)
        .
    \end{equation*}
\end{proposition}
See Appendix~\ref{app:pf-prop-temperature-selection} for the proof.
Proposition~\ref{prop:temperature-selection} establishes that the learning rate can be safely and automatically tuned by maximizing a fully observable empirical lower confidence bound, thereby entirely circumventing the hyperparameter selection issues that typically plague standard general Bayes methods.

For practical computation, we employ a smooth surrogate posterior family over the same joint space $\Theta$.
Define the certified hinge loss as $\ell_{\mathrm{h}}(\theta;Z)\coloneq W_\alpha(Z)(1-Y_\alpha(Z)f_\beta(X))_+$ for $\theta=(\beta,\alpha)\in\Theta$, and construct the corresponding practical posterior:
\begin{equation*}
    \hat{Q}_\lambda^{\mathrm{PROWL,h}}(\dd\theta)
    \propto
    \exp\Biggl\{
        -\lambda\sum_{i=1}^n \ell_{\mathrm{h}}(\theta;Z_i)
    \Biggr\}\Pi_0(\dd\theta)
    ,
    \quad
    \lambda>0
    .
\end{equation*}
This specific family is designed for computational feasibility rather than exact inference: while it may not strictly maximize the theoretical PAC-Bayes bound detailed in Theorem~\ref{thm:exact-pac-bayes}, it can be rigorously evaluated using the exact-value certificate provided in Proposition~\ref{prop:temperature-selection}.
If the prior density follows $\pi_0(\beta,\alpha)\propto \exp\{-\lambda_0 J(\beta,\alpha)\}$, then the Maximum a Posteriori (MAP) estimator under $\hat{Q}_\lambda^{\mathrm{PROWL,h}}$ efficiently solves
\begin{equation*}
    \min_{(\beta,\alpha)\in\Theta}
    \Biggl\{
        \frac{1}{n}\sum_{i=1}^n \ell_{\mathrm{h}}(\theta;Z_i)
        +
        \frac{\lambda_0}{\lambda n}J(\beta,\alpha)
    \Biggr\}
    .
\end{equation*}
While we defer the standard variational characterization of $\hat{Q}_\lambda^{\mathrm{PROWL,h}}$ to Appendix~\ref{app:surrogate}, we establish the theoretical soundness of the certified hinge loss below.
For a measurable score $g:\cX\to\RR$, define the exact certified hinge risk as $\cR_{\mathrm{h}}^\alpha(g) \coloneq \EE[W_\alpha(Z)\{1-Y_\alpha(Z)g(X)\}_+]$.

\begin{proposition}[Fisher consistency and excess-risk domination of the certified hinge loss]
\label{prop:hinge-calibration}
    Fix $\alpha\in\cA$ and define the following quantities for $P_X$-almost every $x$:
    \begin{equation*}
        u_\alpha(x)
        \coloneq
        \EE\bigl[
            W_\alpha(Z)\1\{Y_\alpha(Z)=1\}
            \mid X=x
        \bigr],
        \quad
        v_\alpha(x)
        \coloneq
        \EE\bigl[
            W_\alpha(Z)\1\{Y_\alpha(Z)=-1\}
            \mid X=x
        \bigr]
        .
    \end{equation*}
    Then the following properties hold:
    
    \begin{enumerate}[label=(\roman*)]
        \item For $P_X$-almost every $x$, $u_\alpha(x)-v_\alpha(x)=\mu_1(x)-\mu_{-1}(x)$.
    
        \item Define the optimal measurable score as
        \begin{equation*}
            g_{\alpha}^{\dagger}(x)
            \coloneq
            \begin{cases}
            1, & u_\alpha(x)>v_\alpha(x),\\
            -1, & u_\alpha(x)<v_\alpha(x),\\
            0, & u_\alpha(x)=v_\alpha(x).
            \end{cases}
        \end{equation*}
        This score $g_\alpha^\dagger$ universally minimizes $\cR_{\mathrm{h}}^\alpha(g)$ across all measurable choices of $g$, and the induced classification rule $d_{g_\alpha^\dagger}$ remains Bayes-optimal for the certified value $V_{\uR}$.
        Equivalently, the certified hinge loss is Fisher-consistent with respect to the certified value-maximization objective.
    
        \item For every measurable score $g$, the excess value is strictly upper-bounded by the excess hinge risk:
        \begin{equation}
        \label{eq:excess-hinge}
            V_{\uR}(d_{\uR}^\ast)-V_{\uR}(d_g)
            \le
            \cR_{\mathrm{h}}^\alpha(g)
            -
            \inf_h \cR_{\mathrm{h}}^\alpha(h),
        \end{equation}
        where the infimum is taken over all measurable scores $h:\cX\to\RR$.
    \end{enumerate}
\end{proposition}
See Appendix~\ref{app:pf-prop-hinge-calibration} for the proof.
Proposition~\ref{prop:hinge-calibration} reveals that the certified hinge loss serves as a Fisher-consistent surrogate for the unobserved target value, guaranteeing that optimizing this convex surrogate inherently maximizes the exact certified lower value.

\begin{corollary}[Posterior-family exact-value selection]
\label{cor:family-selection}
    Let $\mathcal{Q}_n=\{\hat{Q}_\lambda:\lambda\in\Lambda_n\}$ be any finite,
    data-dependent family of posteriors satisfying $\hat{Q}_\lambda\ll\Pi_0$ almost surely for every $\lambda\in\Lambda_n$.
    Then, on the high-probability event defined in Proposition~\ref{prop:temperature-selection}, we have
    \begin{equation*}
        V^\ast(d_{\hat{Q}_\lambda})
        \ge
        \mathrm{LCB}_{n,\delta}^{\mathrm{val}}(\hat{Q}_\lambda;\gamma)
        \quad
        \text{for all }(\lambda,\gamma)\in\Lambda_n\times\cT_n
        .
    \end{equation*}
    Consequently, any chosen maximizer $(\hat\lambda,\hat{\gamma})\in \argmax_{(\lambda,\gamma)\in\Lambda_n\times\cT_n}\mathrm{LCB}_{n,\delta}^{\mathrm{val}}(\hat{Q}_\lambda;\gamma)$ rigorously satisfies
    \begin{equation*}
        V^\ast(d_{\hat{Q}_{\hat\lambda}})
        \ge
        \max_{(\lambda,\gamma)\in\Lambda_n\times\cT_n}
        \mathrm{LCB}_{n,\delta}^{\mathrm{val}}(\hat{Q}_\lambda;\gamma)
        .
    \end{equation*}
\end{corollary}
See Appendix~\ref{app:pf-cor-family-selection} for the proof.
Corollary~\ref{cor:family-selection} ensures that our automated tuning strategy remains robust and theoretically sound even when the optimization search space is restricted to a discrete, computationally tractable grid of practical surrogate posteriors.

\subsection{Learned Certificates via Auxiliary Calibration}
\label{subsec:learned-certificates}

While the core theory treats $U$ as a known, valid certificate, a more practical approach when $U$ must be actively learned is to calibrate it using an auxiliary sample that is completely independent of $Z_{1:n}$.
Let $\cD_m^{\mathrm{cal}}$ denote such an auxiliary sample, and let $\hat{U}=\mathfrak{A}_U(\cD_m^{\mathrm{cal}})$ be the estimated certificate.

\begin{corollary}[Learned certificates via auxiliary calibration]
\label{cor:learned-certificate}
    Suppose there exists an event $\cE_{\mathrm{cert}}\in \sigma(\cD_m^{\mathrm{cal}})$ with $\PP(\cE_{\mathrm{cert}})\ge 1-\alpha_{\mathrm{cert}}$ such that, conditional on $\cE_{\mathrm{cert}}$, the learned function $\hat{U}$ is measurable, bounded in $[0,1]$, and satisfies $R^\ast \ge R-\hat{U}(X,A,R)$ almost surely under the policy-learning distribution.
    Then, conditionally on $\cD_m^{\mathrm{cal}}$ and $\cE_{\mathrm{cert}}$, all preceding theoretical results hold seamlessly by simply substituting $U$ with $\hat{U}$ and $\uR$ with $\uR_{\hat{U}}\coloneq (R-\hat{U})_+$.
    In particular, every $1-\delta$ PAC-Bayes bound established above translates to a valid joint-probability guarantee with a confidence level of at least $(1-\alpha_{\mathrm{cert}})(1-\delta)$.
\end{corollary}
See Appendix~\ref{app:pf-cor-learned-certificate} for the proof.
Corollary~\ref{cor:learned-certificate} demonstrates that our theoretical framework effortlessly scales to realistic scenarios where the exact bounds are unknown, allowing researchers to plug in data-driven certificates without compromising the integrity of the rigorous PAC-Bayes guarantees.
\section{Numerical Experiments}
\label{sec:numerical-experiments}

\subsection{Experimental Setup}
\label{subsec:experimental-setup}

We evaluate the proposed method using two complementary synthetic benchmarks for single-stage ITR learning, characterized by a proxy reward $R$, an unobserved target reward $R^\ast$, and an oracle one-sided reward certificate.
For each arm $a\in\{-1,1\}$, the simulator specifies a base certificate envelope $U_{0,a}(x)$ and defines, at an uncertainty level $\rho\ge 0$,
\begin{equation*}
    U_{\rho,a}(x)
    =
    \min\{\rho\,U_{0,a}(x),\bar{u}_a\}
    ,
    \quad
    \uR^a
    =
    (R^a-U_{\rho,a}(X))_+
    .
\end{equation*}
Thus, the parameter $\rho$ controls the severity of the reward uncertainty while preserving its underlying spatial structure.
We employ oracle certificates throughout the synthetic study to ensure that any performance discrepancies purely reflect differences in policy learning capabilities rather than certificate-estimation errors.
In these setups, the true logging propensity is assumed to be known to the learner.
Scenario 1 utilizes balanced randomization, whereas Scenario 2 employs a covariate-dependent logging propensity to emulate non-uniform clinical allocation.
Unless otherwise stated, all configurations are averaged over $30$ Monte Carlo replications and evaluated on an independent test sample of size $10{,}000$.

Scenario 1 serves as a deliberately benign baseline.
Characterized by a linear treatment boundary and policy-invariant reward uncertainty, it represents an environment where standard ITR estimators should naturally excel.
Its primary role is diagnostic: determining whether robust certification introduces an unnecessary performance penalty (robustness cost) in a simple, well-specified setting.

Conversely, Scenario 2 is constructed to be substantially more challenging and realistic.
It incorporates correlated clinical and biomarker covariates, non-linear outcome surfaces, a genuinely beneficial treated subgroup confined to a moderate-risk clinical window, and treatment-specific proxy optimism concentrated in clinically vulnerable regions.
In this environment, directly optimizing the proxy reward can materially conflict with optimizing the latent target reward, perfectly encapsulating the regime that motivates the development of PROWL.

We conduct two primary sets of experiments: a $\rho$-sweep and an $N$-sweep.
In the $\rho$-sweep, we fix the sample size at $N=200$ and vary the uncertainty level
\begin{equation*}
    \rho\in\{0,0.25,0.5,0.75,1.0,1.25,1.5,1.75,2.0\}
    .
\end{equation*}
In the $N$-sweep, we fix $\rho=1.5$ and vary the sample size $N\in\{100,200,500,1000,2000\}$.
To isolate the fundamental effect of reward certification from function-class mismatches, the main benchmark utilizes a shared, scenario-specific score class across all methods:
a standardized linear score with an intercept in Scenario 1, and a commonly enriched clinical basis in Scenario 2.
All exact data-generating formulas, nuisance specifications, and hyperparameter tuning details are deferred to Appendix~\ref{app:subsec:detailed-experimental-setup}.

\subsection{Baselines}
\label{subsec:baselines}

We benchmark PROWL against five established methods:
(i) Standard OWL \citep{Zhao2012}, the canonical direct outcome-weighted classification approach;
(ii) Residual Weighted Learning (RWL) \citep{zhou2017residual}, which utilizes treatment-free residualization to stabilize weights;
(iii) Q-learning \citep{qian2011performance,schulte2015q}, a standard regression-based indirect method;
(iv) Policy Tree \citep{sverdrup2020policytree}, which provides an interpretable empirical-welfare benchmark; and
(v) PROWL ($U=0$), an ablation of our proposed method that retains the PAC-Bayesian learning machinery but removes the reward certification by setting $\uR=R$.
We refer to our full proposed method simply as PROWL.

For OWL, RWL, Q-learning, and Policy Tree, we report outcomes under two distinct reward families.
The $R$-family trains each method using the raw proxy reward $R$.
The $\uR$-family substitutes $R$ with the certified reward $\uR=(R-U)_+$ while keeping the respective learning algorithms otherwise unchanged.
Comparing the methods against the $R$-family isolates the overall cost of ignoring reward uncertainty.
Conversely, the comparison against the $\uR$-family provides a more stringent test, evaluating whether PROWL's performance gains stem from its principled PAC-Bayesian learning mechanism rather than merely from a naive substitution of $R$ with the conservative surrogate $\uR$.

\subsection{Evaluation Metrics}
\label{subsec:evaluation-metrics}

Let $d^\ast\in\argmax_d V^\ast(d)$ and $d_{\uR}^\ast\in\argmax_d V_{\uR}(d)$ denote the oracle policies for the latent target value and the certified lower value, respectively.
For any learned rule $\hat{d}$, we report:
\begin{equation*}
    \mathrm{TargetRegret}(\hat{d})
    =
    V^\ast(d^\ast)-V^\ast(\hat{d})
    ,
\end{equation*}
and
\begin{equation*}
    \mathrm{RobustRegret}(\hat{d})
    =
    V_{\uR}(d_{\uR}^\ast)-V_{\uR}(\hat{d})
    .
\end{equation*}

Target regret captures the latent scientific decision problem defined by $R^\ast$, making it the primary metric of substantive interest.
In contrast, robust regret assesses the certified lower-value objective explicitly optimized by PROWL, quantifying the certifiable utility left uncaptured by the learned policy.
For both metrics, lower values indicate superior performance.
Across all simulations, these regrets are computed on the independent test sample utilizing the true oracle conditional means from the simulator, rather than relying on noisy realized test outcomes.

\subsection{Results}
\label{subsec:results}

Figures~\ref{fig:scenario1-rho-sweep-main}--\ref{fig:scenario2-n-sweep-main} summarize the primary comparisons.
In each figure, the left panel plots target regret against the $R$-family baselines, while the right panel plots robust regret against the $\uR$-family baselines.
Complementary cross-family panels are provided in Appendix Figures~\ref{app:fig:scenario1-rho-sweep-comp}--\ref{app:fig:scenario2-n-sweep-comp}.
Additionally, certificate diagnostics and a split-free ablation study are detailed in Appendix~\ref{app:subsec:num-results}.

Figure~\ref{fig:scenario1-rho-sweep-main} presents the results for Scenario 1 across varying uncertainty levels ($\rho$).
As anticipated for a benign, policy-invariant environment, the best-performing methods are practically indistinguishable.
Although standard OWL and Policy Tree exhibit lower efficiency, PROWL closely matches the performance of Q-learning and RWL throughout the entire sweep.
This confirms that in a well-specified regime where standard methods already excel, incorporating reward certification does not impose a noticeable performance penalty.
 
Figure~\ref{fig:scenario2-rho-sweep-main} details the $\rho$-sweep results for Scenario 2, where the comparative differences become substantial.
As $\rho$ increases, the regret of the classical baselines rises correspondingly, with the uncertified PROWL ($U=0$) ablation exhibiting the most significant deterioration.
In stark contrast, PROWL consistently maintains low target and robust regrets.
The right panel is particularly revealing: even when compared against the $\uR$-family plug-in baselines, PROWL achieves noticeably superior performance at high uncertainty levels, showing the clearest separation around $\rho=2.0$.
This underscores that PROWL's empirical advantage derives from its core PAC-Bayesian learning principles rather than a simple substitution of $R$ with $\uR$.

Figure~\ref{fig:scenario1-n-sweep-main} illustrates the $N$-sweep for Scenario 1.
As the sample size increases, all methods naturally improve, and the performance gap among the leading estimators narrows.
While standard OWL and Policy Tree remain comparatively weaker, PROWL consistently tracks the top performers (Q-learning and RWL).
This reinforces the qualitative takeaway from the $\rho$-sweep: PROWL remains highly competitive with established methods even in simple, well-specified scenarios.

Figure~\ref{fig:scenario2-n-sweep-main} reports the $N$-sweep for Scenario 2.
While regret universally decreases as $N$ grows, PROWL consistently ranks as the best or near-best estimator across all sample sizes, achieving the lowest overall regrets.
This robust performance is evident in both the latent target and the certified objectives, holding up even against stronger $\uR$-family baselines.
As detailed in Appendix~\ref{app:subsec:num-results}, these conclusions remain intact across the complementary metrics: although utilizing $\uR$ instead of $R$ improves classical methods, PROWL yields further substantial gains.
The additional certificate diagnostics and the split-free ablation study provided in the appendix further corroborate that PROWL’s efficacy is fundamentally rooted in learning directly against the reward certificate.

\begin{figure*}[tb]
\vskip 0.2in
\begin{center}
\centerline{\includegraphics[width=\columnwidth]{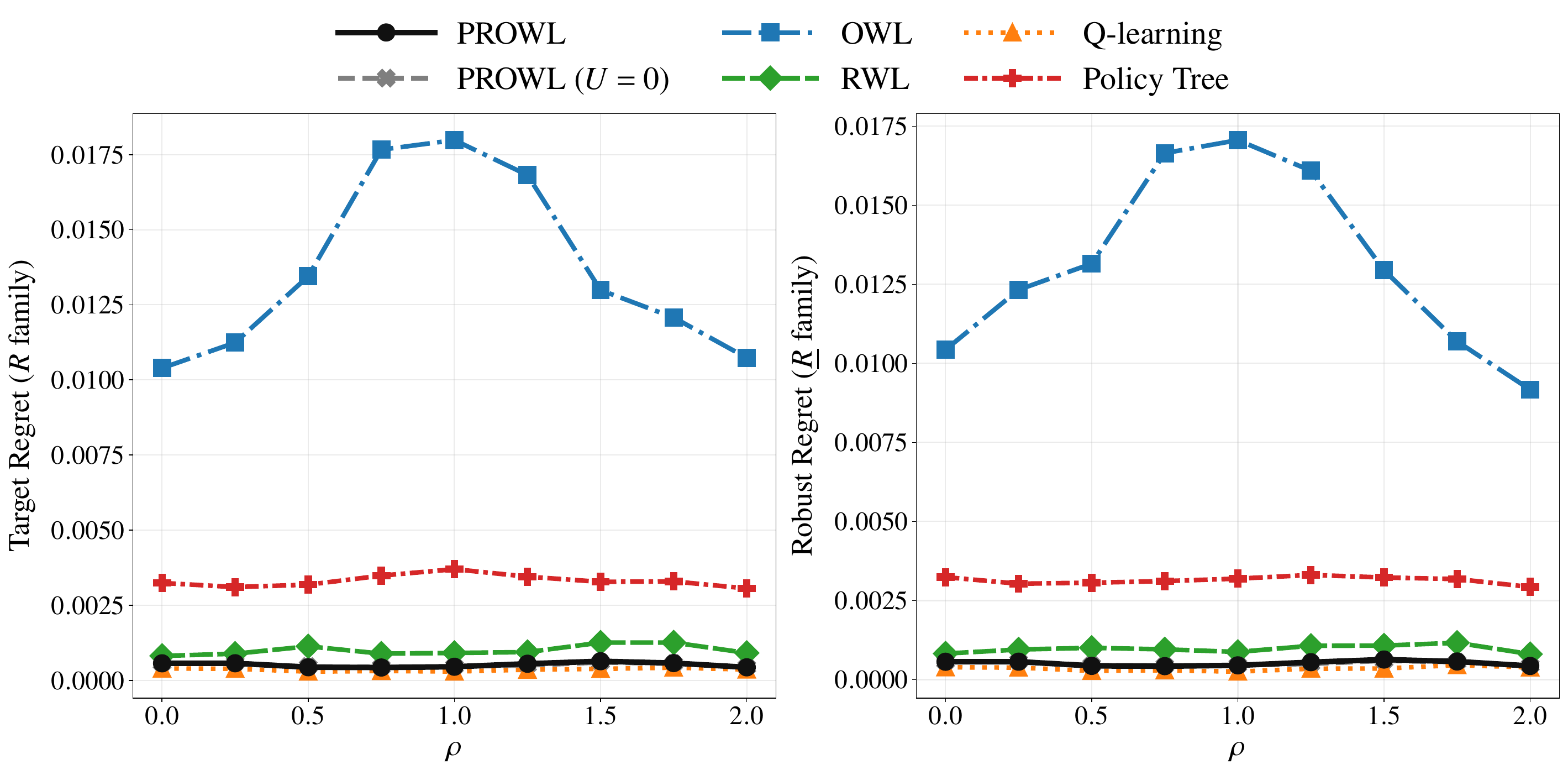}}
\caption{
Performance comparison across varying uncertainty levels ($\rho$) for Scenario 1.
The left panel reports target regret against the $R$-family baselines, and the right panel reports robust regret against the $\uR$-family baselines.
}
\label{fig:scenario1-rho-sweep-main}
\end{center}
\vskip -0.2in
\end{figure*}

\begin{figure*}[tb]
\vskip 0.2in
\begin{center}
\centerline{\includegraphics[width=\columnwidth]{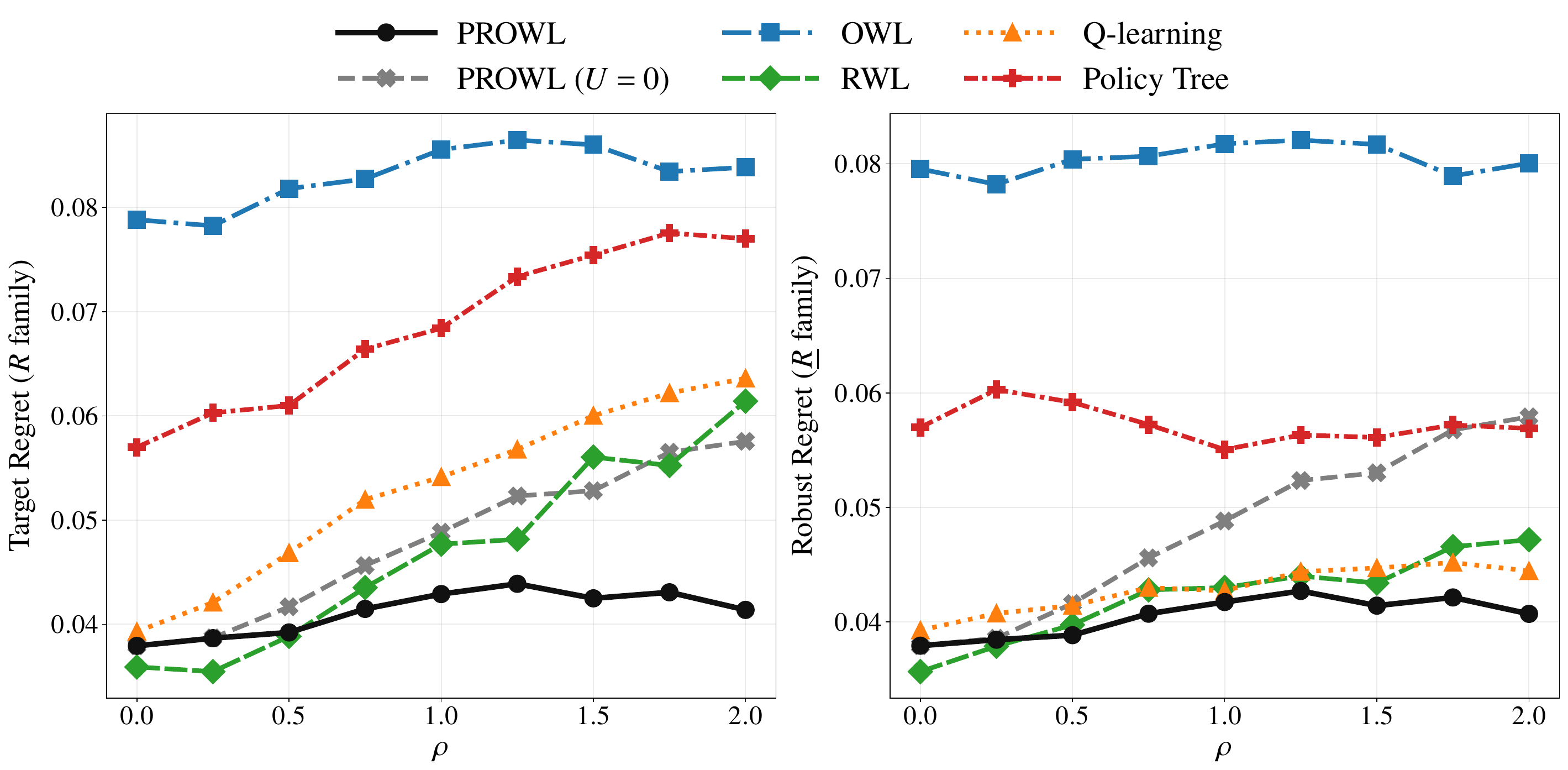}}
\caption{
Performance comparison across varying uncertainty levels ($\rho$) for Scenario 2.
The left panel reports target regret against the $R$-family baselines, and the right panel reports robust regret against the $\uR$-family baselines.
}
\label{fig:scenario2-rho-sweep-main}
\end{center}
\vskip -0.2in
\end{figure*}

\begin{figure*}[tb]
\vskip 0.2in
\begin{center}
\centerline{\includegraphics[width=\columnwidth]{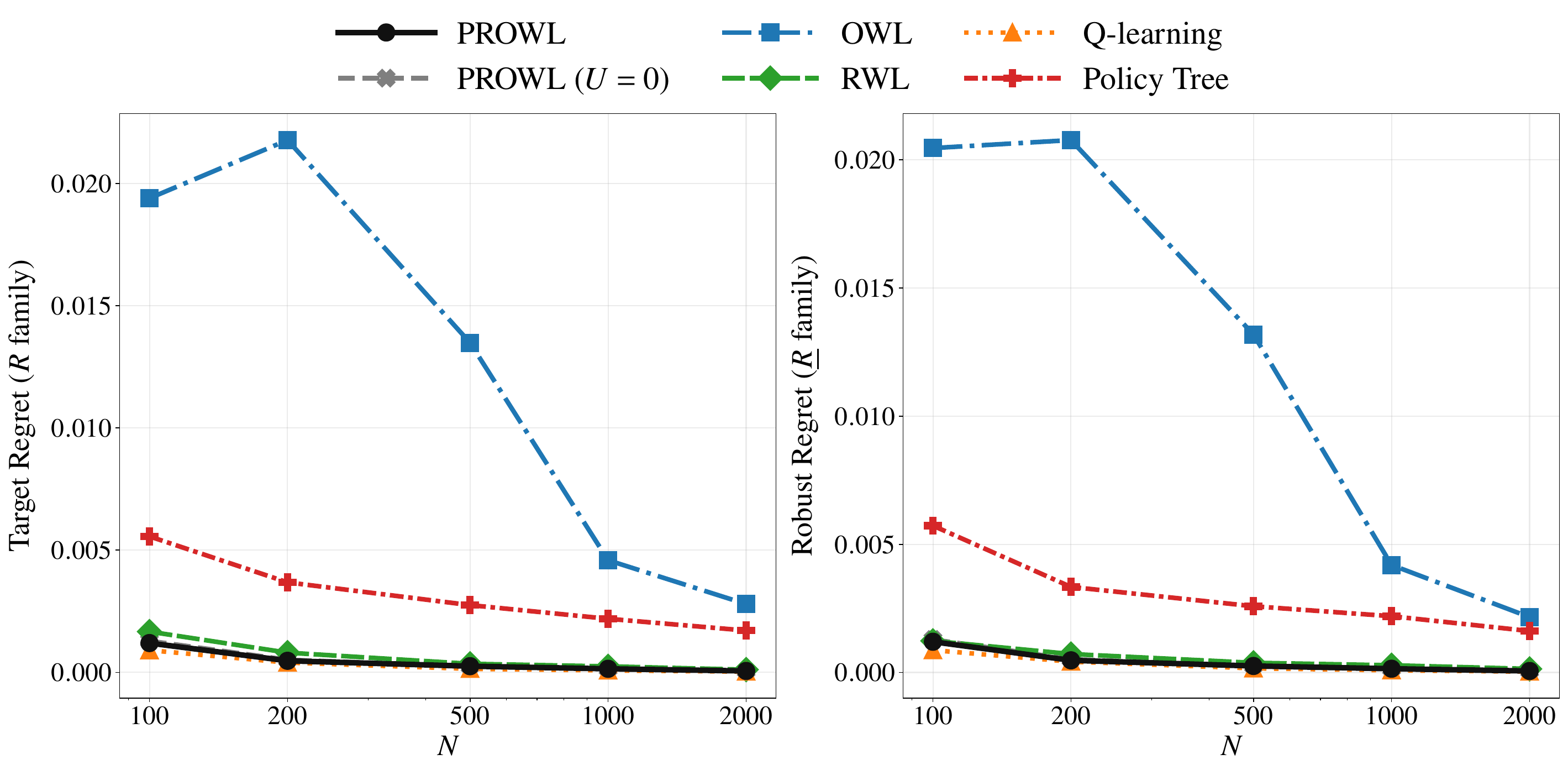}}
\caption{
Performance comparison across varying sample sizes ($N$) for Scenario 1.
The left panel reports target regret against the $R$-family baselines, and the right panel reports robust regret against the $\uR$-family baselines.
}
\label{fig:scenario1-n-sweep-main}
\end{center}
\vskip -0.2in
\end{figure*}

\begin{figure*}[tb]
\vskip 0.2in
\begin{center}
\centerline{\includegraphics[width=\columnwidth]{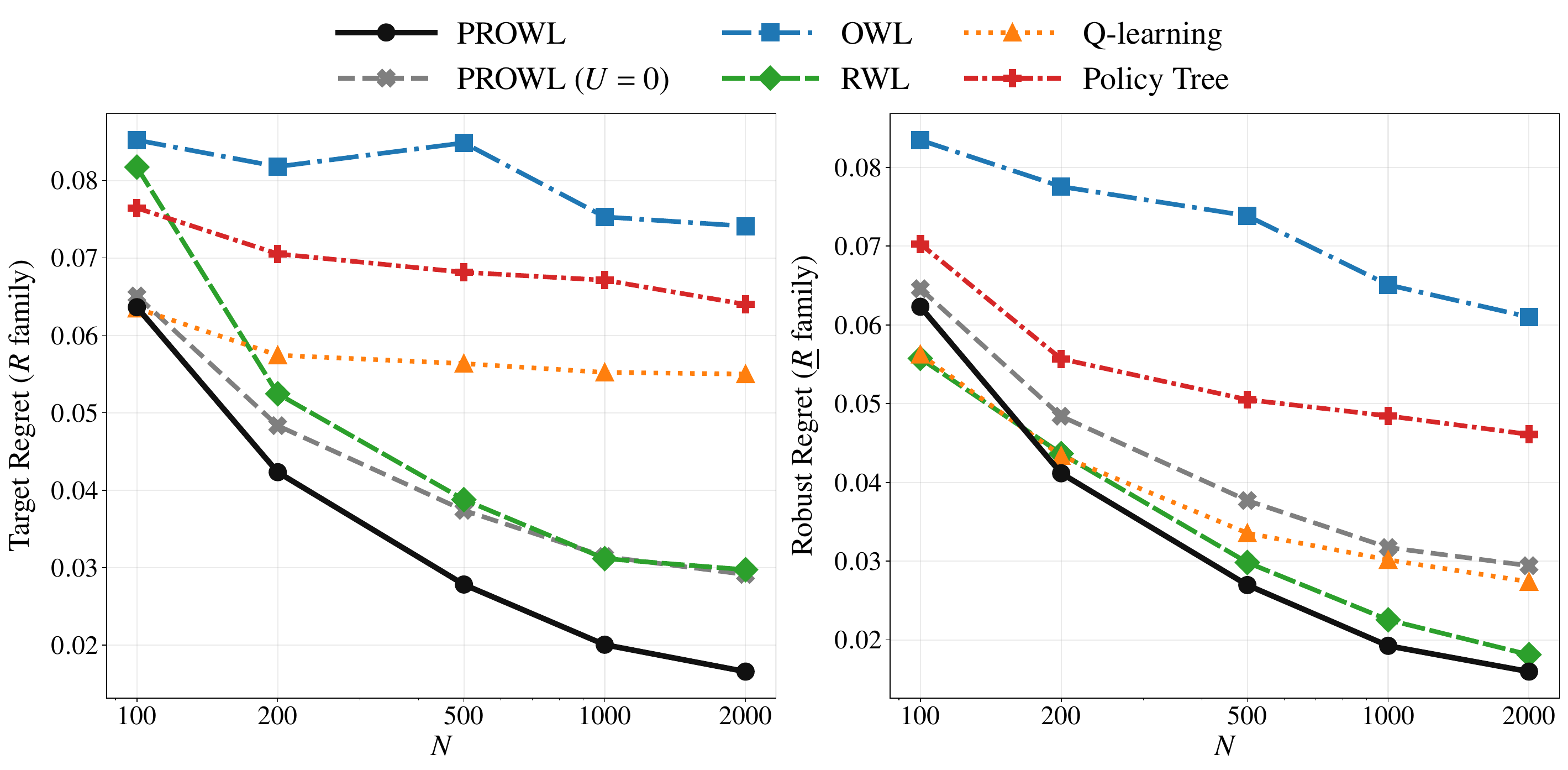}}
\caption{
Performance comparison across varying sample sizes ($N$) for Scenario 1.
The left panel reports target regret against the $R$-family baselines, and the right panel reports robust regret against the $\uR$-family baselines.
}
\label{fig:scenario2-n-sweep-main}
\end{center}
\vskip -0.2in
\end{figure*}
\section{Actual Data Experiments}
\label{sec:actual-data-experiments}

We next evaluate PROWL on deidentified data from the Electronic Alerts for Acute Kidney Injury (AKI) Amelioration study, using the public ELAIA-1 release, which also underlies the secondary analysis of \citet{huang2025identification}.
This experiment is designed to address a different question from those explored in Section~\ref{sec:numerical-experiments}.
Our goal is not to determine whether blanket alerting was beneficial on average, but rather to ascertain whether a clinically sensible, selective alerting rule can be learned when the utility itself is composite and preference-sensitive.
In this setting, a successful method should achieve high certified utility and avoid deterioration on the trial's primary hard-outcome composite, doing so without generating an unnecessarily aggressive alerting policy.

\subsection{Dataset Information}
\label{subsec:actual-dataset-information}

Each row of the ELAIA-1 release corresponds to a hospitalized patient with AKI who was individually randomized to receive either an electronic AKI alert or usual care.
We encode the treatment as
\begin{equation*}
    A_i=
    \begin{cases}
        1, & \text{if } \texttt{alert}_i=1,\\
        -1, & \text{if } \texttt{alert}_i=0.
    \end{cases}
\end{equation*}
Because treatment assignment relies on patient-level simple randomization, we use the known propensity $\pi(A=1\mid X)=\pi(A=-1\mid X)=1/2$ throughout both training and evaluation, rather than estimating it from the data.

Our primary policy covariate set comprises exclusively pre-randomization information.
These encompass hospital and admission context, demographics, AKI timing variables, renal-function markers, laboratory measurements, and vital signs at randomization, comorbidity indicators, recent medication or exposure history, and alert-burden variables.
We explicitly exclude post-randomization process variables, event-time variables, utilization and cost summaries, and the leakage-prone \texttt{max\_stage} field.
Within each training split, continuous variables are median-imputed and standardized with missingness indicators added as necessary; categorical and binary variables are mode-imputed; and hospital affiliation is encoded via six one-hot indicators.
The fitted preprocessing map is then applied without modification to the held-out evaluation fold.
Additional variable-level details are deferred to Appendix~\ref{app:sec:detail-actual-exp}.

To model the uncertain hard-clinical utility, we define for patient $i$
\begin{equation*}
    G_i
    =
    \Bigl(
        1-\texttt{death14}_i,\,
        1-\texttt{dialysis14}_i,\,
        1-\texttt{aki\_progression14}_i
    \Bigr)^\top
    \in\{0,1\}^3
    .
\end{equation*}
Our nominal proxy reward is $R_i=w_0^\top G_i$, with $w_0=(0.60,0.25,0.15)^\top$, which prioritizes survival, followed by dialysis avoidance, and finally the avoidance of AKI progression.
To encode preference uncertainty within these weights, we define the uncertainty set
\begin{equation*}
    \mathcal{W}_\rho
    =
    \Biggl\{
        w\in\mathbb{R}_+^3:
        \sum_{j=1}^3 w_j=1,
        \,
        |w_j-w_{0j}|\le \rho \Delta_j
    \Biggr\}
    ,
    \quad
    \Delta=(0.10,0.05,0.05)^\top
    ,
\end{equation*}
with $\rho=1$ in the main analysis.
The certified lower reward is then defined at the patient level as $\uR_i = \min_{w\in\mathcal{W}_\rho} w^\top G_i$, which constitutes a trivial three-dimensional linear program, and the induced one-sided certificate is $U_i=R_i-\uR_i$.
As an external clinical anchor aligned with the original trial endpoint, we also evaluate the composite-free reward $R_i^{\mathrm{comp}}=1-\texttt{composite\_outcome}_i$.

\subsection{Baselines}
\label{subsec:actual-baselines}

Our primary comparison evaluates two blanket rules (Never alert and Always alert), three classical learners trained on the nominal reward $R$, the same three learners trained on the certified reward $\uR$, the uncertified ablation PROWL $(U=0)$, and PROWL itself.
Specifically, the adaptive baselines are OWL, Q-learning, and RWL alongside their $\uR$-plugins.
The two reward families serve complementary purposes: comparison against the $R$-family quantifies the cost of entirely ignoring reward uncertainty, whereas comparison against the $\uR$-family isolates whether any performance gain stems merely from plugging a conservative reward into an otherwise standard learner.

For comparability, all adaptive learners employ linear policy scores in the main analysis.
OWL, Q-learning, and RWL are tuned via three-fold cross-validation over the penalty grid $\{10^{-3},10^{-2},10^{-1},1\}$, utilizing the fold-held-out augmented IPW (AIPW) value for the corresponding reward variant.
PROWL adopts the same policy class but couples it with arm-specific ridge nuisance regressions and data-driven temperature selection over $\eta,\gamma \in \{1/8,1/4,1/2,1,2,4,8\}$, optimized via the exact-value lower confidence bound described in Section~\ref{subsec:learning-rate}.
In the main text, we deploy PROWL and PROWL $(U=0)$ via the deterministic mean-rule induced by the fitted posterior;
Appendix~\ref{app:sec:detail-actual-exp} compares this approach with MAP and Gibbs deployment strategies.

We intentionally omit PolicyTree from the primary actual-data comparison.
This omission is not conceptual;
tree-based rules are already evaluated in the synthetic experiments;
but rather computational and inferential.
The central empirical question is whether reward certification improves policy learning within a common score-based regime class.
By contrast, PolicyTree optimizes over a fundamentally different hypothesis class and relies on an exact exhaustive tree search, whose amortized runtime for a depth-$k$ tree is on the order of $O(P^{k}N^{k}(\log N + D) + PN\log N)$, where $P$ is the feature dimension, $N$ is the number of distinct training observations, and $D$ is the number of actions \citep{sverdrup2020policytree}.
With $30$ repeated hospital-stratified splits, two reward families, and a high-dimensional EHR covariate set, this exhaustive search becomes the dominant computational bottleneck.
We therefore reserve tree-based comparisons for the synthetic benchmarks, where they remain both informative and computationally tractable.

\subsection{Evaluation Metrics}
\label{subsec:actual-evaluation-metrics}

All actual-data results are derived from $30$ repeated $70/30$ train/test splits, stratified by hospital and treatment arm.
For any reward $Y\in[0,1]$ and learned policy $d$, we estimate the held-out value using the AIPW estimator:
\begin{equation*}
    \hat{V}_{\mathrm{AIPW}}(d;Y)
    =
    \frac{1}{n_{\mathrm{test}}}
    \sum_{i\in\mathrm{test}}
    \Biggl[
        \hat{\mu}_{d(X_i)}(X_i)
        +
        \frac{\1\{A_i=d(X_i)\}}{1/2}
        \{Y_i-\hat{\mu}_{A_i}(X_i)\}
    \Biggr]
    ,
\end{equation*}
where $\hat{\mu}_a(x)$ represents an arm-specific outcome regression estimated on the training fold, and the propensity is fixed at the known randomized value of $1/2$.

We report three primary value criteria:
the certified value $\hat{V}_{\uR}$, the nominal value $\hat{V}_R$, and the composite-free value $\hat{V}_{\mathrm{comp}}$ computed from $R^{\mathrm{comp}}=1-\texttt{composite\_outcome}$.
The first metric strictly aligns with the PROWL objective, the second quantifies performance under the nominal proxy utility, and the third anchors the learned rule to the original trial's hard-outcome composite.
In the main text, we additionally report the estimated 14-day mortality risk and the alert rate:
\begin{equation*}
    \hat{p}_{\mathrm{alert}}
    =
    \frac{1}{n_{\mathrm{test}}}
    \sum_{i\in\mathrm{test}}
    \1\{d(X_i)=1\}
    .
\end{equation*}
Appendix~\ref{app:sec:detail-actual-exp} further decomposes the hard-outcome components into mortality, dialysis, AKI progression, composite outcome, and discharge-to-home rates.
All metrics are reported as the mean, with standard errors (SE) computed across the repeated sample splits.

\subsection{Results}
\label{subsec:actual-results}

Table~\ref{tab:actual-policy-comparison} and Figure~\ref{fig:actual-point-range} summarize the primary results.
The strongest overall rule is the blanket Never alert policy, which attains both the largest certified value and the largest composite-free value.
This finding is inherently informative: the ELAIA data do not represent a regime in which indiscriminate alerting is broadly beneficial. Consequently, any learned rule must justify its utility against a strong conservative benchmark rather than a blanket intervention.

Against this backdrop, PROWL emerges as the strongest adaptive learner.
Among the learned policies, it achieves the highest certified value, the highest composite-free value, and the numerically lowest mortality risk, while recommending alerts for approximately half of the patient population.
Relative to its uncertified ablation PROWL $(U=0)$, the gains are modest but directionally uniform: the certified value increases, the composite-free value improves, and the alert rate decreases.
When compared to the strongest raw-reward and lower-reward plug-in competitors (Q-learning ($R$) and Q-learning ($\uR$)), PROWL achieves marginally better certified and composite-free values while maintaining an equivalent alert burden.
Conversely, the OWL variants are substantially more aggressive, alerting roughly two-thirds to three-quarters of patients while delivering inferior certified and clinical performance.
These absolute differences are necessarily small because the conservative no-alert benchmark is already highly effective on this dataset; the salient empirical question is therefore not whether a learned rule can dominate Never alert, but whether certification systematically improves the quality of a selective alerting regime once the decision to learn such a rule has been made.

Figure~\ref{fig:actual-policy-allocation} clarifies how certification structurally alters the learned policy.
We compare PROWL against the strongest raw-reward baseline, Q-learning $(R)$, and the strongest lower-reward plug-in baseline, Q-learning $(\uR)$.
All three methods increase their alerting frequency in tandem with the estimated baseline risk, and all are more aggressive in teaching hospitals than in non-teaching hospitals.
However, PROWL is systematically more conservative in the non-teaching sites, particularly within the highest-risk quintile, whereas the two Q-learning policies continue to recommend alerts more indiscriminately.
In the top baseline-risk quintile at non-teaching hospitals, PROWL recommends alerts for about $30\%$ of patients, compared to roughly $34\%$ for Q-learning $(R)$ and $35\%$ for Q-learning $(\uR)$.
In teaching hospitals, the behavior of the three rules is much more closely aligned.
Thus, the practical role of reward certification is not to suppress alerts uniformly, but rather to strategically reallocate alerts away from sites and patient strata where the proxy reward appears excessively optimistic relative to the clinically anchored lower utility.

Overall, the actual-data analysis does not suggest that a learned adaptive rule should overturn the conservative Never alert benchmark.
Instead, it demonstrates that when learning an adaptive alerting policy under uncertain clinical utility, PROWL yields the most favorable empirical trade-off among the evaluated competitors:
it secures the highest certified value, prevents deterioration on the trial's primary composite anchor, achieves the lowest mortality risk, and maintains a moderate alert burden.

\begin{table}[t]
\centering
\caption{
Repeated sample-split AIPW comparison on the ELAIA-1 AKI-alert trial under the three-component uncertain hard-clinical utility with $\rho=1$.
Entries are means with standard errors over $30$ repeated $70/30$ hospital-by-alert stratified splits.
Certified value is the estimate of $V_{\uR}$, nominal value is the corresponding estimate of $V_R$, and composite-free value uses $R^{\mathrm{comp}}=1-\texttt{composite\_outcome}$.
Among adaptive learned policies, PROWL attains the largest certified value and the largest composite-free value.
}
\label{tab:actual-policy-comparison}
\setlength{\tabcolsep}{6pt}
\scalebox{0.8}{
\begin{tabular}{lrrrrr}
\toprule
Methods & Certified value & Nominal value & Composite-free value & Mortality risk & Alert rate 
\\
\midrule
Never alert         & $0.905\,(0.001)$ & $0.918\,(0.001)$ &
$0.797\,(0.001)$    & $0.086\,(0.001)$ & $0.000\,(0.000)$ \\

Always alert        & $0.900\,(0.001)$ & $0.914\,(0.001)$ &
$0.783\,(0.002)$    & $0.088\,(0.001)$ & $1.000\,(0.000)$ \\

OWL ($R$)           & $0.900\,(0.001)$ & $0.914\,(0.001)$ &
$0.786\,(0.002)$    & $0.089\,(0.001)$ & $0.672\,(0.024)$ \\

Q-learning ($R$)    & $0.903\,(0.001)$ & $0.917\,(0.001)$ &
$0.793\,(0.002)$    & $0.086\,(0.001)$ & $0.500\,(0.005)$ \\

RWL ($R$)           & $0.903\,(0.001)$ & $0.916\,(0.001)$ &
$0.793\,(0.002)$    & $0.088\,(0.001)$ & $0.450\,(0.030)$ \\

OWL ($\uR$)         & $0.901\,(0.001)$ & $0.915\,(0.001)$ &
$0.787\,(0.002)$    & $0.087\,(0.001)$ & $0.715\,(0.024)$ \\

Q-learning ($\uR$)  & $0.904\,(0.001)$ & $0.917\,(0.001)$ &
$0.794\,(0.002)$    & $0.086\,(0.001)$ & $0.500\,(0.005)$ \\

RWL ($\uR$)         & $0.903\,(0.001)$ & $0.917\,(0.001)$ &
$0.794\,(0.002)$    & $0.087\,(0.001)$ & $0.475\,(0.022)$ \\

PROWL ($U=0$)       & $0.904\,(0.001)$ & $0.917\,(0.001)$ &
$0.794\,(0.002)$    & $0.086\,(0.001)$ & $0.509\,(0.009)$ \\

PROWL               & $0.904\,(0.001)$ & $0.917\,(0.001)$ &
$0.795\,(0.002)$    & $0.086\,(0.001)$ & $0.498\,(0.008)$ \\
\bottomrule
\end{tabular}
}
\end{table}

\begin{figure*}[tb]
\vskip 0.2in
\begin{center}
\centerline{\includegraphics[width=\columnwidth]{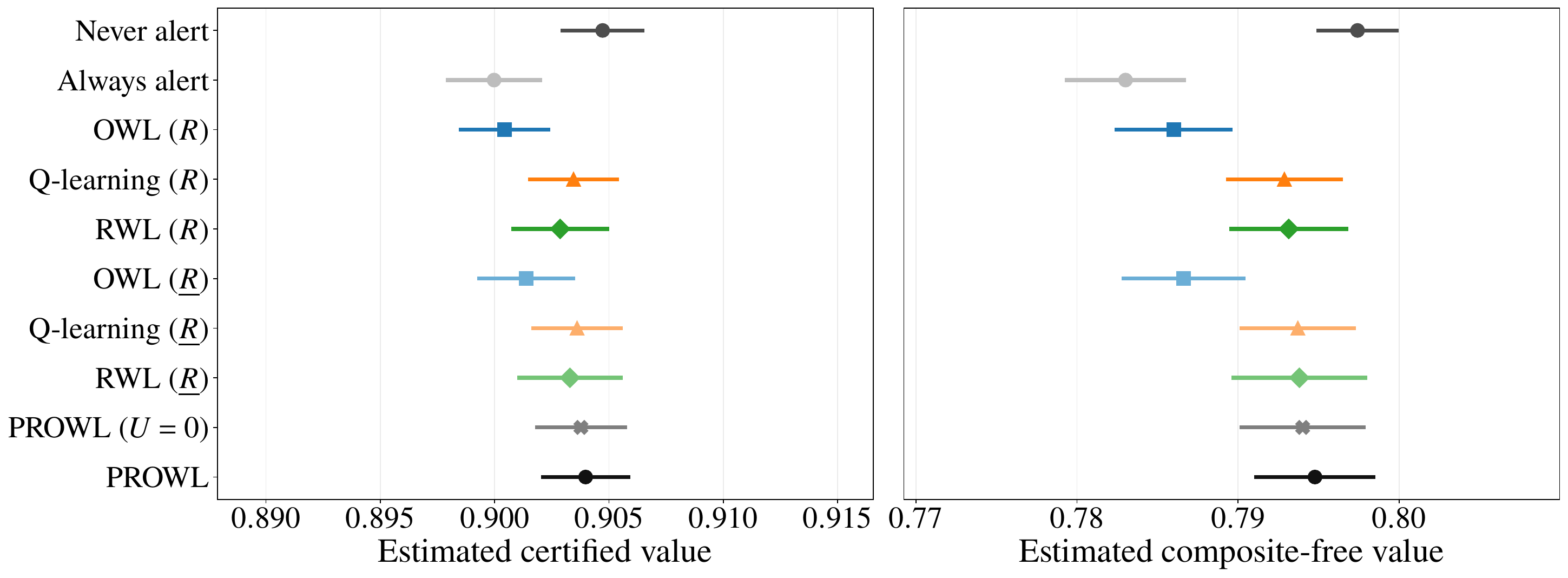}}
\caption{
Point-range summary of the primary actual-data comparison on ELAIA-1.
Left panel reports the estimated certified value $\hat{V}_{\uR}$ with $95\%$ confidence intervals across the $30$ repeated sample splits.
Right panel reports the estimated composite-free value $\hat{V}_{\mathrm{comp}}$ with the same uncertainty summary.
The figure complements Table~\ref{tab:actual-policy-comparison} by emphasizing method ordering on the certified objective and on the trial's primary hard-outcome anchor.
}
\label{fig:actual-point-range}
\end{center}
\vskip -0.2in
\end{figure*}

\begin{figure*}[tb]
\vskip 0.2in
\begin{center}
\centerline{\includegraphics[width=\columnwidth]{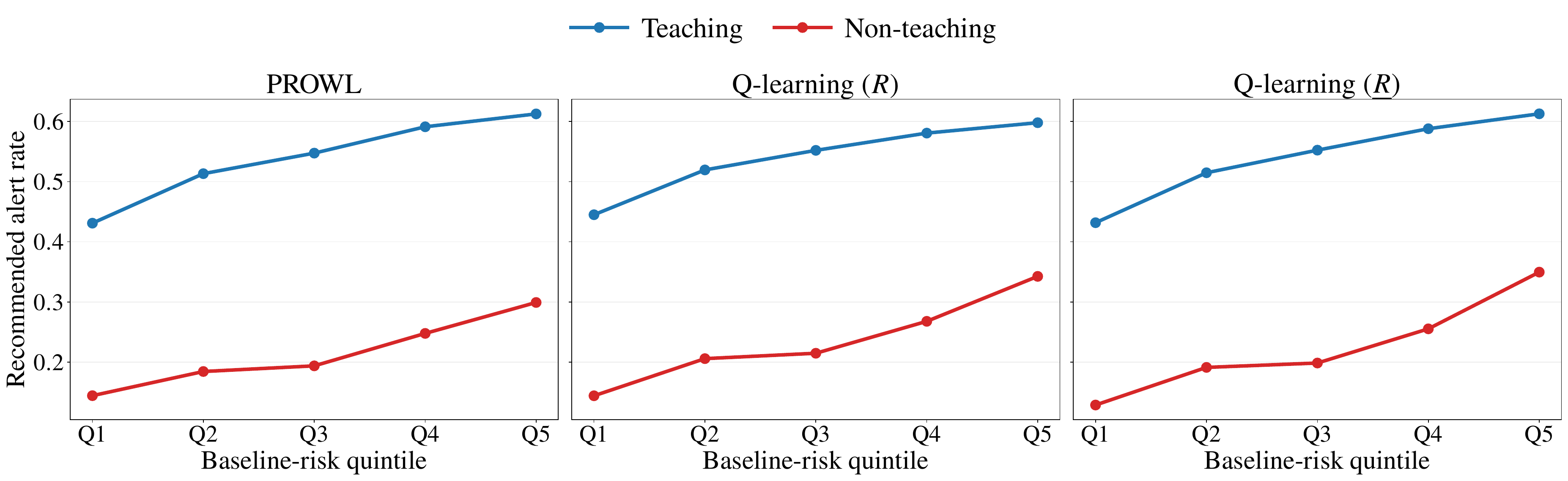}}
\caption{
Learned alert allocation by baseline-risk quintile and hospital type.
Baseline-risk quintiles are obtained within each training fold by fitting a control-arm risk model for the primary composite outcome and then evaluating the resulting risk scores on the held-out fold.
Each panel displays the mean recommended alert rate across the repeated sample splits for one policy.
Teaching hospitals are shown in blue and non-teaching hospitals in red.
PROWL is compared against the strongest raw-reward baseline, Q-learning $(R)$, and the strongest lower-reward plug-in baseline, Q-learning $(\uR)$.
}
\label{fig:actual-policy-allocation}
\end{center}
\vskip -0.2in
\end{figure*}
\section{Discussion}
\label{sec:discussion}

This paper establishes a reward-certified framework for outcome weighted learning. 
Motivated by applications where the observable reward serves merely as a noisy, delayed, or overly optimistic proxy for the latent utility of interest, PROWL effectively addresses this structural mismatch.
It achieves this by directly embedding a one-sided uncertainty certificate into the learning objective to optimize a nonasymptotic lower bound on the target value.
Two theoretical implications of this formulation are particularly notable.
First, the exact certified reduction in Theorem~\ref{thm:exact-certified-reduction} demonstrates that uncertainty-aware policy learning remains exactly reducible to a weighted classification problem.
This reduction preserves the algorithmic tractability that has made standard OWL widely adopted.
Second, Theorem~\ref{thm:exact-pac-bayes} establishes that the posterior maximizing our PAC-Bayes lower bound is precisely a generalized Bayes posterior.
Consequently, the generalized Bayesian update is not merely an axiomatic postulate in our setting.
Rather, it emerges naturally as the unique and exact optimizer of a finite-sample lower-confidence criterion for learning treatment rules.

This perspective rigorously situates PROWL within the broader direct-learning literature.
A substantial body of work has advanced OWL through residualization, augmentation, doubly robust scoring, sparse regularization, and flexible policy classes \citep{Song2015,zhou2017residual,liu2018augmented,zhu2017greedy,zhang2020multicategory,pan2021improved,zhou2023offline}.
However, these methods predominantly assume that the observed outcome is the definitive target and therefore optimize a point-estimated reward.
In contrast, PROWL is explicitly designed for environments where the observed reward is systematically misaligned with the true utility.
Furthermore, our framework complements recent causal learning approaches that rely on surrogate outcomes, semi-supervised labels, or latent variables to recover long-term targets prior to policy optimization \citep{chen2021learning,sonabend2023semi,yang2024targeting,fei2026learning}.
Whereas existing methods focus on identifying the target reward, PROWL shifts the paradigm toward certification.
Specifically, it asks which policy yields the highest utility that can be mathematically guaranteed from finite samples under strictly one-sided reward information.

The inherent conservatism underpinning the safety guarantees of PROWL also dictates its primary limitation.
When the certificate is tight and informative, the procedure effectively penalizes policies whose apparent values are inflated in regions of high reward uncertainty.
However, if the certificate is valid but excessively loose, the certified reward becomes overly pessimistic.
This pessimism attenuates the certified advantage, effectively flattening the learning objective and complicating policy discrimination.
This intrinsic trade-off highlights that certificate construction is not a benign preprocessing step but rather a core component of the statistical decision problem itself.
Practical success relies not only on marginal validity but also on ensuring that the certificate is highly policy-discriminating.
This necessary discrimination can be achieved through either treatment-specific structural constraints or active clipping.
As noted in Remark~\ref{rem:certificate-ranking}, if the uncertainty penalty acts in a policy-invariant manner, the certification remains valid but structurally fails to shift the optimal decision boundary.

Several directions for future theoretical work remain. 
First, while Theorem~\ref{thm:exact-pac-bayes} strictly certifies the expected value of the randomized Gibbs policy $V^\ast(d_Q)$, clinical applications typically demand a deterministic rule $T(Q)$. 
Bounding the derandomization regret $V^\ast(d_Q)-V^\ast(T(Q))$ is an open challenge, requiring the extension of margin-based and disintegrated PAC-Bayesian tools \citep{germain2015risk,biggs2022margins,viallard2024disintegration} to handle value-based objectives with inverse-propensity weighting and joint nuisance learning.
Second, our current analysis assumes that the prior and any actively learned certificates are constructed independently of the policy-learning sample. 
Tuning the certificate $U$ on the same data $Z_{1:n}$ to tighten the bound would induce a double-overfitting phenomenon that alters the empirical objective itself; formalizing this necessitates new simultaneous high-probability bounds leveraging distribution-dependent priors or random hypothesis sets \citep{parrado2012pac,lever2013tighter,dziugaite2018data,dupuis2024uniform}. 
Third, our exact general Bayes update relies on the linearity of the certified value, which follows from its exact reduction to a weighted $0$--$1$ risk. 
Extending this mathematically rigorous framework to non-linear, ambiguity-aware value functionals, while simultaneously controlling uncertainties arising from both off-policy evaluation and the reward certificate \citep{athey2021policy,sakhi2023pac,gouverneur2025refined}, represents a critical next step for robust policy learning.

Finally, the PROWL framework can be overly conservative in finite samples because it bounds the optimistic error instead of using an exact oracle bias correction.
Mitigating this over-correction while maintaining strict one-sided validity remains a crucial open problem.
One possible solution is to develop higher-order corrections or value-aware priors, similar to asymptotically unbiased Bayesian frameworks \citep{hartigan1965asymptotically,sakai2025priors}.
However, this application is challenging, since the target policy-dependent objective requires inverse-probability weighting and joint nuisance learning, which introduce large variances and complex correlations, making standard prior adjustments difficult. 

In conclusion, PROWL establishes a formal bridge connecting causal policy learning, PAC-Bayesian generalization theory, and generalized Bayesian inference. 
By elevating reward uncertainty from a post-hoc consideration to a fundamental component of the optimization objective, our framework shifts the focus of individualized treatment estimation from maximizing nominal values to verifiable safety. 
This certified approach provides a mathematically principled foundation for deploying data-driven decision rules in high-stakes domains. 
In such settings, establishing the reliability of a policy is as critical as optimizing its empirical performance.

\section*{Code Availability}
The Python implementation of the proposed method and simulation experiments in this study are available at \url{https://github.com/shutech2001/PROWL}.

\section*{Acknowledgements}
This work was supported by JSPS KAKENHI Grant Number 25K24203.

\bibliographystyle{apalike}
\bibliography{mybib}

\appendix
\section{Additional Theory}
\label{app:sec:additional-theory}

Since $\cX\subset\RR^p$ is a Borel subset of a Polish space, regular conditional distributions given $X$ exist.
Throughout this appendix, we fix measurable versions of every conditional expectation and conditional variance that appears below.
Furthermore, because $U$, $\nu_a$, and $f$ are jointly measurable and the mapping $\sgn:\RR\to\{-1,1\}$ is Borel measurable, the composite mappings $(x,\beta)\mapsto d_\beta(x)$, $(z,\alpha)\mapsto \Gamma_a^\alpha(z)$, $(z,\beta,\alpha)\mapsto \Gamma_{d_\beta}^\alpha(z)$, as well as the induced maps $(z,\alpha)\mapsto D_\alpha(z)$, $(z,\alpha)\mapsto Y_\alpha(z)$, $(z,\alpha)\mapsto W_\alpha(z)$, $(z,\beta,\alpha)\mapsto \ell_{\mathrm{val}}(\beta,\alpha;z)$, and $(z,\beta,\alpha)\mapsto \ell_{\mathrm{h}}(\beta,\alpha;z)$ are all rigorously measurable.
Consequently, all expectations, conditional expectations, and posterior integrals utilized below are rigorously well-defined.

\subsection{Identification and Certified Lower-Value Domination}
\label{app:identification}

We begin by presenting the standard inverse-probability weighting representation, followed by a formalization of the policy-wise domination induced by the certificate.

\begin{proposition}[Inverse-propensity representations]
\label{prop:ips}
    Under Assumption~\ref{ass:id}, every measurable ITR $d:\cX\to\{-1,1\}$ exactly satisfies
    \begin{equation}
    \label{eq:ips}
        V(d)
        =
        \EE\Biggl[
            \frac{R\1\{A=d(X)\}}{\pi(A\mid X)}
        \Biggr],
        \quad
        V^\ast(d)
        =
        \EE\Biggl[
            \frac{R^\ast\1\{A=d(X)\}}{\pi(A\mid X)}
        \Biggr]
        .
    \end{equation}
\end{proposition}
See Appendix~\ref{app:pf-prop-ips} for the proof.
Proposition~\ref{prop:ips} formally establishes that the expected and target values of any measurable policy can be robustly identified using standard inverse-probability weighting under the strict overlap and ignorability assumptions.

\begin{proposition}[Certified lower-value domination]
\label{prop:lower-value-domination}
    Under Assumptions~\ref{ass:id} and \ref{ass:certificate}, every measurable ITR $d$ strictly satisfies
    \begin{equation}
    \label{eq:lower-value-domination}
        V^\ast(d)\ge V_{\uR}(d)
        .
    \end{equation}
    If, in addition, $U\le R$ almost surely, then $\uR=R-U$ almost surely, yielding
    \begin{equation}
    \label{eq:lower-value-domination-unclipped}
        V^\ast(d)
        \ge
        V(d)
        -
        \EE\Biggl[
            \frac{U\1\{A=d(X)\}}{\pi(A\mid X)}
        \Biggr]
        .
    \end{equation}
\end{proposition}
See Appendix~\ref{app:pf-prop-lower-value-domination} for the proof.
Proposition~\ref{prop:lower-value-domination} demonstrates that any policy's unobserved target value is fundamentally bounded below by its observable certified value, guaranteeing that optimizing this proxy strictly improves the lower bound of the target reward.

\begin{remark}[When the certificate does not alter policy ranking]
\label{rem:certificate-ranking}
    Suppose that $U\le R$ almost surely and that $U=U(X)$ depends exclusively on $X$.
    Then, for every measurable ITR $d$, we have
    \begin{equation*}
        \EE\Biggl[
            \frac{U(X)\1\{A=d(X)\}}{\pi(A\mid X)}
        \Biggr]
        =
        \EE\bigl[U(X)\bigr]
        ,
    \end{equation*}
    which follows directly from the identity
    \begin{equation*}
        \EE\Biggl[
            \frac{\1\{A=d(X)\}}{\pi(A\mid X)}
            \Bigm| X
        \Biggr]
        =
        1
        \quad\text{a.s.}
    \end{equation*}
    Hence, the penalty term detailed in \eqref{eq:lower-value-domination-unclipped} is entirely independent of the policy.
    Consequently, reward uncertainty impacts the ranking of policies only through treatment-dependent certification or the active clipping inherent in the definition $\uR=(R-U)_+$.
\end{remark}

\subsection{Special Cases and Nuisance Design}
\label{app:special-cases}

We next highlight several salient special cases of the certified reduction, accompanied by conditional moment calculations that elucidate the optimal design of the nuisance functions.

\begin{proposition}[Special cases of the certified reduction]
\label{prop:special-cases}
    Fix $\alpha\in\cA$.
    
    \begin{enumerate}[label=(\roman*)]
        \item \textbf{Certified OWL.}
        If $\nu_{1,\alpha}\equiv \nu_{-1,\alpha}\equiv 0$, then for $a\in\{-1,1\}$,
        \begin{equation*}
            \Gamma_a^\alpha(Z)
            =
            \frac{\uR\1\{A=a\}}{\pi(a\mid X)}
            ,
        \end{equation*}
        and therefore,
        \begin{equation*}
            D_\alpha(Z)
            =
            \frac{A\uR}{\pi(A\mid X)}
            ,
            \quad
            W_\alpha(Z)
            =
            \frac{\uR}{\pi(A\mid X)}
            ,
            \quad
            Y_\alpha(Z)=A
            \quad\text{on }\{\uR>0\}
            .
        \end{equation*}
        Because $W_\alpha(Z)=0$ on the event $\{\uR=0\}$, the resulting weighted classification problem corresponds exactly to OWL applied to the certified reward $\uR$.
    
        \item \textbf{Certified residual weighted learning.}
        If $\nu_{1,\alpha}=\nu_{-1,\alpha}=m_\alpha$ for some measurable mapping $m_\alpha:\cX\to[0,1]$, then
        \begin{equation*}
            D_\alpha(Z)
            =
            \frac{A\{\uR-m_\alpha(X)\}}{\pi(A\mid X)}
            ,
            \quad
            W_\alpha(Z)
            =
            \frac{|\uR-m_\alpha(X)|}{\pi(A\mid X)}
            ,
            \quad
            Y_\alpha(Z)
            =
            A\,\sgn\{\uR-m_\alpha(X)\}
            .
        \end{equation*}
        Under a balanced randomization scenario and setting $m_\alpha(x)=\EE[\uR\mid X=x]$, this yields the exact certified equivalent of residual weighted learning.
    
        \item \textbf{Arm-specific augmentation.}
        For arbitrarily measurable arm-specific nuisances $\nu_{1,\alpha}$ and $\nu_{-1,\alpha}$, each $\Gamma_a^\alpha$ functions as an unbiased action-specific score, such that
        \begin{equation*}
            \EE[\Gamma_a^\alpha(Z)\mid X=x]=\mu_a(x)
            \quad\text{for }P_X\text{-a.e. }x
            .
        \end{equation*}
    \end{enumerate}
\end{proposition}
See Appendix~\ref{app:pf-prop-special-cases} for the proof.
Proposition~\ref{prop:special-cases} clarifies the versatility of our certified reduction framework, revealing that it seamlessly recovers certified versions of Outcome Weighted Learning and Residual Weighted Learning through specific, principled choices of the nuisance functions.

\subsection{Properties of the Certified Hinge Surrogate}
\label{app:surrogate}

For a given measurable score function $g:\cX\to\RR$, recall the induced decision rule $d_g(x)\coloneq \sgn\{g(x)\}$ and the exact certified hinge risk $\cR_{\mathrm{h}}^\alpha(g)$ defined in the main text.

\begin{lemma}[Elementary bounds for the certified advantage]
\label{lem:DW-bounds}
    Under Assumption~\ref{ass:id}, for every $\alpha\in\cA$,
    \begin{equation}
    \label{eq:D-bound}
        |D_\alpha(Z)|\le 1+\varepsilon^{-1}
        \quad\text{a.s.,}
    \end{equation}
    which directly implies
    \begin{equation}
    \label{eq:W-bound}
        0\le W_\alpha(Z)\le 1+\varepsilon^{-1}
        \quad\text{a.s.}
    \end{equation}
\end{lemma}
See Appendix~\ref{app:pf-lem-DW-bounds} for the proof.
Lemma~\ref{lem:DW-bounds} establishes deterministic upper bounds on the certified advantage and absolute weights, which are essential for controlling the empirical processes in our subsequent nonasymptotic analysis.

\begin{proposition}[Variational characterization of the practical posterior]
\label{prop:surrogate-optimal-posterior}
    Fix $\lambda>0$ and define the empirical hinge loss
    \begin{equation*}
        \hat{L}_{\mathrm{h},n}(\theta)
        \coloneq
        \frac{1}{n}\sum_{i=1}^n \ell_{\mathrm{h}}(\theta;Z_i),
        \quad
        \theta\in\Theta
        .
    \end{equation*}
    Then, for any realized sample $Z_{1:n}$, the objective functional
    \begin{equation*}
        Q\mapsto
        \EE_{\theta\sim Q}\bigl[
            \hat{L}_{\mathrm{h},n}(\theta)
        \bigr]
        +
        \frac{1}{\lambda n}\KL(Q\|\Pi_0)
    \end{equation*}
    defined over the space of probability measures $\{Q:Q\ll\Pi_0\}$, admits the unique global minimizer
    \begin{equation}
    \label{eq:surrogate-posterior-variational}
        \hat{Q}_\lambda^{\mathrm{PROWL,h}}(\dd\theta)
        =
        \frac{
            \exp\{-\lambda n\hat{L}_{\mathrm{h},n}(\theta)\}\Pi_0(\dd\theta)
        }{
            \int_\Theta \exp\{-\lambda n\hat{L}_{\mathrm{h},n}(t)\}\Pi_0(\dd t)
        }.
    \end{equation}
    Because $\hat{L}_{\mathrm{h},n}(\theta)=n^{-1}\sum_{i=1}^n \ell_{\mathrm{h}}(\theta;Z_i)$, the expression in \eqref{eq:surrogate-posterior-variational} aligns exactly with the posterior family defined in Section~\ref{subsec:learning-rate}.
\end{proposition}
See Appendix~\ref{app:pf-prop-surrogate-optimal-posterior} for the proof.
Proposition~\ref{prop:surrogate-optimal-posterior} formalizes the variational characterization of the empirical Gibbs posterior, directly linking the minimization of the empirical certified hinge risk to a well-defined and theoretically sound Bayesian update.

\begin{corollary}[MAP form under an exponential prior density]
\label{cor:MAP-penalized}
    Assume that the prior $\Pi_0$ admits a density with respect to a $\sigma$-finite reference measure $\rho$ of the form
    \begin{equation*}
        \pi_0(\beta,\alpha)
        =
        c_0\exp\{-\lambda_0 J(\beta,\alpha)\}
        ,
        \quad
        (\beta,\alpha)\in\Theta
        ,
    \end{equation*}
    where $c_0>0$ is a normalization constant, $\lambda_0>0$ controls the prior concentration, and $J:\Theta\to(-\infty,\infty]$ is a measurable penalty function.
    Under these conditions, the posterior $\hat{Q}_\lambda^{\mathrm{PROWL,h}}$ admits the $\rho$-density
    \begin{equation*}
        (\beta,\alpha)
        \mapsto
        c_{\lambda,n}
        \exp\Biggl\{
            -\lambda\sum_{i=1}^n \ell_{\mathrm{h}}((\beta,\alpha);Z_i)
            -\lambda_0 J(\beta,\alpha)
        \Biggr\}
        ,
    \end{equation*}
    where $c_{\lambda,n}>0$ is the normalizing constant.
    Consequently, any MAP estimator under $\hat{Q}_\lambda^{\mathrm{PROWL,h}}$ natively minimizes the regularized empirical risk:
    \begin{equation}
    \label{eq:MAP-penalized}
        (\beta,\alpha)
        \mapsto
        \frac{1}{n}\sum_{i=1}^n \ell_{\mathrm{h}}((\beta,\alpha);Z_i)
        +
        \frac{\lambda_0}{\lambda n}J(\beta,\alpha)
        .
    \end{equation}
\end{corollary}
See Appendix~\ref{app:pf-cor-MAP-penalized} for the proof.
Corollary~\ref{cor:MAP-penalized} demonstrates that computing the MAP estimator elegantly reduces to solving a standard regularized empirical risk minimization problem, ensuring practical computational tractability.

\subsection{Auxiliary PAC-Bayes Tools}
\label{app:pacbayes-tools}

This subsection consolidates the foundational variational and concentration inequalities utilized throughout the proofs of our main PAC-Bayes results.

\begin{lemma}[Donsker--Varadhan variational principle]
\label{lem:DV}
    Let $S:\Theta\to\RR$ be a measurable function, and suppose there exists a $\tau>0$ satisfying the integrability condition
    \begin{equation*}
        \int_\Theta e^{\tau S(\theta)}\Pi_0(\dd\theta)<\infty
        .
    \end{equation*}
    Then, the following dual representation holds:
    \begin{equation}
    \label{eq:DV}
        \sup_{Q\ll\Pi_0}
        \Biggl\{
            \EE_{\theta\sim Q}[S(\theta)]
            -
            \frac{1}{\tau}\KL(Q\|\Pi_0)
        \Biggr\}
        =
        \frac{1}{\tau}
        \log\int_\Theta e^{\tau S(\theta)}\Pi_0(\dd\theta)
        ,
    \end{equation}
    where the supremum is uniquely achieved by the tilted measure
    \begin{equation*}
        Q^\ast(\dd\theta)
        =
        \frac{e^{\tau S(\theta)}\Pi_0(\dd\theta)}
        {\int_\Theta e^{\tau S(t)}\Pi_0(\dd t)}
        .
    \end{equation*}
    Correspondingly, if $T:\Theta\to\RR$ is measurable and $\int_\Theta e^{-\tau T(\theta)}\Pi_0(\dd\theta)<\infty$, then
    \begin{equation}
    \label{eq:DV-inf}
        \inf_{Q\ll\Pi_0}
        \Biggl\{
            \EE_{\theta\sim Q}[T(\theta)]
            +
            \frac{1}{\tau}\KL(Q\|\Pi_0)
        \Biggr\}
        =
        -\frac{1}{\tau}
        \log\int_\Theta e^{-\tau T(\theta)}\Pi_0(\dd\theta)
        ,
    \end{equation}
    and the infimum is exclusively attained by the Gibbs measure proportional to $e^{-\tau T(\theta)}\Pi_0(\dd\theta)$.
\end{lemma}
See Appendix~\ref{app:pf-lem-DV} for the proof.
Lemma~\ref{lem:DV} provides the standard Donsker--Varadhan variational principle, which serves as the core algebraic mechanism for converting exponentiated empirical bounds into penalized risk bounds over probability measures.

\begin{lemma}[Scalar exponential bound for bounded variables]
\label{lem:scalar-kl}
    Let $Y_1,\dots,Y_n$ be a sequence of i.i.d.\ random variables bounded in $[0,1]$ with expectation $\mu=\EE[Y_1]$, and define their empirical mean $\bar{Y}_n=n^{-1}\sum_{i=1}^n Y_i$.
    Then, the following exponential inequality holds:
    \begin{equation}
    \label{eq:scalar-kl}
        \EE\Bigl[
            \exp\bigl\{
                n\kl(\bar{Y}_n\|\mu)
            \bigr\}
        \Bigr]\le \xi(n)
        ,
        \quad
        \xi(n)\coloneq \exp\Biggl(
            \frac{1}{12n}
        \Biggr)\sqrt{\frac{\pi n}{2}}+2
        .
    \end{equation}
\end{lemma}
See Appendix~\ref{app:pf-lem-scalar-kl} for the proof.
Lemma~\ref{lem:scalar-kl} provides a tight, nonasymptotic exponential inequality for the Kullback--Leibler divergence between Bernoulli distributions, yielding a strictly bounded moment-generating function that is essential for our uniform concentration results.

\begin{lemma}[Uniform Seeger--Langford inequality for bounded losses]
\label{lem:uniform-seeger-langford}
    Let $\ell:\Theta\times(\cX\times\{-1,1\}\times[0,1])\to[0,1]$ be a measurable loss function, and define
    \begin{equation*}
        L(\theta)\coloneq \EE[\ell(\theta;Z)],
        \quad
        \hat{L}_n(\theta)\coloneq \frac{1}{n}\sum_{i=1}^n \ell(\theta;Z_i)
        .
    \end{equation*}
    For any posterior $Q\ll\Pi_0$, define the aggregate risks
    \begin{equation*}
        L(Q)\coloneq \EE_{\theta\sim Q}[L(\theta)]
        ,
        \quad
        \hat{L}_n(Q)\coloneq \EE_{\theta\sim Q}[\hat{L}_n(\theta)]
        .
    \end{equation*}
    Then, for any tolerance level $\delta\in(0,1)$, the bound
    \begin{equation}
    \label{eq:uniform-seeger-langford}
        \kl\bigl(\hat{L}_n(Q)\|L(Q)\bigr)
        \le
        \frac{\KL(Q\|\Pi_0)+\log\{\xi(n)/\delta\}}{n}
    \end{equation}
    holds with probability at least $1-\delta$, simultaneously across all posteriors $Q\ll\Pi_0$.
\end{lemma}
See Appendix~\ref{app:pf-lem-uniform-seeger-langford} for the proof.
Lemma~\ref{lem:uniform-seeger-langford} establishes a uniform Seeger--Langford-type PAC-Bayes bound, allowing pointwise concentration properties to hold simultaneously across all absolutely continuous posterior distributions.

\begin{lemma}[Uniform Catoni bound for bounded losses]
\label{lem:uniform-catoni}
    Under the same assumptions and notation as Lemma~\ref{lem:uniform-seeger-langford}, for any tolerance level $\delta\in(0,1)$, the bound
    \begin{equation}
        \label{eq:uniform-catoni}
        L(Q)
        \le
        \frac{
            1-\exp\bigl\{
                -c_{n,\delta}(Q)-\gamma \hat{L}_n(Q)
            \bigr\}
        }{
            1-\exp(-\gamma)
        }
    \end{equation}
    holds with probability at least $1-\delta$, simultaneously for all $Q\ll\Pi_0$ and any scale parameter $\gamma>0$, where we denote
    \begin{equation*}
        c_{n,\delta}(Q)
        \coloneq
        \frac{\KL(Q\|\Pi_0)+\log\{\xi(n)/\delta\}}{n}
        .
    \end{equation*}
\end{lemma}
See Appendix~\ref{app:pf-lem-uniform-catoni} for the proof.
Lemma~\ref{lem:uniform-catoni} offers a uniform Catoni-style bound, enabling the derivation of our exact, temperature-dependent empirical lower confidence bounds for the expected loss.

\begin{lemma}[Range of the certified value score and loss]
\label{lem:Gamma-bounds}
    Under Assumption~\ref{ass:id}, for every measurable ITR $d$ and $\alpha\in\cA$, we have
    \begin{equation}
    \label{eq:Gamma-range}
        \Gamma_d^\alpha(Z)\in[1-\varepsilon^{-1},\varepsilon^{-1}]
        \quad\text{a.s.}
    \end{equation}
    As a direct consequence, defining the scaling constants
    \begin{equation*}
        c_\varepsilon\coloneq \varepsilon^{-1},
        \quad
        K_\varepsilon\coloneq 2\varepsilon^{-1}-1,
    \end{equation*}
    ensures that the transformed loss function
    \begin{equation*}
        \ell_{\mathrm{val}}(\theta;Z)
        \coloneq
        \frac{c_\varepsilon-\Gamma_{d_\beta}^\alpha(Z)}{K_\varepsilon}
        ,
        \quad
        \theta=(\beta,\alpha)\in\Theta
        ,
    \end{equation*}
    is bounded in $[0,1]$ almost surely.
\end{lemma}
See Appendix~\ref{app:pf-lem-Gamma-bounds} for the proof.
Lemma~\ref{lem:Gamma-bounds} guarantees that the certified value score and its induced loss function are almost surely bounded, satisfying the strict range requirements necessary to apply the aforementioned concentration inequalities.

\section{Proofs}
\label{app:sec:proofs}

\subsection{Proof of Theorem~\ref{thm:exact-certified-reduction}}
\label{app:pf-thm-exact-certified-reduction}

\begin{proof}[Proof of Theorem~\ref{thm:exact-certified-reduction}]
    Fix a measurable ITR $d$ and a nuisance parameter $\alpha\in\cA$.
    We first establish the exact certified value representation.
    By definition~\eqref{eq:Gamma-d}, we have
    \begin{equation*}
        \Gamma_d^\alpha(Z)
        =
        \sum_{a\in\{-1,1\}}
        \1\{d(X)=a\}
        \Biggl[
            \nu_{a,\alpha}(X)
            +
            \frac{\1\{A=a\}}{\pi(a\mid X)}
            \{\uR-\nu_{a,\alpha}(X)\}
        \Biggr]
        .
    \end{equation*}
    Fix $x\in\cX$ outside a $P_X$-null set, and let $a_d=d(x)$.
    Conditioning on $X=x$, we obtain
    \begin{equation*}
        \begin{split}
            \EE[\Gamma_d^\alpha(Z)\mid X=x]
            &=
            \nu_{a_d,\alpha}(x)
            +
            \EE\Biggl[
                \frac{\1\{A=a_d\}}{\pi(a_d\mid x)}
                \{\uR-\nu_{a_d,\alpha}(x)\}
                \Bigm| X=x
            \Biggr]\\
            &=
            \nu_{a_d,\alpha}(x)
            +
            \EE[\uR-\nu_{a_d,\alpha}(x)\mid X=x,A=a_d]\\
            &=
            \mu_{a_d}(x)
            .
        \end{split}
    \end{equation*}
    Therefore, $\EE[\Gamma_d^\alpha(Z)] = \EE[\mu_{d(X)}(X)]$.
    On the other hand,
    \begin{equation*}
        \begin{split}
            V_{\uR}(d)
            &=
            \EE\Biggl[
                \frac{\uR\1\{A=d(X)\}}{\pi(A\mid X)}
            \Biggr]\\
            &=
            \EE\Biggl[
                \EE\Biggl(
                    \frac{\uR\1\{A=d(X)\}}{\pi(A\mid X)}
                    \Bigm| X
                \Biggr)
            \Biggr]\\
            &=
            \EE[\mu_{d(X)}(X)]
            .
        \end{split}
    \end{equation*}
    Hence, $V_{\uR}(d)=\EE[\Gamma_d^\alpha(Z)]$, which concludes the first part of the proof.
    
    Next, we prove the exact weighted classification reduction.
    Define the observable quantity
    \begin{equation*}
        C_\alpha^{\mathrm{obs}}(Z)
        \coloneq
        \frac{1}{2}\Bigl\{
            \Gamma_1^\alpha(Z)+\Gamma_{-1}^\alpha(Z)+W_\alpha(Z)
        \Bigr\}
        ,
    \end{equation*}
    so that $C_\alpha^\sharp=\EE[C_\alpha^{\mathrm{obs}}(Z)]$.
    For any rule $d$, we can write
    \begin{equation*}
        \begin{split}
            \Gamma_d^\alpha(Z)
            &=
            \frac{\Gamma_1^\alpha(Z)+\Gamma_{-1}^\alpha(Z)}{2}
            +
            \frac{d(X)}{2}\{\Gamma_1^\alpha(Z)-\Gamma_{-1}^\alpha(Z)\}\\
            &=
            \frac{\Gamma_1^\alpha(Z)+\Gamma_{-1}^\alpha(Z)}{2}
            +
            \frac{d(X)}{2}D_\alpha(Z)\\
            &=
            \frac{\Gamma_1^\alpha(Z)+\Gamma_{-1}^\alpha(Z)}{2}
            +
            \frac{W_\alpha(Z)}{2}d(X)Y_\alpha(Z)
            .
        \end{split}
    \end{equation*}
    Since $d(X), Y_\alpha(Z) \in \{-1,1\}$, it follows that $d(X)Y_\alpha(Z) = 1 - 2\1\{Y_\alpha(Z) \neq d(X)\}$.
    Substituting this relationship into the previous display yields
    \begin{equation*}
        \Gamma_d^\alpha(Z)
        =
        C_\alpha^{\mathrm{obs}}(Z)
        -
        W_\alpha(Z)\1\{Y_\alpha(Z)\neq d(X)\}
        .
    \end{equation*}
    Taking expectations and applying the identity established in the first part of the proof, we find
    \begin{equation*}
        V_{\uR}(d)
        =
        \EE[\Gamma_d^\alpha(Z)]
        =
        C_\alpha^\sharp-\cR_{01}^\alpha(d)
        .
    \end{equation*}
    This confirms the exact reduction.
    
    Finally, observing that
    \begin{equation*}
        V_{\uR}(d)
        =
        \EE\bigl[
            \mu_1(X)\1\{d(X)=1\}
            +
            \mu_{-1}(X)\1\{d(X)=-1\}
        \bigr]
        ,
    \end{equation*}
    the integrand is maximized pointwise by selecting an action $a \in \argmax_{a\in\{-1,1\}}\mu_a(X)$.
    Therefore, a measurable ITR $d$ is Bayes-optimal for $V_{\uR}$ if and only if
    \begin{equation*}
        d(x)\in \argmax_{a\in\{-1,1\}}\mu_a(x)
        \quad\text{for }P_X\text{-a.e. }x
        .
    \end{equation*}
    In particular, the measurable selector $d_{\uR}^\ast(x)=\sgn\{\mu_1(x)-\mu_{-1}(x)\}$ (adopting the convention $\sgn(0)=1$) is strictly Bayes-optimal.
\end{proof}

\subsection{Proof of Proposition~\ref{prop:variance-optimality}}
\label{app:pf-prop-variance-optimality}

\begin{proof}[Proof of Proposition~\ref{prop:variance-optimality}]
    Fix an $x\in\cX$ outside a $P_X$-null set such that all requisite conditional moments are well-defined, and introduce the shorthand notations:
    \begin{equation*}
        p_a=p_a(x)
        ,
        \quad
        \mu_a=\mu_a(x)
        ,
        \quad
        \sigma_a^2=\sigma_a^2(x)
        ,
        \quad
        \nu_a=\nu_{a,\alpha}(x)
        .
    \end{equation*}
    
    To prove part (i), define $\Delta_a(Z)\coloneq \Gamma_a^\alpha(Z)-\mu_a$.
    Using \eqref{eq:Gamma-a}, we have
    \begin{equation*}
        \begin{split}
            \Delta_a(Z)
            &=
            \nu_a-\mu_a
            +
            \frac{\1\{A=a\}}{p_a}\{\uR-\nu_a\}
            \\
            &=
            \frac{\1\{A=a\}}{p_a}(\uR-\mu_a)
            +
            \Biggl(
                \frac{\1\{A=a\}}{p_a}-1
            \Biggr)(\mu_a-\nu_a)
            .
        \end{split}
    \end{equation*}
    Taking the conditional expectation given $X=x$, the first term vanishes because
    \begin{equation*}
        \EE\Biggl[
            \frac{\1\{A=a\}}{p_a}(\uR-\mu_a)
            \Bigm| X=x
        \Biggr]
        =
        \EE[\uR-\mu_a\mid X=x,A=a]
        =
        0
        ,
    \end{equation*}
    and the second term similarly vanishes because
    \begin{equation*}
        \EE\Biggl[
            \frac{\1\{A=a\}}{p_a}-1
            \Bigm| X=x
        \Biggr]
        =
        \frac{\PP(A=a\mid X=x)}{p_a}-1
        =
        0
        .
    \end{equation*}
    Thus, $\EE[\Gamma_a^\alpha(Z)\mid X=x]=\mu_a$.
    
    To compute the conditional variance, we first note that the conditional covariance between the two centered terms is zero:
    \begin{equation*}
        \begin{split}
            &\EE\Biggl[
                \frac{\1\{A=a\}}{p_a}(\uR-\mu_a)
                \Biggl(
                    \frac{\1\{A=a\}}{p_a}-1
                \Biggr)
                \Bigm| X=x
            \Biggr]
            \\
            &\quad=
            \Biggl(
                \frac{1}{p_a}-1
            \Biggr)
            \EE\Biggl[
                \frac{\1\{A=a\}}{p_a}(\uR-\mu_a)
                \Bigm| X=x
            \Biggr]
            \\
            &\quad=
            0
            .
        \end{split}
    \end{equation*}
    Therefore, the conditional variance simplifies to
    \begin{equation*}
        \begin{split}
            &\Var(\Gamma_a^\alpha(Z)\mid X=x)
            \\
            &=
            \EE\Biggl[
                \Biggl(
                    \frac{\1\{A=a\}}{p_a}(\uR-\mu_a)
                \Biggr)^2
                \Bigm| X=x
            \Biggr]
            +
            \EE\Biggl[
                \Biggl(
                    \frac{\1\{A=a\}}{p_a}-1
                \Biggr)^2
                \Bigm| X=x
            \Biggr]
            (\mu_a-\nu_a)^2
            .
        \end{split}
    \end{equation*}
    Evaluating the first term, we obtain
    \begin{equation*}
        \begin{split}
            \EE\Biggl[
                \Biggl(
                    \frac{\1\{A=a\}}{p_a}(\uR-\mu_a)
                \Biggr)^2
                \Bigm| X=x
            \Biggr]
            &=
            \frac{1}{p_a^2}
            \EE\bigl[
                \1\{A=a\}(\uR-\mu_a)^2
                \mid X=x
            \bigr]\\
            &=
            \frac{1}{p_a}
            \EE\bigl[
                (\uR-\mu_a)^2
                \mid X=x,A=a
            \bigr]
            =
            \frac{\sigma_a^2}{p_a}
            .
        \end{split}
    \end{equation*}
    For the second term, we have
    \begin{equation*}
        \EE\Biggl[
            \Biggl(
                \frac{\1\{A=a\}}{p_a}-1
            \Biggr)^2
            \Bigm| X=x
        \Biggr]
        =
        p_a\Biggl(\frac{1}{p_a}-1\Biggr)^2+(1-p_a)
        =
        \frac{(1-p_a)^2}{p_a}+1-p_a
        =
        \frac{1-p_a}{p_a}
        .
    \end{equation*}
    This establishes \eqref{eq:Gamma-mean-var}.
    The final statement in part (i) follows immediately because the second summand in \eqref{eq:Gamma-mean-var} is a strictly convex quadratic function of $\nu_a$, possessing a unique minimizer at $\nu_a=\mu_a$.
    
    For part (ii), let
    \begin{equation*}
        p=p(x)
        ,
        \quad
        q=q(x)
        ,
        \quad
        \delta_1\coloneq \mu_1(x)-\nu_{1,\alpha}(x)
        ,
        \quad
        \delta_{-1}\coloneq \mu_{-1}(x)-\nu_{-1,\alpha}(x)
        ,
    \end{equation*}
    and introduce the abbreviation $\Delta\coloneq \mu_1(x)-\mu_{-1}(x)$.
    On the event $\{A=1\}$, we have
    \begin{equation*}
        D_\alpha(Z)
        =
        \nu_{1,\alpha}(x)-\nu_{-1,\alpha}(x)
        +
        \frac{\uR-\nu_{1,\alpha}(x)}{p}
        ,
    \end{equation*}
    which implies
    \begin{equation*}
        D_\alpha(Z)-\Delta
        =
        \frac{\uR-\mu_1(x)}{p}
        +
        \frac{q}{p}\delta_1
        +
        \delta_{-1}
        .
    \end{equation*}
    Similarly, on the event $\{A=-1\}$, we have
    \begin{equation*}
        D_\alpha(Z)
        =
        \nu_{1,\alpha}(x)-\nu_{-1,\alpha}(x)
        -
        \frac{\uR-\nu_{-1,\alpha}(x)}{q}
        ,
    \end{equation*}
    yielding
    \begin{equation*}
        D_\alpha(Z)-\Delta
        =
        -\frac{\uR-\mu_{-1}(x)}{q}
        -
        \delta_1
        -
        \frac{p}{q}\delta_{-1}
        .
    \end{equation*}
    Taking the conditional expectation given $X=x$ directly yields \eqref{eq:D-mean}.
    Since $\EE[D_\alpha(Z)-\Delta\mid X=x]=0$, the conditional variance can be expressed as
    \begin{equation*}
        \begin{split}
            \Var(D_\alpha(Z)\mid X=x)
            &=
            p\Biggl\{
                \frac{\sigma_1^2(x)}{p^2}
                +
                \Biggl(
                    \frac{q}{p}\delta_1+\delta_{-1}
                \Biggr)^2
            \Biggr\}\\
            &\quad+
            q\Biggl\{
                \frac{\sigma_{-1}^2(x)}{q^2}
                +
                \Biggl(
                    \delta_1+\frac{p}{q}\delta_{-1}
                \Biggr)^2
            \Biggr\}\\
            &=
            \frac{\sigma_1^2(x)}{p}
            +
            \frac{\sigma_{-1}^2(x)}{q}
            +
            \frac{q}{p}\delta_1^2
            +
            \frac{p}{q}\delta_{-1}^2
            +
            2\delta_1\delta_{-1}
            .
        \end{split}
    \end{equation*}
    The last three terms factor neatly as
    \begin{equation*}
        \frac{q}{p}\delta_1^2
        +
        \frac{p}{q}\delta_{-1}^2
        +
        2\delta_1\delta_{-1}
        =
        \frac{(q\delta_1+p\delta_{-1})^2}{pq}
        ,
    \end{equation*}
    which establishes \eqref{eq:D-var}.
    Because the quadratic term is strictly nonnegative and vanishes if and only if $q\delta_1+p\delta_{-1}=0$, the minimizers are precisely the pairs of nuisance functions that satisfy the linear constraint \eqref{eq:variance-min-condition}.
    
    To prove part (iii), we enforce the restriction $\nu_{1,\alpha}=\nu_{-1,\alpha}=m_\alpha$.
    Under this condition, \eqref{eq:variance-min-condition} reduces to
    \begin{equation*}
        q(x)\{\mu_1(x)-m_\alpha(x)\}
        +
        p(x)\{\mu_{-1}(x)-m_\alpha(x)\}
        =
        0
        ,
    \end{equation*}
    which is algebraically equivalent to $m_\alpha(x)=q(x)\mu_1(x)+p(x)\mu_{-1}(x)$.
    This proves \eqref{eq:treatment-free-optimal}.
    The uniqueness of this minimizer follows from the strict convexity of the quadratic term in \eqref{eq:D-var} with respect to $m_\alpha(x)$ within this one-dimensional subclass.
    In the special case of balanced randomization where $p(x)=q(x)=1/2$, this expression simplifies to $m^\dagger(x)=\frac{\mu_1(x)+\mu_{-1}(x)}{2}$.
    Finally, recognizing that $\EE[\uR\mid X=x] = p(x)\mu_1(x)+q(x)\mu_{-1}(x)$, we see that under balanced randomization, the previous expression is exactly equal to $\EE[\uR\mid X=x]$.
    This concludes the proof of \eqref{eq:balanced-treatment-free-optimal}.
\end{proof}

\subsection{Proof of Theorem~\ref{thm:exact-pac-bayes}}
\label{app:pf-thm-exact-pac-bayes}

\begin{proof}[Proof of Theorem~\ref{thm:exact-pac-bayes}]
    For $\theta=(\beta,\alpha)\in\Theta$, define the expected and empirical certified value losses as
    \begin{equation*}
        L_{\mathrm{val}}(\theta)
        \coloneq
        \EE[\ell_{\mathrm{val}}(\theta;Z)]
        ,
        \quad
        \hat{L}_{\mathrm{val},n}(\theta)
        \coloneq
        \frac{1}{n}\sum_{i=1}^n \ell_{\mathrm{val}}(\theta;Z_i)
        .
    \end{equation*}
    For a posterior $Q\ll\Pi_0$, define their $Q$-averaged counterparts:
    \begin{equation*}
        L_{\mathrm{val}}(Q)
        \coloneq
        \EE_{\theta\sim Q}[L_{\mathrm{val}}(\theta)]
        ,
        \quad
        \hat{L}_{\mathrm{val},n}(Q)
        \coloneq
        \EE_{\theta\sim Q}[\hat{L}_{\mathrm{val},n}(\theta)]
        .
    \end{equation*}
    By the boundedness established in Lemma~\ref{lem:Gamma-bounds}, Fubini's theorem applies, yielding
    \begin{align}
        L_{\mathrm{val}}(Q)
        &=
        \EE_{\theta\sim Q}\Biggl[
            \frac{c_\varepsilon-\EE[\Gamma_{d_\beta}^\alpha(Z)]}{K_\varepsilon}
        \Biggr]
        =
        \frac{c_\varepsilon-V_{\uR}(d_Q)}{K_\varepsilon}
        ,
        \label{eq:Lval-linear}
        \\
        \hat{L}_{\mathrm{val},n}(Q)
        &=
        \EE_{\theta\sim Q}\Biggl[
            \frac{c_\varepsilon-\hat{V}_{\uR,n}(d_\theta)}{K_\varepsilon}
        \Biggr]
        =
        \frac{c_\varepsilon-\hat{V}_{\uR,n}(d_Q)}{K_\varepsilon}
        ,
        \label{eq:Lhatval-linear}
    \end{align}
    where the first equality utilizes Theorem~\ref{thm:exact-certified-reduction} pointwise in $\theta$.
    
    We first establish the fixed-temperature PAC-Bayes inequality.
    Fix $\theta\in\Theta$.
    By Lemma~\ref{lem:Gamma-bounds}, the random variable $\ell_{\mathrm{val}}(\theta;Z)$ takes values in $[0,1]$.
    Thus, Hoeffding's lemma yields, for every $s\in\RR$,
    \begin{equation*}
        \EE\Bigl[
            \exp\{s(\ell_{\mathrm{val}}(\theta;Z)-L_{\mathrm{val}}(\theta))\}
        \Bigr]
        \le
        \exp\Bigl(
            \frac{s^2}{8}
        \Bigr)
        .
    \end{equation*}
    Applying this to the empirical mean of the $n$ i.i.d.\ observations gives
    \begin{equation*}
        \EE\Bigl[
            \exp\{s(\hat{L}_{\mathrm{val},n}(\theta)-L_{\mathrm{val}}(\theta))\}
        \Bigr]
        \le
        \exp\Bigl(
            \frac{s^2}{8n}
        \Bigr)
        .
    \end{equation*}
    Integrating with respect to the prior and applying Fubini's theorem, we have
    \begin{equation*}
        \EE\Biggl[
            \int_\Theta
            \exp\{s(\hat{L}_{\mathrm{val},n}(\theta)-L_{\mathrm{val}}(\theta))\}
            \Pi_0(\dd\theta)
        \Biggr]
        \le
        \exp\Bigl(\frac{s^2}{8n}\Bigr)
        .
    \end{equation*}
    Markov's inequality then implies that, with probability at least $1-\delta$,
    \begin{equation*}
        \int_\Theta
        \exp\{s(\hat{L}_{\mathrm{val},n}(\theta)-L_{\mathrm{val}}(\theta))\}
        \Pi_0(\dd\theta)
        \le
        \delta^{-1}\exp\Bigl(
            \frac{s^2}{8n}
        \Bigr)
        .
    \end{equation*}
    Conditioning on this high-probability event, applying Lemma~\ref{lem:DV} with $S(\theta)=s\{\hat{L}_{\mathrm{val},n}(\theta)-L_{\mathrm{val}}(\theta)\}$ and $\tau=1$ yields, for every $Q\ll\Pi_0$,
    \begin{equation*}
        s\{\hat{L}_{\mathrm{val},n}(Q)-L_{\mathrm{val}}(Q)\}
        \le
        \KL(Q\|\Pi_0)+\log(1/\delta)+\frac{s^2}{8n}
        .
    \end{equation*}
    Setting $s=\eta n$, we obtain
    \begin{equation*}
        L_{\mathrm{val}}(Q)
        \le
        \hat{L}_{\mathrm{val},n}(Q)
        +
        \frac{\KL(Q\|\Pi_0)+\log(1/\delta)}{\eta n}
        +
        \frac{\eta}{8}
    \end{equation*}
    simultaneously for all $Q\ll\Pi_0$.
    Substituting Equations~\eqref{eq:Lval-linear} and \eqref{eq:Lhatval-linear} into this inequality and rearranging the terms yields
    \begin{equation*}
        V_{\uR}(d_Q)
        \ge
        \hat{V}_{\uR,n}(d_Q)
        -
        K_\varepsilon
        \Biggl\{
            \frac{\KL(Q\|\Pi_0)+\log(1/\delta)}{\eta n}
            +
            \frac{\eta}{8}
        \Biggr\}
        .
    \end{equation*}
    By Proposition~\ref{prop:lower-value-domination}, $V^\ast(d_Q)\ge V_{\uR}(d_Q)$, and therefore
    \begin{equation*}
        V^\ast(d_Q)
        \ge
        \hat{V}_{\uR,n}(d_Q)
        -
        K_\varepsilon
        \Biggl\{
            \frac{\KL(Q\|\Pi_0)+\log(1/\delta)}{\eta n}
            +
            \frac{\eta}{8}
        \Biggr\}
        ,
    \end{equation*}
    which confirms the claimed simultaneous lower bound.
    
    Next, we identify the posterior that maximizes the right-hand side for a fixed $\eta$.
    Because the terms involving $\log(1/\delta)$ and $\eta/8$ are independent of $Q$, maximizing the bound is strictly equivalent to maximizing the functional
    \begin{equation*}
        Q
        \mapsto
        \hat{V}_{\uR,n}(d_Q)-\frac{K_\varepsilon}{\eta n}\KL(Q\|\Pi_0).
    \end{equation*}
    We apply Lemma~\ref{lem:DV} with
    \begin{equation*}
        S(\theta)=\hat{V}_{\uR,n}(d_\theta)
        ,
        \quad
        \tau=\frac{\eta n}{K_\varepsilon}
        .
    \end{equation*}
    Since $\hat{V}_{\uR,n}(d_\theta)$ is bounded (as shown in Lemma~\ref{lem:Gamma-bounds}), the integrability condition is automatically satisfied.
    The variational identity~\eqref{eq:DV} dictates that the unique maximizer is
    \begin{equation*}
        \hat{Q}_\eta^{\mathrm{PROWL,val}}(\dd\theta)
        =
        \frac{
            \exp\{\eta n\hat{V}_{\uR,n}(d_\theta)/K_\varepsilon\}\Pi_0(\dd\theta)
        }{
            \int_\Theta \exp\{\eta n\hat{V}_{\uR,n}(d_t)/K_\varepsilon\}\Pi_0(\dd t)
        }
        .
    \end{equation*}
    Finally, noting the algebraic relation
    \begin{equation*}
        \sum_{i=1}^n \ell_{\mathrm{val}}(\theta;Z_i)
        =
        \frac{1}{K_\varepsilon}\sum_{i=1}^n\{c_\varepsilon-\Gamma_{d_\beta}^\alpha(Z_i)\}
        =
        \frac{nc_\varepsilon}{K_\varepsilon}
        -
        \frac{n}{K_\varepsilon}\hat{V}_{\uR,n}(d_\theta)
        ,
    \end{equation*}
    we have
    \begin{equation*}
        \exp\Biggl\{
            -\eta\sum_{i=1}^n\ell_{\mathrm{val}}(\theta;Z_i)
        \Biggr\}
        =
        \exp\Bigl(-\eta nc_\varepsilon/K_\varepsilon\Bigr)
        \exp\Bigl\{
            \eta n\hat{V}_{\uR,n}(d_\theta)/K_\varepsilon
        \Bigr\}
        .
    \end{equation*}
    Because the prefactor $\exp(-\eta nc_\varepsilon/K_\varepsilon)$ is independent of $\theta$, it cancels out upon normalization, leaving
    \begin{equation*}
        \hat{Q}_\eta^{\mathrm{PROWL,val}}(\dd\theta)
        \propto
        \exp\Biggl\{
            -\eta\sum_{i=1}^n \ell_{\mathrm{val}}(\theta;Z_i)
        \Biggr\}\Pi_0(\dd\theta)
        .
    \end{equation*}
    This coincides exactly with the general Bayes update associated with the loss $\ell_{\mathrm{val}}$, as formulated by \citet{bissiri2016general}.
\end{proof}

\subsection{Proof of Proposition~\ref{prop:temperature-selection}}
\label{app:pf-prop-temperature-selection}

\begin{proof}[Proof of Proposition~\ref{prop:temperature-selection}]
    We first apply Lemma~\ref{lem:uniform-catoni} to the bounded loss $\ell_{\mathrm{val}}(\theta;Z)\in[0,1]$, whose boundedness is guaranteed by Lemma~\ref{lem:Gamma-bounds}.
    With probability at least $1-\delta$, we have
    \begin{equation*}
        L_{\mathrm{val}}(Q)
        \le
        \frac{
            1-\exp\{-c_{n,\delta}(Q)-\gamma \hat{L}_{\mathrm{val},n}(Q)\}
        }{
            1-\exp(-\gamma)
        }
    \end{equation*}
    simultaneously for all posteriors $Q\ll\Pi_0$ and all scale parameters $\gamma>0$.
    Substituting the linear relations \eqref{eq:Lval-linear} and \eqref{eq:Lhatval-linear} into this inequality yields
    \begin{equation*}
        V_{\uR}(d_Q)
        \ge
        c_\varepsilon
        -
        K_\varepsilon
        \frac{
            1-\exp\{-c_{n,\delta}(Q)-\gamma \hat{L}_{\mathrm{val},n}(Q)\}
        }{
            1-\exp(-\gamma)
        }
        =
        \mathrm{LCB}_{n,\delta}^{\mathrm{val}}(Q;\gamma)
        .
    \end{equation*}
    Invoking Proposition~\ref{prop:lower-value-domination}, we obtain the simultaneous lower confidence bound:
    \begin{equation*}
        V^\ast(d_Q)\ge V_{\uR}(d_Q)\ge \mathrm{LCB}_{n,\delta}^{\mathrm{val}}(Q;\gamma)
        .
    \end{equation*}
    
    Now, fix $\gamma>0$.
    Because $1-e^{-\gamma}>0$, the mapping
    \begin{equation*}
        t\mapsto
        c_\varepsilon
        -
        K_\varepsilon\frac{1-e^{-t}}{1-e^{-\gamma}}
    \end{equation*}
    is strictly monotonically decreasing on $\RR$.
    Therefore, maximizing the objective $Q\mapsto \mathrm{LCB}_{n,\delta}^{\mathrm{val}}(Q;\gamma)$ is strictly equivalent to minimizing $t(Q)\coloneq c_{n,\delta}(Q)+\gamma \hat{L}_{\mathrm{val},n}(Q)$, or equivalently, to maximizing
    \begin{equation*}
        -t(Q)
        =
        -\frac{\KL(Q\|\Pi_0)+\log\{\xi(n)/\delta\}}{n}
        -\gamma \hat{L}_{\mathrm{val},n}(Q)
        .
    \end{equation*}
    Discarding the $Q$-independent additive constant $-\log\{\xi(n)/\delta\}/n$ and substituting \eqref{eq:Lhatval-linear}, the problem reduces to maximizing
    \begin{equation*}
        \hat{V}_{\uR,n}(d_Q)-\frac{K_\varepsilon}{\gamma n}\KL(Q\|\Pi_0)
        .
    \end{equation*}
    By Theorem~\ref{thm:exact-pac-bayes}, the unique maximizer of this functional is precisely $\hat{Q}_\gamma^{\mathrm{PROWL,val}}$
    
    Finally, consider a finite grid $\cT_n\subset(0,\infty)$ and let
    \begin{equation*}
        \hat{\gamma}\in
        \argmax_{\gamma\in\cT_n}
        \mathrm{LCB}_{n,\delta}^{\mathrm{val}}
        \bigl(
            \hat{Q}_\gamma^{\mathrm{PROWL,val}};\gamma
        \bigr)
        .
    \end{equation*}
    Because the high-probability event established above holds uniformly for all $Q\ll\Pi_0$ and all $\gamma>0$, it holds in particular for the data-dependent choices
    \begin{equation*}
        Q=\hat{Q}_{\hat{\gamma}}^{\mathrm{PROWL,val}}
        \quad\text{and}\quad
        \gamma=\hat{\gamma}
        .
    \end{equation*}
    Consequently, on that event, it is guaranteed that
    \begin{equation*}
        \begin{split}
            V^\ast(d_{\hat{Q}_{\hat{\gamma}}^{\mathrm{PROWL,val}}})
            &\ge
            \mathrm{LCB}_{n,\delta}^{\mathrm{val}}
            \bigl(
                \hat{Q}_{\hat{\gamma}}^{\mathrm{PROWL,val}};\hat{\gamma}
            \bigr)
            \\
            &=
            \max_{\gamma\in\cT_n}
            \mathrm{LCB}_{n,\delta}^{\mathrm{val}}
            \bigl(
                \hat{Q}_\gamma^{\mathrm{PROWL,val}};\gamma
            \bigr)
            ,
        \end{split}
    \end{equation*}
    which concludes the proof of the temperature-selection guarantee.
\end{proof}

\subsection{Proof of Proposition~\ref{prop:hinge-calibration}}
\label{app:pf-prop-hinge-calibration}

\begin{proof}[Proof of Proposition~\ref{prop:hinge-calibration}]
    To establish part (i), observe that $W_\alpha(Z)Y_\alpha(Z)=D_\alpha(Z)$ by definition. Consequently, for $P_X$-almost every $x$, we have
    \begin{equation*}
        \begin{split}
            u_\alpha(x)-v_\alpha(x)
            &=
            \EE\bigl[
                W_\alpha(Z)Y_\alpha(Z)
                \mid X=x
            \bigr]\\
            &=
            \EE[D_\alpha(Z)\mid X=x]
            =
            \mu_1(x)-\mu_{-1}(x)
            ,
        \end{split}
    \end{equation*}
    where the final equality follows from \eqref{eq:D-mean} in Proposition~\ref{prop:variance-optimality}.
    
    For part (ii), fix $x$ outside a $P_X$-null set and abbreviate $u=u_\alpha(x)$ and $v=v_\alpha(x)$.
    For a scalar score value $t\in\RR$, the conditional hinge risk at $x$ is given by $\phi_x(t)\coloneq u(1-t)_+ + v(1+t)_+$.
    A direct piecewise expansion yields
    \begin{equation*}
        \phi_x(t)
        =
        \begin{cases}
            u(1-t), & t\le -1,\\
            u(1-t)+v(1+t), & -1\le t\le 1,\\
            v(1+t), & t\ge 1.
        \end{cases}
    \end{equation*}
    If $u>v$, the function $\phi_x(t)$ is strictly decreasing on the interval $[-1,1]$ (with a slope of $v-u<0$) and strictly increasing on $[1,\infty)$ (with a slope of $v\ge 0$); thus, the unique minimizer is $t=1$.
    Conversely, if $u<v$, a symmetric argument shows that the unique minimizer is $t=-1$.
    If $u=v$, the function is constant on $[-1,1]$, meaning any $t\in[-1,1]$ minimizes $\phi_x$.
    Therefore, the pointwise measurable choice $g_\alpha^\dagger(x)$ defined previously is a conditional risk minimizer for almost every $x$.
    Because the expected hinge risk satisfies
    \begin{equation*}
        \cR_{\mathrm{h}}^\alpha(g)
        =
        \EE\bigl[
            \phi_X(g(X))
        \bigr]
        ,
    \end{equation*}
    it follows that $g_\alpha^\dagger$ universally minimizes $\cR_{\mathrm{h}}^\alpha(g)$ across all measurable scores.
    
    By part (i), we have
    \begin{equation*}
        \operatorname{sgn}\{u_\alpha(x)-v_\alpha(x)\}
        =
        \operatorname{sgn}\{\mu_1(x)-\mu_{-1}(x)\}
        \quad\text{for }P_X\text{-a.e. }x
        .
    \end{equation*}
    Hence, the induced rule $d_{g_\alpha^\dagger}(x)$ selects an action in $\argmax_{a\in\{-1,1\}}\mu_a(x)$ for $P_X$-almost every $x$.
    By Theorem~\ref{thm:exact-certified-reduction}, such a rule is strictly Bayes-optimal for the certified value $V_{\uR}$, proving Fisher consistency.
    
    For part (iii), again fix $x$ outside a $P_X$-null set and retain the abbreviations $u=u_\alpha(x)$ and $v=v_\alpha(x)$.
    Define the conditional exact weighted classification risk for a score $t$ as $\psi_x(t)\coloneq u\1\{t<0\}+v\1\{t\ge 0\}$, noting that our convention $\sgn(0)=1$ dictates $d_g(x)=1$ when $g(x)=0$.
    Its minimum is $\psi_x^\ast=\min\{u,v\}$, while the minimum conditional hinge risk is $\phi_x^\ast=2\min\{u,v\}$.
    We claim that the pointwise excess classification risk is strictly dominated by the excess hinge risk:
    \begin{equation}
    \label{eq:pointwise-domination}
        \psi_x(t)-\psi_x^\ast
        \le
        \phi_x(t)-\phi_x^\ast
        \quad\text{for all }t\in\RR
        .
    \end{equation}
    
    Suppose first that $u\ge v$.
    If $t\ge 0$, then $\psi_x(t)=v=\psi_x^\ast$, rendering the left-hand side of \eqref{eq:pointwise-domination} zero, and the inequality holds trivially.
    If $-1\le t<0$, then $\psi_x(t)-\psi_x^\ast=u-v$, while
    \begin{equation*}
        \phi_x(t)-\phi_x^\ast
        =
        u(1-t)+v(1+t)-2v
        =
        (u-v)(1-t)
        \ge
        u-v
        .
    \end{equation*}
    If $t<-1$, we again have $\psi_x(t)-\psi_x^\ast=u-v$, whereas
    \begin{equation*}
        \phi_x(t)-\phi_x^\ast
        =
        u(1-t)-2v
        =
        (u-v)+\{u(-t)-v\}
        \ge
        u-v
        ,
    \end{equation*}
    since $-t>1$ and $u\ge v$ imply $u(-t)-v\ge u-v\ge 0$.
    Thus, \eqref{eq:pointwise-domination} holds when $u\ge v$.
    
    The case where $u<v$ follows by a symmetric argument.
    If $t<0$, then $\psi_x(t)=u=\psi_x^\ast$, making the inequality trivial.
    If $0\le t\le 1$, we have
    \begin{equation*}
        \psi_x(t)-\psi_x^\ast=v-u
        ,
        \quad
        \phi_x(t)-\phi_x^\ast
        =
        u(1-t)+v(1+t)-2u
        =
        (v-u)(1+t)
        \ge
        v-u
        .
    \end{equation*}
    If $t>1$, we similarly find
    \begin{equation*}
        \psi_x(t)-\psi_x^\ast=v-u
        ,
        \quad
        \phi_x(t)-\phi_x^\ast
        =
        v(1+t)-2u
        =
        (v-u)+\{vt-u\}
        \ge
        v-u
        ,
    \end{equation*}
    because $t>1$ and $v>u$ strictly guarantee $vt-u\ge v-u>0$.
    Consequently, \eqref{eq:pointwise-domination} holds unconditionally for all $x$ and all $t$.
    
    Substituting $t=g(x)$ and integrating \eqref{eq:pointwise-domination} with respect to $P_X$ yields
    \begin{equation*}
        \cR_{01}^\alpha(g)-\inf_h \cR_{01}^\alpha(h)
        \le
        \cR_{\mathrm{h}}^\alpha(g)-\inf_h \cR_{\mathrm{h}}^\alpha(h)
        .
    \end{equation*}
    Finally, by the equivalence established in Theorem~\ref{thm:exact-certified-reduction},
    \begin{equation*}
        V_{\uR}(d_{\uR}^\ast)-V_{\uR}(d_g)
        =
        \cR_{01}^\alpha(g)-\inf_h \cR_{01}^\alpha(h)
        ,
    \end{equation*}
    which verifies the excess-risk domination bound \eqref{eq:excess-hinge}.
\end{proof}

\subsection{Proof of Corollary~\ref{cor:family-selection}}
\label{app:pf-cor-family-selection}

\begin{proof}[Proof of Corollary~\ref{cor:family-selection}]
    Consider any finite, data-dependent family of posteriors $\mathcal{Q}_n=\{\hat{Q}_\lambda:\lambda\in\Lambda_n\}$ such that $\hat{Q}_\lambda\ll\Pi_0$ almost surely for every $\lambda\in\Lambda_n$.
    Conditional on the high-probability event established in Proposition~\ref{prop:temperature-selection}, we have
    \begin{equation*}
        V^\ast(d_Q)\ge \mathrm{LCB}_{n,\delta}^{\mathrm{val}}(Q;\gamma)
        \quad
        \text{simultaneously for all }Q\ll\Pi_0\text{ and all }\gamma>0
        .
    \end{equation*}
    Because this event holds uniformly over all absolutely continuous posteriors, the inequality remains perfectly valid when evaluated at the random, data-dependent choices $Q=\hat{Q}_\lambda$.
    Thus, on the same event,
    \begin{equation*}
        V^\ast(d_{\hat{Q}_\lambda})
        \ge
        \mathrm{LCB}_{n,\delta}^{\mathrm{val}}(\hat{Q}_\lambda;\gamma)
        \quad
        \text{for all }(\lambda,\gamma)\in\Lambda_n\times\cT_n
        .
    \end{equation*}
    Let the empirical maximizer be defined as
    \begin{equation*}
        (\hat\lambda,\hat{\gamma})
        \in
        \argmax_{(\lambda,\gamma)\in\Lambda_n\times\cT_n}
        \mathrm{LCB}_{n,\delta}^{\mathrm{val}}(\hat{Q}_\lambda;\gamma)
        .
    \end{equation*}
    Evaluating the preceding simultaneous inequality at $(\hat\lambda,\hat{\gamma})$ directly yields
    \begin{equation*}
        V^\ast(d_{\hat{Q}_{\hat\lambda}})
        \ge
        \mathrm{LCB}_{n,\delta}^{\mathrm{val}}(\hat{Q}_{\hat\lambda};\hat{\gamma})
        =
        \max_{(\lambda,\gamma)\in\Lambda_n\times\cT_n}
        \mathrm{LCB}_{n,\delta}^{\mathrm{val}}(\hat{Q}_\lambda;\gamma)
        ,
    \end{equation*}
    which completes the proof.
\end{proof}

\subsection{Proof of Corollary~\ref{cor:learned-certificate}}
\label{app:pf-cor-learned-certificate}

\begin{proof}[Proof of Corollary~\ref{cor:learned-certificate}]
    Let us condition on the $\sigma$-algebra $\sigma(\cD_m^{\mathrm{cal}})$.
    By assumption, on the event $\cE_{\mathrm{cert}}$, the learned certificate $\hat{U}$ is measurable, takes values in $[0,1]$, and satisfies $R^\ast\ge R-\hat{U}(X,A,R)$ almost surely under the policy-learning distribution.
    Consequently, on $\cE_{\mathrm{cert}}$, the induced random variable $\uR_{\hat{U}}\coloneq (R-\hat{U})_+$ is measurable, bounded in $[0,1]$, and robustly satisfies $R^\ast\ge \uR_{\hat{U}}$ almost surely.
    Therefore, every theoretical construction detailed in Sections~\ref{subsec:reward-uncertainty} through \ref{subsec:learning-rate} remains structurally valid upon substituting $U$ with $\hat{U}$ and $\uR$ with $\uR_{\hat{U}}$.

    Since the auxiliary sample $\cD_m^{\mathrm{cal}}$ is completely independent of the policy-learning sample $Z_{1:n}$, the conditional distribution of $Z_{1:n}$ given $\cD_m^{\mathrm{cal}}$ remains i.i.d., preserving the original policy-learning distribution.
    Thus, conditional on both $\cD_m^{\mathrm{cal}}$ and $\cE_{\mathrm{cert}}$, all preceding PAC-Bayes derivations apply verbatim.

    To formalize the unconditional joint-probability guarantee, let $\cE_{\mathrm{PB}}$ denote the event that any specific $1-\delta$ PAC-Bayes bound established above holds when substituting $U$ with $\hat{U}$.
    We then have
    \begin{equation*}
        \PP(\cE_{\mathrm{PB}}\mid \cD_m^{\mathrm{cal}})\ge 1-\delta
        \quad\text{almost surely on }\cE_{\mathrm{cert}}
        .
    \end{equation*}
    Integrating over the distribution of $\cD_m^{\mathrm{cal}}$, we obtain
    \begin{equation*}
        \begin{split}
            \PP(\cE_{\mathrm{cert}}\cap \cE_{\mathrm{PB}})
            &=
            \EE\Bigl[
                \1\{\cE_{\mathrm{cert}}\}
                \PP(\cE_{\mathrm{PB}}\mid \cD_m^{\mathrm{cal}})
            \Bigr]\\
            &\ge
            (1-\delta)\PP(\cE_{\mathrm{cert}})\\
            &\ge
            (1-\delta)(1-\alpha_{\mathrm{cert}})
            .
        \end{split}
    \end{equation*}
    This establishes the claimed joint-probability bound, verifying that all preceding results—including the fixed-temperature bounds and the automated lower-confidence-bound selection—scale seamlessly to the setting of actively learned certificates.
\end{proof}

\subsection{Proof of Proposition~\ref{prop:ips}}
\label{app:pf-prop-ips}

\begin{proof}[Proof of Proposition~\ref{prop:ips}]
    We establish the identity for the target value $V^\ast(d)$; the derivation for the proxy value $V(d)$ follows symmetrically by replacing $R^\ast$ with $R$.
    Applying the law of iterated expectations, we have
    \begin{equation*}
        \EE\Biggl[
            \frac{R^\ast\1\{A=d(X)\}}{\pi(A\mid X)}
        \Biggr]
        =
        \EE\Biggl[
            \EE\Biggl(
                \frac{R^\ast\1\{A=d(X)\}}{\pi(A\mid X)}
                \Bigm| X
            \Biggr)
        \Biggr]
        .
    \end{equation*}
    Let $x\in\cX$ be a point outside a $P_X$-null set where the relevant conditional expectations are well-defined, and denote the deterministic action $a_d=d(x)\in\{-1,1\}$.
    Conditioning on $X=x$ yields
    \begin{equation*}
        \begin{split}            
            \EE\Biggl(
                \frac{R^\ast\1\{A=d(X)\}}{\pi(A\mid X)}
                \Bigm| X=x
            \Biggr)
            &=
            \sum_{a\in\{-1,1\}}
            \1\{a_d=a\}
            \EE\Biggl(
                \frac{R^\ast\1\{A=a\}}{\pi(a\mid x)}
                \Bigm| X=x
            \Biggr)
            .
        \end{split}
    \end{equation*}
    By causal consistency (Assumption~\ref{ass:id}), $R^\ast=(R^\ast)^a$ almost surely on the event $\{A=a\}$, which implies $R^\ast\1\{A=a\}=(R^\ast)^a\1\{A=a\}$ almost surely.
    Hence,
    \begin{equation*}
        \begin{split}
            \EE\Biggl(
                \frac{R^\ast\1\{A=a\}}{\pi(a\mid x)}
                \Bigm| X=x
            \Biggr)
            &=
            \EE\Biggl(
                \frac{(R^\ast)^a\1\{A=a\}}{\pi(a\mid x)}
                \Bigm| X=x
            \Biggr)\\
            &=
            \frac{1}{\pi(a\mid x)}
            \EE\Bigl[
                (R^\ast)^a\1\{A=a\}
                \mid X=x
            \Bigr]
            .
        \end{split}
    \end{equation*}
    Invoking conditional ignorability (Assumption~\ref{ass:id}), we have $(R^\ast)^a\indep A\mid X$, which factors the conditional expectation as
    \begin{equation*}
        \begin{split}
            \EE\Bigl[
                (R^\ast)^a\1\{A=a\}
                \mid X=x
            \Bigr]
            &=
            \EE\bigl[
                (R^\ast)^a\mid X=x
            \bigr]\PP(A=a\mid X=x)
            \\
            &=
            \EE\bigl[
                (R^\ast)^a\mid X=x
            \bigr]\pi(a\mid x)
            .
        \end{split}
    \end{equation*}
    Substituting this back, we obtain
    \begin{equation*}
        \EE\Biggl(
            \frac{R^\ast\1\{A=a\}}{\pi(a\mid x)}
            \Bigm| X=x
        \Biggr)
        =
        \EE\bigl[
            (R^\ast)^a\mid X=x
        \bigr]
        ,
    \end{equation*}
    and consequently,
    \begin{equation*}
        \EE\Biggl(
            \frac{R^\ast\1\{A=d(X)\}}{\pi(A\mid X)}
            \Bigm| X=x
        \Biggr)
        =
        \EE\bigl[
            (R^\ast)^{a_d}\mid X=x
        \bigr]
        =
        \EE\bigl[
            (R^\ast)^{d(x)}\mid X=x
        \bigr]
        .
    \end{equation*}
    Finally, taking the outer expectation over the marginal distribution of $X$ directly recovers the target value:
    \begin{equation*}
        \EE\Biggl[
            \frac{R^\ast\1\{A=d(X)\}}{\pi(A\mid X)}
        \Biggr]
        =
        \EE\bigl[
            (R^\ast)^{d(X)}
        \bigr]
        =
        V^\ast(d)
        .
    \end{equation*}
    This completes the proof of the identity \eqref{eq:ips}.
\end{proof}

\subsection{Proof of Proposition~\ref{prop:lower-value-domination}}
\label{app:pf-prop-lower-value-domination}

\begin{proof}[Proof of Proposition~\ref{prop:lower-value-domination}]
    By Assumption~\ref{ass:certificate}, $R^\ast\ge R-U$ almost surely.
    Given that $R^\ast\in[0,1]$, it naturally follows that $R^\ast\ge (R-U)_+=\uR$ almost surely.
    Consequently, we have
    \begin{equation*}
        \frac{R^\ast\1\{A=d(X)\}}{\pi(A\mid X)}
        \ge
        \frac{\uR\1\{A=d(X)\}}{\pi(A\mid X)}
        \quad\text{a.s.}
    \end{equation*}
    Taking the expectation of both sides and invoking Proposition~\ref{prop:ips} yields
    \begin{equation*}
        V^\ast(d)
        =
        \EE\Biggl[
            \frac{R^\ast\1\{A=d(X)\}}{\pi(A\mid X)}
        \Biggr]
        \ge
        \EE\Biggl[
            \frac{\uR\1\{A=d(X)\}}{\pi(A\mid X)}
        \Biggr]
        =
        V_{\uR}(d)
        ,
    \end{equation*}
    which establishes \eqref{eq:lower-value-domination}.
    
    Furthermore, if $U\le R$ almost surely, then $\uR=R-U$ almost surely. In this case, we obtain
    \begin{equation*}
        \begin{split}
            V_{\uR}(d)
            &=
            \EE\Biggl[
                \frac{(R-U)\1\{A=d(X)\}}{\pi(A\mid X)}
            \Biggr]
            \\
            &=
            \EE\Biggl[
                \frac{R\1\{A=d(X)\}}{\pi(A\mid X)}
            \Biggr]
            -
            \EE\Biggl[
                \frac{U\1\{A=d(X)\}}{\pi(A\mid X)}
            \Biggr]
            \\
            &=
            V(d)
            -
            \EE\Biggl[
                \frac{U\1\{A=d(X)\}}{\pi(A\mid X)}
            \Biggr]
            ,
        \end{split}
    \end{equation*}
    where the final equality follows from Proposition~\ref{prop:ips}.
    Combining this identity with \eqref{eq:lower-value-domination} yields the desired bound \eqref{eq:lower-value-domination-unclipped}.
\end{proof}

\subsection{Proof of Proposition~\ref{prop:special-cases}}
\label{app:pf-prop-special-cases}

\begin{proof}[Proof of Proposition~\ref{prop:special-cases}]
    To establish part (i), suppose $\nu_{1,\alpha}\equiv \nu_{-1,\alpha}\equiv 0$. Then, by \eqref{eq:Gamma-a}, we have
    \begin{equation*}
        \Gamma_a^\alpha(Z)
        =
        0+\frac{\1\{A=a\}}{\pi(a\mid X)}\{\uR-0\}
        =
        \frac{\uR\1\{A=a\}}{\pi(a\mid X)}
        .
    \end{equation*}
    Consequently,
    \begin{equation*}
        D_\alpha(Z)
        =
        \Gamma_1^\alpha(Z)-\Gamma_{-1}^\alpha(Z)
        =
        \frac{\uR\1\{A=1\}}{\pi(1\mid X)}
        -
        \frac{\uR\1\{A=-1\}}{\pi(-1\mid X)}
        =
        \frac{A\uR}{\pi(A\mid X)}
        .
    \end{equation*}
    Taking the absolute value yields the expression for $W_\alpha(Z)$.
    On the event $\{\uR>0\}$, the sign of $D_\alpha(Z)$ matches the sign of $A$, which implies $Y_\alpha(Z)=A$.
    Conversely, on the event $\{\uR=0\}$, we have $W_\alpha(Z)=0$; thus, the corresponding pseudo-label does not contribute to the weighted $0$--$1$ risk.
    
    Regarding part (ii), if we restrict $\nu_{1,\alpha}=\nu_{-1,\alpha}=m_\alpha$, we obtain
    \begin{equation*}
        D_\alpha(Z)
        =
        m_\alpha(X)-m_\alpha(X)
        +
        \frac{A}{\pi(A\mid X)}\{\uR-m_\alpha(X)\}
        =
        \frac{A\{\uR-m_\alpha(X)\}}{\pi(A\mid X)}
        ,
    \end{equation*}
    where the second equality follows from the definition in \eqref{eq:DYW}.
    The corresponding formulas for $W_\alpha(Z)$ and $Y_\alpha(Z)$ follow immediately from their definitions.
    
    For part (iii), fix $a\in\{-1,1\}$ and an $x\in\cX$ outside a $P_X$-null set.
    Invoking \eqref{eq:Gamma-a}, we can write
    \begin{equation*}
        \begin{split}
            \EE[\Gamma_a^\alpha(Z)\mid X=x]
            &=
            \nu_{a,\alpha}(x)
            +
            \EE\Biggl[
                \frac{\1\{A=a\}}{\pi(a\mid x)}
                \{\uR-\nu_{a,\alpha}(x)\}
                \Bigm| X=x
            \Biggr]\\
            &=
            \nu_{a,\alpha}(x)
            +
            \EE[\uR-\nu_{a,\alpha}(x)\mid X=x,A=a]\\
            &=
            \mu_a(x)
            .
        \end{split}
    \end{equation*}
    This confirms the unbiasedness claim.
\end{proof}

\subsection{Proof of Lemma~\ref{lem:DW-bounds}}
\label{app:pf-lem-DW-bounds}

\begin{proof}[Proof of Lemma~\ref{lem:DW-bounds}]
    By the definition in \eqref{eq:DYW}, we can express the certified advantage as
    \begin{equation*}
        D_\alpha(Z)
        =
        \nu_{1,\alpha}(X)-\nu_{-1,\alpha}(X)
        +
        \frac{A}{\pi(A\mid X)}\{\uR-\nu_{A,\alpha}(X)\}
        .
    \end{equation*}
    Since $\nu_{1,\alpha}(X)$, $\nu_{-1,\alpha}(X)$, and $\uR$ are bounded in $[0,1]$ almost surely, it naturally follows that
    \begin{equation*}
        |\nu_{1,\alpha}(X)-\nu_{-1,\alpha}(X)|\le 1,
        \quad
        |\uR-\nu_{A,\alpha}(X)|\le 1
        \quad\text{a.s.}
    \end{equation*}
    Furthermore, the strict overlap condition (Assumption~\ref{ass:id}) guarantees $\pi(A\mid X)\ge\varepsilon$ almost surely, which implies
    \begin{equation*}
        \Biggl|
            \frac{A}{\pi(A\mid X)}\{\uR-\nu_{A,\alpha}(X)\}
        \Biggr|
        \le
        \varepsilon^{-1}
        \quad\text{a.s.}
    \end{equation*}
    Applying the triangle inequality directly yields \eqref{eq:D-bound}.
    Finally, the bound in \eqref{eq:W-bound} is an immediate consequence of the definition $W_\alpha(Z)\coloneq |D_\alpha(Z)|$.
\end{proof}

\subsection{Proof of Proposition~\ref{prop:surrogate-optimal-posterior}}
\label{app:pf-prop-surrogate-optimal-posterior}

\begin{proof}[Proof of Proposition~\ref{prop:surrogate-optimal-posterior}]
    Since the hinge loss $\ell_{\mathrm{h}}(\theta;Z_i)$ is non-negative and finite for all $(\theta,Z_i)$, the normalizing constant in \eqref{eq:surrogate-posterior-variational} is guaranteed to be strictly positive and bounded by $1$.
    Applying the infimum formulation of the Donsker--Varadhan variational principle (Lemma~\ref{lem:DV}) with
    \begin{equation*}
        T(\theta)=\hat{L}_{\mathrm{h},n}(\theta)
        \quad\text{and}\quad
        \tau=\lambda n
        ,
    \end{equation*}
    yields
    \begin{equation*}
    \inf_{Q\ll\Pi_0}
        \Biggl\{
            \EE_{\theta\sim Q}\bigl[
                \hat{L}_{\mathrm{h},n}(\theta)
            \bigr]
            +
            \frac{1}{\lambda n}\KL(Q\|\Pi_0)
        \Biggr\}
        =
        -\frac{1}{\lambda n}
        \log\int_\Theta
        \exp\{-\lambda n\hat{L}_{\mathrm{h},n}(\theta)\}
        \Pi_0(\dd\theta)
        .
    \end{equation*}
    The unique minimizer of this variational problem is precisely the Gibbs measure defined in \eqref{eq:surrogate-posterior-variational}.
    The equivalence with the practical posterior defined in Section~\ref{subsec:learning-rate} is readily confirmed by substituting the identity
    \begin{equation*}
        n\hat{L}_{\mathrm{h},n}(\theta)=\sum_{i=1}^n \ell_{\mathrm{h}}(\theta;Z_i)
        .
    \end{equation*}
\end{proof}

\subsection{Proof of Corollary~\ref{cor:MAP-penalized}}
\label{app:pf-cor-MAP-penalized}

\begin{proof}[Proof of Corollary~\ref{cor:MAP-penalized}]
    The expression for the posterior density follows immediately by multiplying the prior density by the exponentiated empirical risk factor in \eqref{eq:surrogate-posterior-variational}.
    Taking the negative logarithm reveals that computing the MAP estimator is strictly equivalent to minimizing the regularized objective:
    \begin{equation*}
        (\beta,\alpha)
        \mapsto
        \lambda\sum_{i=1}^n \ell_{\mathrm{h}}((\beta,\alpha);Z_i)+\lambda_0 J(\beta,\alpha)
        ,
    \end{equation*}
    which directly translates to the empirical risk minimization problem \eqref{eq:MAP-penalized} upon division by $\lambda n$.
\end{proof}

\subsection{Proof of Lemma~\ref{lem:DV}}
\label{app:pf-lem-DV}

\begin{proof}[Proof of Lemma~\ref{lem:DV}]
    Define the tilted probability measure
    \begin{equation*}
        \widetilde{\Pi}_\tau(\dd\theta)
        \coloneq
        \frac{e^{\tau S(\theta)}\Pi_0(\dd\theta)}
        {\int_\Theta e^{\tau S(t)}\Pi_0(\dd t)}
        .
    \end{equation*}
    For any absolutely continuous measure $Q\ll\Pi_0$, we can express the Kullback--Leibler divergence as
    \begin{equation*}
        \begin{split}
            \KL(Q\|\widetilde{\Pi}_\tau)
            &=
            \int_\Theta
            \log\Biggl(
                \frac{\dd Q}{\dd \widetilde{\Pi}_\tau}
            \Biggr)
            Q(\dd\theta)\\
            &=
            \int_\Theta
            \log\Biggl(
                \frac{\dd Q}{\dd\Pi_0}
                \frac{\int_\Theta e^{\tau S(t)}\Pi_0(\dd t)}{e^{\tau S(\theta)}}
            \Biggr)
            Q(\dd\theta)\\
            &=
            \KL(Q\|\Pi_0)-\tau\EE_{\theta\sim Q}[S(\theta)]
            +\log\int_\Theta e^{\tau S(t)}\Pi_0(\dd t)
            .
        \end{split}
    \end{equation*}
    Since $\KL(Q\|\widetilde{\Pi}_\tau)\ge 0$, rearranging the terms yields the fundamental inequality:
    \begin{equation*}
        \EE_{\theta\sim Q}[S(\theta)]-\frac{1}{\tau}\KL(Q\|\Pi_0)
        \le
        \frac{1}{\tau}\log\int_\Theta e^{\tau S(t)}\Pi_0(\dd t)
        .
    \end{equation*}
    Equality is achieved if and only if $\KL(Q\|\widetilde{\Pi}_\tau)=0$, which holds if and only if $Q=\widetilde{\Pi}_\tau$ almost everywhere.
    This establishes \eqref{eq:DV} and the uniqueness of its maximizer.
    Finally, the infimum formulation \eqref{eq:DV-inf} is obtained directly by substituting $S=-T$ into \eqref{eq:DV}.
\end{proof}

\subsection{Proof of Lemma~\ref{lem:scalar-kl}}
\label{app:pf-lem-scalar-kl}

\begin{proof}[Proof of Lemma~\ref{lem:scalar-kl}]
    If $\mu\in\{0,1\}$, the boundedness $0\le Y_i\le 1$ implies $Y_i=\mu$ almost surely for all $i$.
    Consequently, $\bar{Y}_n=\mu$ almost surely, and the left-hand side of \eqref{eq:scalar-kl} trivially evaluates to $1$.
    Thus, it suffices to consider $\mu\in(0,1)$.
    
    Define
    \begin{equation*}
        g_\mu(x)\coloneq \exp\{n\kl(x\|\mu)\}
        ,
        \quad
        x\in[0,1]
        .
    \end{equation*}
    Because the mapping $x\mapsto \kl(x\|\mu)$ is convex on $[0,1]$ and the exponential function is both convex and strictly increasing, the composition $g_\mu$ is convex on $[0,1]$.
    Let
    \begin{equation*}
        F(y_1,\dots,y_n)\coloneq g_\mu\Bigl(
            \frac{1}{n}\sum_{i=1}^n y_i
        \Bigr)
        ,
        \quad
        (y_1,\dots,y_n)\in[0,1]^n
        .
    \end{equation*}
    For each coordinate $i$, the map
    \begin{equation*}
        y\mapsto F(y_1,\dots,y_{i-1},y,y_{i+1},\dots,y_n)
    \end{equation*}
    is convex on $[0,1]$.
    Hence, for every deterministic $y\in[0,1]$, we have
    \begin{equation*}
        \begin{split}
            &F(y_1,\dots,y_{i-1},y,y_{i+1},\dots,y_n)
            \\
            &\le
            (1-y)F(y_1,\dots,y_{i-1},0,y_{i+1},\dots,y_n)
            +
            yF(y_1,\dots,y_{i-1},1,y_{i+1},\dots,y_n)
            .
        \end{split}
    \end{equation*}
    Taking the conditional expectation with respect to $Y_i$ and noting that $\EE[Y_i]=\mu$ yields
    \begin{equation*}
        \begin{split}
            &\EE\bigl[
                F(Y_1,\dots,Y_n)
                \mid Y_j,\ j\neq i
            \bigr]\\
            &\quad\le
            (1-\mu)F(Y_1,\dots,Y_{i-1},0,Y_{i+1},\dots,Y_n)
            +
            \mu F(Y_1,\dots,Y_{i-1},1,Y_{i+1},\dots,Y_n)
            .
        \end{split}
    \end{equation*}
    Let $B_i\sim\mathrm{Bernoulli}(\mu)$ be independent of $(Y_j)_{j\neq i}$.
    The right-hand side of the above inequality is exactly
    \begin{equation*}
        \EE\bigl[
            F(Y_1,\dots,Y_{i-1},B_i,Y_{i+1},\dots,Y_n)
            \mid Y_j,\ j\neq i
        \bigr]
        .
    \end{equation*}
    Taking expectations over the remaining variables yields
    \begin{equation*}
        \EE\bigl[
            F(Y_1,\dots,Y_n)
        \bigr]
        \le
        \EE\bigl[
            F(Y_1,\dots,Y_{i-1},B_i,Y_{i+1},\dots,Y_n)
        \bigr]
        .
    \end{equation*}
    Repeating this replacement iteratively for $i=1,\dots,n$, we obtain
    \begin{equation*}
        \EE\bigl[
            g_\mu(\bar{Y}_n)
        \bigr]
        \le
        \EE\bigl[
            g_\mu(\bar{B}_n)
        \bigr]
        ,
    \end{equation*}
    where $B_1,\dots,B_n$ are i.i.d.\ $\mathrm{Bernoulli}(\mu)$ and $\bar{B}_n=n^{-1}\sum_{i=1}^n B_i$.
    
    Let $K_n=\sum_{i=1}^n B_i\sim \mathrm{Binomial}(n,\mu)$.
    Then,
    \begin{equation*}
        \begin{split}
            \EE\bigl[g_\mu(\bar{B}_n)\bigr]
            &=
            \sum_{k=0}^n
            \binom{n}{k}\mu^k(1-\mu)^{n-k}
            \exp\Bigl\{
                n\kl\Bigl(\frac{k}{n}\Big\|\mu\Bigr)
            \Bigr\}\\
            &=
            \sum_{k=0}^n
            \binom{n}{k}
            \Bigl(\frac{k}{n}\Bigr)^k
            \Bigl(1-\frac{k}{n}\Bigr)^{n-k}
            ,
        \end{split}
    \end{equation*}
    adopting the conventions $0^0=1$ and $1^0=1$.
    Note that the boundary terms for $k=0$ and $k=n$ both evaluate to $1$.
    
    Fix $1\le k\le n-1$.
    Using the standard Stirling bounds
    \begin{equation*}
        \sqrt{2\pi m}\Bigl(\frac{m}{e}\Bigr)^m
        \le
        m!
        \le
        \exp\Bigl(\frac{1}{12m}\Bigr)\sqrt{2\pi m}\Bigl(\frac{m}{e}\Bigr)^m,
        \quad
        m\in\mathbb{N}
        ,
    \end{equation*}
    we obtain
    \begin{equation*}
        \begin{split}
            \binom{n}{k}
            \Bigl(
                \frac{k}{n}
            \Bigr)^k
            \Bigl(
                1-\frac{k}{n}
            \Bigr)^{n-k}
            &=
            \frac{n!}{k!(n-k)!}\frac{k^k(n-k)^{n-k}}{n^n}\\
            &\le
            \frac{
                \exp\bigl(\frac{1}{12n}\bigr)\sqrt{2\pi n}(n/e)^n
            }{
                \sqrt{2\pi k}(k/e)^k
                \sqrt{2\pi(n-k)}((n-k)/e)^{n-k}
            }
            \frac{k^k(n-k)^{n-k}}{n^n}\\
            &=
            \frac{\exp(\frac{1}{12n})}{\sqrt{2\pi}}
            \sqrt{\frac{n}{k(n-k)}}
            .
        \end{split}
    \end{equation*}
    
    It remains to bound the sum
    \begin{equation*}
        S_n\coloneq \sum_{k=1}^{n-1}\frac{1}{\sqrt{k(n-k)}}
        .
    \end{equation*}
    Define $f_n(x)\coloneq \{x(n-x)\}^{-1/2}$ on $(0,n)$.
    Observe that $f_n$ is strictly decreasing on $(0,n/2]$ and strictly increasing on $[n/2,n)$.
    
    If $n=2m+1$ is odd, then
    \begin{equation*}
        \begin{split}
            S_n
            &=
            \sum_{k=1}^{m} f_n(k)+\sum_{k=m+1}^{2m} f_n(k)\\
            &\le
            \int_0^{m} f_n(x)\dd x+\int_{m+1}^{2m+1} f_n(x)\dd x\\
            &\le
            \int_0^{n} f_n(x)\dd x
            .
        \end{split}
    \end{equation*}
    If $n=2m$ is even, $f_n$ attains its minimum at $x=m$, implying $f_n(x)\ge f_n(m)=1/m$ for all $x\in[m-1,m+1]$.
    Hence,
    \begin{equation*}
        \begin{split}
            S_n
            &=
            \sum_{k=1}^{m-1} f_n(k)+f_n(m)+\sum_{k=m+1}^{2m-1} f_n(k)\\
            &\le
            \int_0^{m-1} f_n(x)\dd x+\int_{m-1}^{m+1} f_n(x)\dd x+\int_{m+1}^{2m} f_n(x)\dd x
            \\
            &=
            \int_0^{n} f_n(x)\dd x
            .
        \end{split}
    \end{equation*}
    In both cases, we establish the upper bound
    \begin{equation*}
        S_n\le \int_0^{n}\frac{\dd x}{\sqrt{x(n-x)}}
        .
    \end{equation*}
    Using the substitution $x=n\sin^2 t$, we evaluate the integral as
    \begin{equation*}
        \int_0^{n}\frac{\dd x}{\sqrt{x(n-x)}}
        =
        \int_0^{\pi/2} 2\dd t
        =
        \pi
        .
    \end{equation*}
    Therefore,
    \begin{equation*}
        \EE\bigl[
            g_\mu(\bar{Y}_n)
        \bigr]
        \le
        2+\frac{\exp(\frac{1}{12n})}{\sqrt{2\pi}}\sqrt{n}\,S_n
        \le
        2+\exp\Bigl(
            \frac{1}{12n}
        \Bigr)\sqrt{\frac{\pi n}{2}}
        =
        \xi(n)
        ,
    \end{equation*}
    which completes the proof of \eqref{eq:scalar-kl}.
\end{proof}

\subsection{Proof of Lemma~\ref{lem:uniform-seeger-langford}}
\label{app:pf-lem-uniform-seeger-langford}

\begin{proof}[Proof of Lemma~\ref{lem:uniform-seeger-langford}]
    Fix $\theta\in\Theta$.
    Since $\ell(\theta;Z_i)\in[0,1]$ almost surely and the observations are i.i.d., applying Lemma~\ref{lem:scalar-kl} to $Y_i=\ell(\theta;Z_i)$ yields
    \begin{equation*}
        \EE\Bigl[
            \exp\bigl\{
                n\kl(\hat{L}_n(\theta)\|L(\theta))
            \bigr\}
        \Bigr]
        \le
        \xi(n)
        .
    \end{equation*}
    Integrating with respect to the prior $\Pi_0$ and invoking Fubini's theorem, we obtain
    \begin{equation*}
        \EE\Biggl[
            \int_\Theta
            \exp\bigl\{
                n\kl(\hat{L}_n(\theta)\|L(\theta))
            \bigr\}
            \Pi_0(\dd\theta)
        \Biggr]
        \le
        \xi(n)
        .
    \end{equation*}
    By Markov's inequality, with probability at least $1-\delta$, we have
    \begin{equation}
    \label{eq:seeger-markov}
        \int_\Theta
        \exp\bigl\{n\kl(\hat{L}_n(\theta)\|L(\theta))\bigr\}
        \Pi_0(\dd\theta)
        \le
        \frac{\xi(n)}{\delta}
        .
    \end{equation}
    Conditioning on the high-probability event where \eqref{eq:seeger-markov} holds, we apply Lemma~\ref{lem:DV} with
    \begin{equation*}
        S(\theta)=n\kl(\hat{L}_n(\theta)\|L(\theta))
        \quad\text{and}\quad
        \tau=1
        ,
    \end{equation*}
    which ensures that for every $Q\ll\Pi_0$,
    \begin{equation*}
        n\EE_{\theta\sim Q}\bigl[\kl(\hat{L}_n(\theta)\|L(\theta))\bigr]
        \le
        \KL(Q\|\Pi_0)+\log\{\xi(n)/\delta\}
        .
    \end{equation*}
    Because the Bernoulli relative entropy $(p,q)\mapsto \kl(p\|q)$ is jointly convex on $[0,1]^2$, applying Jensen's inequality yields
    \begin{equation*}
        \kl\bigl(
            \hat{L}_n(Q)\|L(Q)
        \bigr)
        =
        \kl\Bigl(
            \EE_{\theta\sim Q}[\hat{L}_n(\theta)]
            \Big\|
            \EE_{\theta\sim Q}[L(\theta)]
        \Bigr)
        \le
        \EE_{\theta\sim Q}\bigl[\kl(\hat{L}_n(\theta)\|L(\theta))\bigr].
    \end{equation*}
    Combining the last two displays proves \eqref{eq:uniform-seeger-langford}.
\end{proof}

\subsection{Proof of Lemma~\ref{lem:uniform-catoni}}
\label{app:pf-lem-uniform-catoni}

\begin{proof}[Proof of Lemma~\ref{lem:uniform-catoni}]
    By Lemma~\ref{lem:uniform-seeger-langford}, with probability at least $1-\delta$, the inequality
    \begin{equation*}
        \kl\bigl(\hat{L}_n(Q)\|L(Q)\bigr)\le n\,c_{n,\delta}(Q)
    \end{equation*}
    holds simultaneously for all $Q\ll\Pi_0$.
    Conditioning on this event, fix an arbitrary posterior $Q\ll\Pi_0$ and a scale parameter $\gamma>0$.
    The Bernoulli relative entropy admits the variational representation
    \begin{equation}
        \label{eq:kl-var}
        \kl(p\|q)
        =
        \sup_{\lambda\in\RR}
        \Bigl\{
            \lambda p - \log\bigl(1-q+qe^\lambda\bigr)
        \Bigr\}
        ,
        \quad
        p,q\in[0,1]
        .
    \end{equation}
    Applying \eqref{eq:kl-var} with $\lambda=-\gamma$ provides the lower bound
    \begin{equation*}
        -\gamma \hat{L}_n(Q)
        -\log\bigl(1-L(Q)+L(Q)e^{-\gamma}\bigr)
        \le
        \kl\bigl(\hat{L}_n(Q)\|L(Q)\bigr)
        \le
        n\,c_{n,\delta}(Q)
        .
    \end{equation*}
    Exponentiating and rearranging the terms, we obtain
    \begin{equation*}
        L(Q)\{1-e^{-\gamma}\}
        \le
        1-\exp\bigl\{
            -c_{n,\delta}(Q)-\gamma \hat{L}_n(Q)
        \bigr\}
        .
    \end{equation*}
    Since $\gamma>0$, it strictly follows that $1-e^{-\gamma}>0$.
    Dividing both sides by this quantity yields \eqref{eq:uniform-catoni}, completing the proof.
\end{proof}

\subsection{Proof of Lemma~\ref{lem:Gamma-bounds}}
\label{app:pf-lem-Gamma-bounds}

\begin{proof}[Proof of Lemma~\ref{lem:Gamma-bounds}]
    Fix a measurable ITR $d$ and a nuisance parameter $\alpha\in\cA$, and let $a_d=d(X)$.
    If $A\neq a_d$, it trivially follows that
    \begin{equation*}
        \Gamma_d^\alpha(Z)=\nu_{a_d,\alpha}(X)\in[0,1]
        \quad\text{a.s.}
    \end{equation*}
    If $A=a_d$, then by \eqref{eq:Gamma-d}, we have
    \begin{equation*}
        \Gamma_d^\alpha(Z)
        =
        \nu_{a_d,\alpha}(X)
        +
        \frac{\uR-\nu_{a_d,\alpha}(X)}{\pi(a_d\mid X)}
        =
        \Biggl(
            1-\frac{1}{\pi(a_d\mid X)}
        \Biggr)\nu_{a_d,\alpha}(X)
        +
        \frac{\uR}{\pi(a_d\mid X)}
        .
    \end{equation*}
    Since $\uR$ and $\nu_{a_d,\alpha}(X)$ are bounded in $[0,1]$ and $\pi(a_d\mid X)\ge\varepsilon$, the supremum is attained at $\uR=1$ and $\nu_{a_d,\alpha}(X)=0$, yielding the upper bound $\Gamma_d^\alpha(Z)\le \varepsilon^{-1}$.
    Conversely, because $1-\pi(a_d\mid X)^{-1}\le 0$, the infimum is attained at $\uR=0$ and $\nu_{a_d,\alpha}(X)=1$, yielding the lower bound $\Gamma_d^\alpha(Z)\ge 1-\varepsilon^{-1}$.
    Thus, \eqref{eq:Gamma-range} holds almost surely.
    
    Subtracting $\Gamma_d^\alpha(Z)$ from $c_\varepsilon=\varepsilon^{-1}$ and dividing by
    \begin{equation*}
        K_\varepsilon
        =
        \varepsilon^{-1}-(1-\varepsilon^{-1})
        =
        2\varepsilon^{-1}-1
    \end{equation*}
    confirms that the transformed loss $\ell_{\mathrm{val}}(\theta;Z)$ lies in $[0,1]$ almost surely.
\end{proof}
\section{Detailed Numerical Experiments}
\label{app:sec:detail-num-exp}

\subsection{Detailed Experimental Setup}
\label{app:subsec:detailed-experimental-setup}

This subsection details the exact simulation and implementation configurations utilized in Section~\ref{sec:numerical-experiments}.
We refer to the two primary environments as Scenario~1 and Scenario~2.
For notational convenience, we define the clipping operator as
\begin{equation*}
    \operatorname{clip}_{[a,b]}(t)
    \coloneq
    \min\{b,\max\{a,t\}\}
    .
\end{equation*}

\subsubsection{Common Design}
\label{app:subsubsec:common-design}

In each Monte Carlo replication, we draw a policy-learning sample of size $N$ and an independent test sample of size $N_{\mathrm{test}}=10{,}000$.
The $\rho$-sweep experiments presented in the main text fix $N=200$ and vary $\rho\in\{0,0.25,0.5,0.75,1.0,1.25,1.5,1.75,2.0\}$, whereas the $N$-sweep experiments fix $\rho=1.5$ and vary $N\in\{100,200,500,1000,2000\}$.
The results for both sweeps are averaged over $30$ independent replications.
For the certificate diagnostics presented exclusively in this appendix, we fix $N=1000$ and reuse the identical $\rho$ grid.
The split-free ablation study fixes $\rho=1.5$ and utilizes the same sample size grid for $N$.
Unless explicitly stated otherwise, we employ oracle certificates throughout the synthetic experiments.

For each arm $a\in\{-1,1\}$, let
\begin{equation*}
    \mu_a^\ast(x)
    \coloneq
    \EE[(R^\ast)^a\mid X=x]
    ,
    \quad
    \mu_a(x)
    \coloneq
    \EE[R^a\mid X=x]
    ,
    \quad
    \underline{\mu}_a(x)
    \coloneq
    \EE[\uR^a\mid X=x]
    .
\end{equation*}
The true logging propensities are assumed to be known to the learner.
Scenario~1 employs balanced randomization, whereas Scenario~2 utilizes a covariate-dependent propensity score, as detailed below.
In both scenarios, all realized rewards are strictly bounded within $[0,1]$ via clipping.

\subsubsection{Scenario~1}
\label{app:subsubsec:scenario1}

We generate the covariates as
\begin{equation*}
    X_1,X_2
    \overset{\mathrm{i.i.d.}}{\sim}
    \mathrm{Unif}[-1,1]
    ,
    \quad
    \pi(1\mid x)=\pi(-1\mid x)=\frac{1}{2}
    .
\end{equation*}
The main and treatment effects are formulated as
\begin{equation*}
    m_1(x)=0.38x_1-0.22x_2
    ,
    \quad
    \tau_1(x)=0.72(x_1+0.65x_2)/1.65
    .
\end{equation*}
The arm-specific target conditional means are given by
\begin{equation*}
    \mu_a^\ast(x)
    =
    \operatorname{clip}_{[0,1]}
    \bigl(
        0.5+0.15\,m_1(x)+0.15\,a\,\tau_1(x)
    \bigr)
    ,
    \quad
    a\in\{-1,1\}
    .
\end{equation*}

We define an identical base certificate envelope for both arms:
\begin{equation*}
    U_{0,1}^{(1)}(x)=U_{0,-1}^{(1)}(x)
    =
    0.02+0.04\1\{x_1>0,\ x_2>0\}
    .
\end{equation*}
For a given uncertainty level $\rho$, the envelope is scaled as
\begin{equation*}
    U_{\rho,a}^{(1)}(x)
    =
    \min\{\rho\,U_{0,a}^{(1)}(x),\,0.10\}
    ,
    \quad
    a\in\{-1,1\}
    .
\end{equation*}
The proxy conditional means are then defined as
\begin{equation*}
    \mu_a(x)
    =
    \operatorname{clip}_{[0,1]}
    \bigl(
        \mu_a^\ast(x)+0.50\,U_{\rho,a}^{(1)}(x)
    \bigr)
    .
\end{equation*}

The realized target rewards are generated according to
\begin{equation*}
    (R^\ast)^a
    =
    \operatorname{clip}_{[0,1]}
    \bigl(
        \mu_a^\ast(X)+0.10\,\varepsilon_a^\ast
    \bigr)
    ,
    \quad
    \varepsilon_a^\ast
    \sim
    \mathrm{Unif}[-1,1]
    .
\end{equation*}
The optimism bias is introduced as
\begin{equation*}
    B^a
    =
    U_{\rho,a}^{(1)}(X)V_a
    ,
    \quad
    V_a
    \sim
    \mathrm{Unif}[0,1]
    ,
\end{equation*}
yielding the observed proxy and certified rewards:
\begin{equation*}
    R^a
    =
    \operatorname{clip}_{[0,1]}
    \bigl(
        (R^\ast)^a+B^a
    \bigr)
    ,
    \quad
    \uR^a
    =
    (R^a-U_{\rho,a}^{(1)}(X))_+
    .
\end{equation*}
Since an identical envelope is applied to both arms and the effect of active clipping is negligible in these configurations, the reward uncertainty in Scenario~1 is effectively policy-invariant.

\subsubsection{Scenario~2}
\label{app:subsubsec:scenario2}

We generate the covariates by first drawing
\begin{equation*}
    X_1,X_2,X_3,X_4
    \overset{\mathrm{i.i.d.}}{\sim}
    \mathrm{Unif}[-1,1]
    ,
    \quad
    \varepsilon_5,\varepsilon_6,\varepsilon_7,\varepsilon_8
    \overset{\mathrm{i.i.d.}}{\sim}
    N(0,1)
    ,
\end{equation*}
and subsequently defining the following transformations:
\begin{equation*}
    X_5
    =
    \operatorname{clip}_{[-1,1]}
    \bigl(
        0.50X_1-0.22X_2+0.14X_3+0.32\varepsilon_5
    \bigr)
    ,
\end{equation*}
\begin{equation*}
    X_6
    =
    \operatorname{clip}_{[-1,1]}
    \bigl(
        -0.30X_2+0.26X_3+0.18X_5+0.34\varepsilon_6
    \bigr)
    ,
\end{equation*}
\begin{equation*}
    X_7
    =
    \operatorname{clip}_{[-1,1]}
    \bigl(
        0.68X_5+0.40X_6+0.30X_1+0.28\varepsilon_7
    \bigr)
    ,
\end{equation*}
and
\begin{equation*}
    X_8
    =
    \operatorname{clip}_{[-1,1]}
    \bigl(
        0.46X_3-0.28X_4+0.60X_7+0.18X_5+0.30\varepsilon_8
    \bigr)
    .
\end{equation*}
Consequently, the variables $(X_5,\ldots,X_8)$ serve as proxies for frailty, workload, and biomarkers, exhibiting correlation with the primary clinical risk factors.

The main effect is formulated as
\begin{equation*}
    m_2(x)
    =
    \operatorname{clip}_{[-1,1]}
    \Bigl(
        0.20\sin(1.2\pi x_3)
        +0.14\cos(0.8\pi x_4)
        +0.10x_5
        -0.08x_6
        +0.06x_7x_8
    \Bigr)
    .
\end{equation*}
To define the treatment effect, we set
\begin{equation*}
    r(x)=0.90x_1-0.64x_2+0.45\sin(\pi x_3)+0.24x_4
    ,
    \quad
    w(x)=\exp\{-2.9(r(x)-0.02)^2\}
    ,
\end{equation*}
and
\begin{equation*}
    b(x)
    =
    0.14\tanh(x_3-0.85x_4)
    -0.04(x_5)_+
    -0.03(x_6)_+
    .
\end{equation*}
This yields the treatment effect function:
\begin{equation*}
    \tau_2(x)
    =
    \operatorname{clip}_{[-1,1]}
    \Bigl(
        \tanh\bigl(
            2.05\{w(x)+b(x)-0.60\}
        \bigr)
    \Bigr)
    .
\end{equation*}
The arm-specific target conditional means are then given by
\begin{equation*}
    \mu_a^\ast(x)
    =
    \operatorname{clip}_{[0,1]}
    \bigl(
        0.5+0.15\,m_2(x)+0.15\,a\,\tau_2(x)
    \bigr)
    ,
    \quad
    a\in\{-1,1\}
    .
\end{equation*}
This design yields a genuinely beneficial treatment effect localized strictly within a moderate-risk patient subgroup.

To construct the arm-specific reward certificates, we define the following auxiliary functions:
\begin{equation*}
    r_u(x)
    =
    0.94x_1-0.70x_2+0.48\sin(\pi x_3)+0.24x_4
    ,
    \quad
    h(x)
    =
    \frac{1}{1+\exp\{-2.6(r_u(x)-0.02)\}}
    ,
\end{equation*}
\begin{equation*}
    q(x)
    =
    \exp\{-8.6(r_u(x)-0.02)^2\}
    ,
    \quad
    f(x)
    =
    0.78(x_5)_+ + 0.58(x_6)_+
    ,
\end{equation*}
\begin{equation*}
    o(x)
    =
    \frac{1}{1+\exp\{-2.4(1.18x_7+0.60x_5-0.12)\}}
    ,
\end{equation*}
\begin{equation*}
    s(x)
    =
    \frac{1}{1+\exp\{-2.8(1.28x_8+0.92x_3x_4+0.78x_5x_7+0.48x_6x_8-0.06)\}}
    ,
\end{equation*}
\begin{equation*}
    t(x)
    =
    \frac{1}{1+\exp\{-2.5(1.08x_7+0.96x_8+0.92x_5x_7+0.70x_6x_8-0.08)\}}
    ,
\end{equation*}
and
\begin{equation*}
    c(x)
    =
    \frac{1}{1+\exp\{-1.9(0.48x_6-0.52x_7-0.28x_8-0.10)\}}
    .
\end{equation*}
The base envelopes are specified as
\begin{equation*}
    \begin{split}
        U_{0,1}^{(2)}(x)
        =
        \operatorname{clip}_{[0,0.32]}
        &\Bigl(
            0.010
            +0.020h(x)
            +q(x)\{0.16f(x)+0.17s(x)+0.11o(x)+0.19t(x)\}
        \\
        &\quad
            +0.075t(x)
            +0.022f(x)t(x)
        \Bigr)
        ,
    \end{split}
\end{equation*}
and
\begin{equation*}
    U_{0,-1}^{(2)}(x)
    =
    \operatorname{clip}_{[0,0.06]}
    \Bigl(
        0.002
        +0.006h(x)
        +0.010q(x)c(x)
        +0.004c(x)
    \Bigr)
    .
\end{equation*}
For a given uncertainty level $\rho$, these envelopes are scaled to
\begin{equation*}
    U_{\rho,a}^{(2)}(x)
    =
    \min\{\rho\,U_{0,a}^{(2)}(x),\,0.32\}
    ,
    \quad
    a\in\{-1,1\}
    .
\end{equation*}
Consequently, the treated arm is assigned a substantially larger and more complex uncertainty certificate compared to the control arm.

The proxy conditional means are constructed as
\begin{equation*}
    \mu_1(x)
    =
    \operatorname{clip}_{[0,1]}
    \bigl(
        \mu_1^\ast(x)+0.92\,U_{\rho,1}^{(2)}(x)
    \bigr)
    ,
    \quad
    \mu_{-1}(x)
    =
    \operatorname{clip}_{[0,1]}
    \bigl(
        \mu_{-1}^\ast(x)+0.14\,U_{\rho,-1}^{(2)}(x)
    \bigr)
    .
\end{equation*}
The realized target rewards are generated according to
\begin{equation*}
    (R^\ast)^1
    =
    \operatorname{clip}_{[0,1]}
    \bigl(
        \mu_1^\ast(X)+\{0.08+0.68U_{\rho,1}^{(2)}(X)\}\varepsilon_1^\ast
    \bigr)
    ,
\end{equation*}
\begin{equation*}
    (R^\ast)^{-1}
    =
    \operatorname{clip}_{[0,1]}
    \bigl(
        \mu_{-1}^\ast(X)+\{0.05+0.04U_{\rho,-1}^{(2)}(X)\}\varepsilon_{-1}^\ast
    \bigr)
    ,
\end{equation*}
where $\varepsilon_1^\ast,\varepsilon_{-1}^\ast\sim\mathrm{Unif}[-1,1]$.
The optimism biases are modeled as
\begin{equation*}
    B^1
    =
    U_{\rho,1}^{(2)}(X)V_1
    ,
    \quad
    V_1\sim\mathrm{Beta}(7.0,1.0)
    ,
\end{equation*}
and
\begin{equation*}
    B^{-1}
    =
    U_{\rho,-1}^{(2)}(X)V_{-1}
    ,
    \quad
    V_{-1}\sim\mathrm{Beta}(1.08,5.2)
    .
\end{equation*}
Finally, the observed proxy and certified rewards are computed as
\begin{equation*}
    R^a
    =
    \operatorname{clip}_{[0,1]}
    \bigl(
        (R^\ast)^a+B^a
    \bigr)
    ,
    \quad
    \uR^a
    =
    (R^a-U_{\rho,a}^{(2)}(X))_+
    .
\end{equation*}
This configuration establishes a challenging nuisance-conflict scenario: while the treatment is genuinely beneficial for a moderate-risk subgroup, the observed proxy reward for the treated arm exhibits the highest optimism bias specifically within clinically vulnerable regions.

Furthermore, the logging propensity in Scenario~2 is non-uniform.
We define
\begin{equation*}
    r_\pi(x)
    =
    0.84x_1-0.60x_2+0.42\sin(\pi x_3)+0.20x_4
    ,
\end{equation*}
and
\begin{equation*}
    \ell_\pi(x)
    =
    0.20r_\pi(x)
    +1.35(x_5)_+
    +1.10(x_6)_+
    +1.55x_7
    +1.25x_8
    +1.55x_5x_7
    +0.95x_6x_8
    -0.40x_3x_4
    .
\end{equation*}
This yields the propensity scores:
\begin{equation*}
    \pi(1\mid x)
    =
    \operatorname{clip}_{[0.01,0.99]}
    \Bigl(
        \frac{1}{1+\exp\{-\ell_\pi(x)\}}
    \Bigr)
    ,
    \quad
    \pi(-1\mid x)=1-\pi(1\mid x)
    .
\end{equation*}

\subsubsection{Feature Maps and Policy Classes}
\label{app:subsubsec:feature-maps-and-policy-classes}

All covariates are standardized based on the policy-learning sample prior to any feature expansion.
In Scenario~1, all score-based methods employ a standardized linear score that includes an intercept.
For Scenario~2, we utilize a richer, common clinical basis for all score-based methods.
This feature map comprises main effects, quadratic terms, sine transformations of $x_1,\ldots,x_4$, positive-part thresholding for $x_5,x_6,x_7$, and select low-order interactions among the primary clinical, frailty, workload, and biomarker variables (specifically, $x_1x_2$, $x_3x_4$, $x_1x_3$, $x_1x_4$, $x_2x_3$, $x_2x_4$, $x_4x_6$, $x_5x_6$, $x_5x_7$, $x_6x_7$, $x_4x_7$, $x_3x_8$), along with an intercept.
In contrast, Policy Tree operates directly on the raw, unexpanded covariates.

\subsubsection{Model Construction}
\label{app:subsubsec:model-construction}

For OWL, we train a weighted hinge-loss classifier, selecting the $\ell_2$ penalty from the grid $\{10^{-3},10^{-2},10^{-1},1\}$ via 5-fold cross-validation.
The tuning objective maximizes the inverse-probability-weighted empirical value corresponding to the specified reward variant.
For the $R$-family models, the training labels are the observed treatments $A$ with sample weights $R/\pi(A\mid X)$.
For the $\uR$-family models, the weights are replaced by $\uR/\pi(A\mid X)$.

For Q-learning, we fit a single-stage linear working model of the form $Q(x,a)=\beta^\top h_0(x)+a\,\psi^\top h_1(x)$, utilizing the scenario-specific feature map for both $h_0$ and $h_1$.
The ridge penalty is selected from the identical grid based on 5-fold cross-validated mean squared prediction error.

For RWL, we implement a linear residual weighted learner paired with a treatment-free residual regression model.
Both the policy score and the residual model employ the aforementioned scenario-specific feature map.
The regularization parameter is tuned over the same grid via 5-fold cross-validation, targeting the maximization of the stabilized value criterion.

For Policy Tree, we constrain the hypothesis class to depth-2 decision trees, enforcing a minimum node size of 20 and a split step of 25.
The input score matrix is constructed using a doubly robust plug-in estimator, which relies on arm-specific nuisance outcome regressions fitted to the corresponding reward variant.

PROWL and its uncertified ablation, PROWL ($U=0$), share an identical finite-library PAC-Bayesian implementation.
The nuisance models are estimated separately for each arm via ridge regression (with a fixed penalty of $10^{-6}$) on the entire policy-learning sample, explicitly avoiding sample splitting.
Temperature scaling parameters $\eta$ and $\gamma$ are jointly tuned over the grid $\{1/8,1/4,1/2,1,2,4,8\}$, selecting the pair that maximizes the exact-value lower confidence bound derived in Section~\ref{subsec:learning-rate}.
We specify the confidence level as $\delta=0.1$, the prior standard deviation as $5$, and restrict the score bound to $3$.

The finite posterior library is constructed using two Gaussian anchor particles, 32 prior particles, and four local perturbations per anchor with a local scale parameter of $0.3$.
This candidate library is further enriched by incorporating anchor solutions obtained from certified weighted-hinge optimizations, treatment-free residualized fits, and plug-in Q-learning or RWL solutions, scanned across an auxiliary penalty grid of $\{10^{-4},10^{-3},10^{-2},10^{-1},1\}$.
The PROWL ($U=0$) benchmark is executed by substituting the certified reward $\uR$ with the nominal proxy reward $R$ throughout this entire pipeline.

To ensure a fair comparison with the deterministic classical baselines, both PROWL and PROWL ($U=0$) output the deterministic rule corresponding to the Maximum a Posteriori (MAP) element selected from the fitted finite posterior library.

\subsubsection{Evaluation}
\label{app:subsubsec:evaluation}

For any learned deterministic policy $\hat{d}$, we evaluate the out-of-sample target and robust regrets on the independent test set utilizing the oracle conditional means:
\begin{equation*}
    \hat{V}^\ast(\hat{d})
    =
    \frac{1}{N_{\mathrm{test}}}
    \sum_{i=1}^{N_{\mathrm{test}}}
    \mu^\ast_{\hat{d}(X_i)}(X_i)
    ,
    \quad
    \hat{V}_{\uR}(\hat{d})
    =
    \frac{1}{N_{\mathrm{test}}}
    \sum_{i=1}^{N_{\mathrm{test}}}
    \underline{\mu}_{\hat{d}(X_i)}(X_i)
    .
\end{equation*}
The true target and robust regrets are subsequently computed by subtracting these realized values from the corresponding theoretical optimal values.

For the certificate diagnostics presented solely in this appendix, we further report the following policy-level gaps:
\begin{equation*}
    \begin{split}
        \mathrm{ProxyTargetGap}(\hat{d})
        &\coloneq
        V(\hat{d})-V^\ast(\hat{d})
        ,
        \\
        \mathrm{TargetCertifiedGap}(\hat{d})
        &\coloneq
        V^\ast(\hat{d})-V_{\uR}(\hat{d})
        .
    \end{split}
\end{equation*}
These metrics quantify the inherent optimism of the proxy reward and the degree of conservatism introduced by the certification process, respectively.
At the dataset level, we compute the expected certificate magnitude $\EE[U]$, the active clipping rate $\mathrm{Clip}$, and the certificate validity rate $\mathrm{Valid}$, defined as follows:
\begin{equation*}
    \mathrm{Clip}
    \coloneq
    \PP\{U>R\}
    ,
    \quad
    \mathrm{Valid}
    \coloneq
    \PP\bigl\{
        R^1-(R^\ast)^1\le U_{\rho,1}(X),\
        R^{-1}-(R^\ast)^{-1}\le U_{\rho,-1}(X)
    \bigr\}
    .
\end{equation*}
Here, $\EE[U]$ and $\mathrm{Clip}$ are estimated using the policy-learning sample, whereas $\mathrm{Valid}$ is evaluated on the independent test sample.

In the split-free ablation study, we compare our default split-free PROWL implementation against an honest sample-split variant.
The latter allocates one arm-stratified half of the data for estimating the nuisance models and reserves the remaining half for policy learning, maintaining identical tuning grids and evaluation criteria.

\subsection{Additional Results}
\label{app:subsec:num-results}

Figures~\ref{app:fig:scenario1-rho-sweep-comp} through \ref{app:fig:scenario2-n-sweep-comp} in the appendix present the complementary metric-family pairings that were omitted from the main text for brevity.
The qualitative conclusions derived from these figures align closely with those in the primary analysis.
While substituting $R$ with $\uR$ demonstrably improves the classical baselines—particularly in Scenario~2—PROWL consistently emerges as the most robust procedure.
Notably, its performance advantage is most pronounced in this highly heterogeneous clinical setting.

\begin{figure*}[tb]
\vskip 0.2in
\begin{center}
\centerline{\includegraphics[width=\columnwidth]{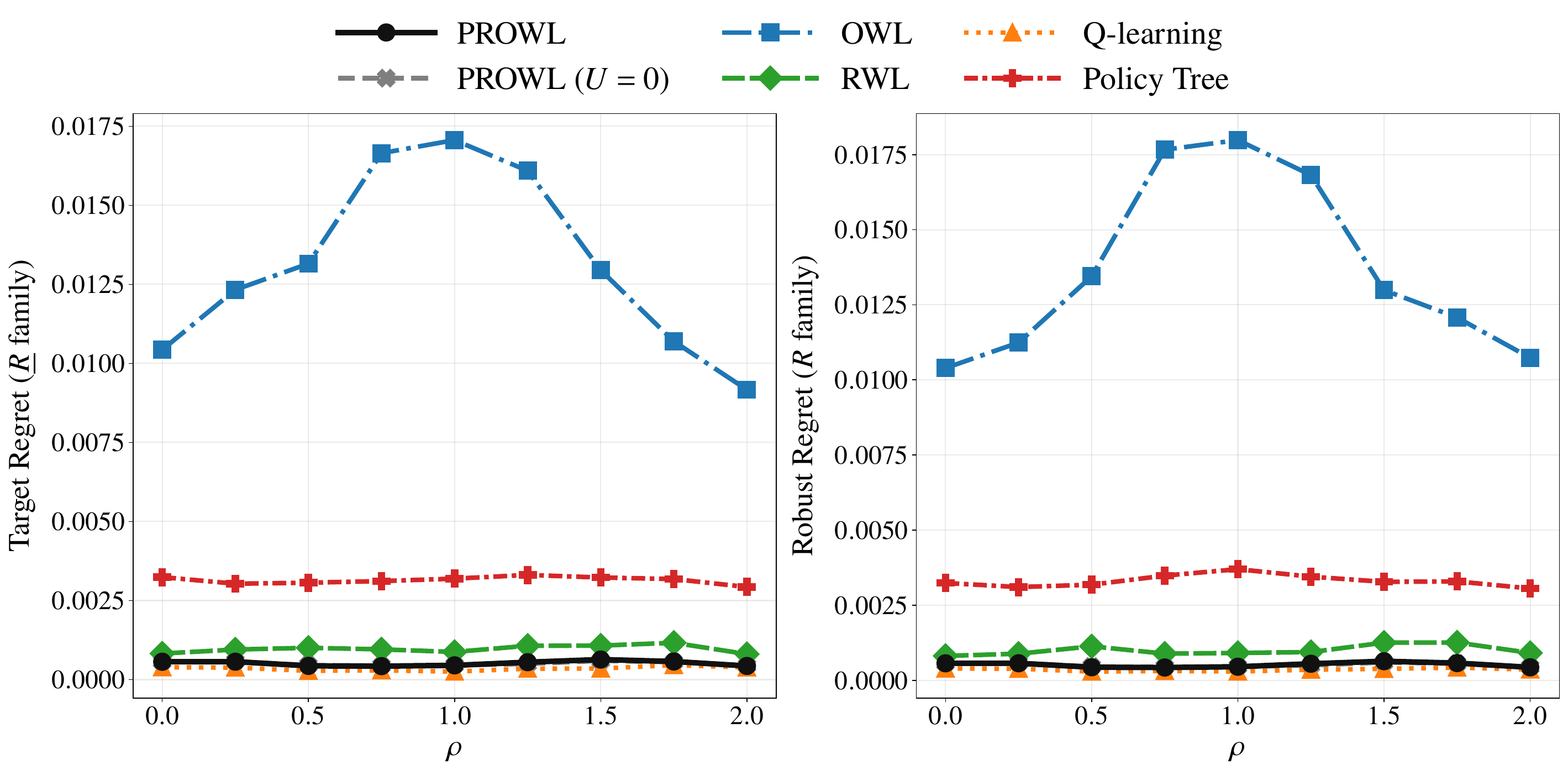}}
\caption{
Complementary performance comparison across varying uncertainty levels ($\rho$) for Scenario 1.
The left panel reports target regret against the $\uR$-family baselines, and the right panel reports robust regret against the $R$-family baselines.
}
\label{app:fig:scenario1-rho-sweep-comp}
\end{center}
\vskip -0.2in
\end{figure*}

\begin{figure*}[tb]
\vskip 0.2in
\begin{center}
\centerline{\includegraphics[width=\columnwidth]{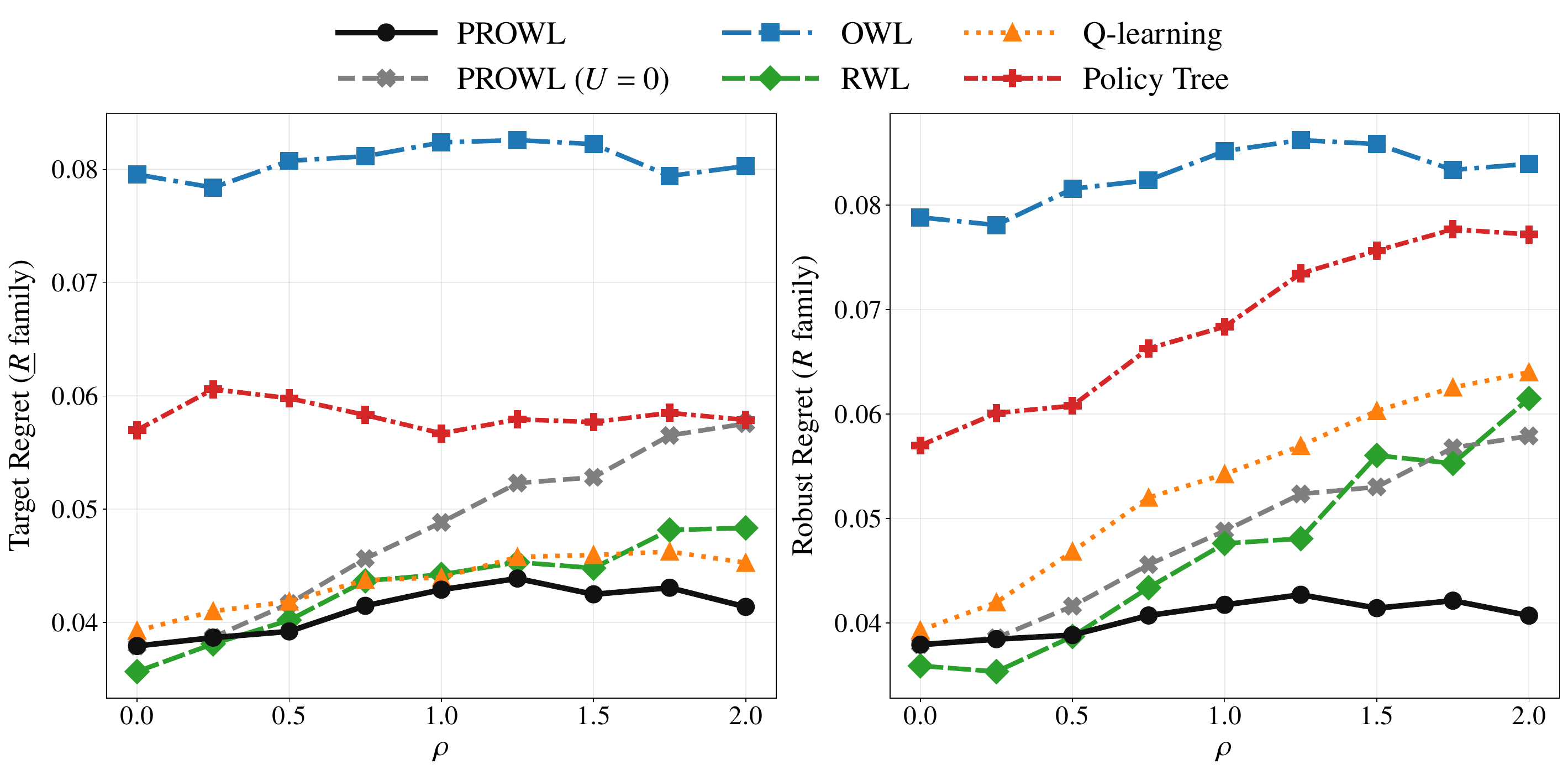}}
\caption{
Complementary performance comparison across varying uncertainty levels ($\rho$) for Scenario 2.
The left panel reports target regret against the $\uR$-family baselines, and the right panel reports robust regret against the $R$-family baselines.
}
\label{app:fig:scenario2-rho-sweep-comp}
\end{center}
\vskip -0.2in
\end{figure*}

\begin{figure*}[tb]
\vskip 0.2in
\begin{center}
\centerline{\includegraphics[width=\columnwidth]{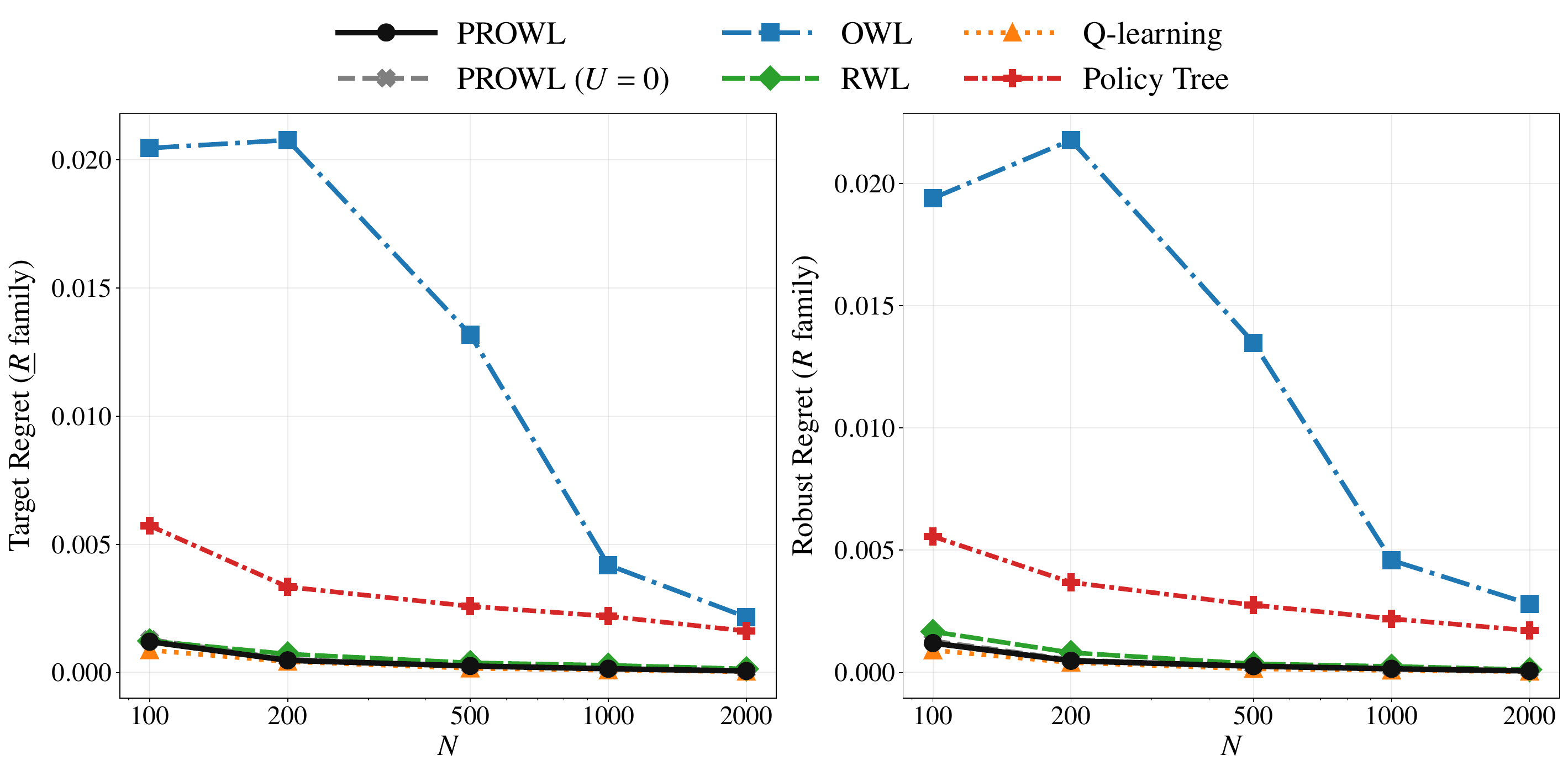}}
\caption{
Complementary performance comparison across varying sample sizes ($N$) for Scenario 1.
The left panel reports target regret against the $\uR$-family baselines, and the right panel reports robust regret against the $R$-family baselines.
}
\label{app:fig:scenario1-n-sweep-comp}
\end{center}
\vskip -0.2in
\end{figure*}

\begin{figure*}[tb]
\vskip 0.2in
\begin{center}
\centerline{\includegraphics[width=\columnwidth]{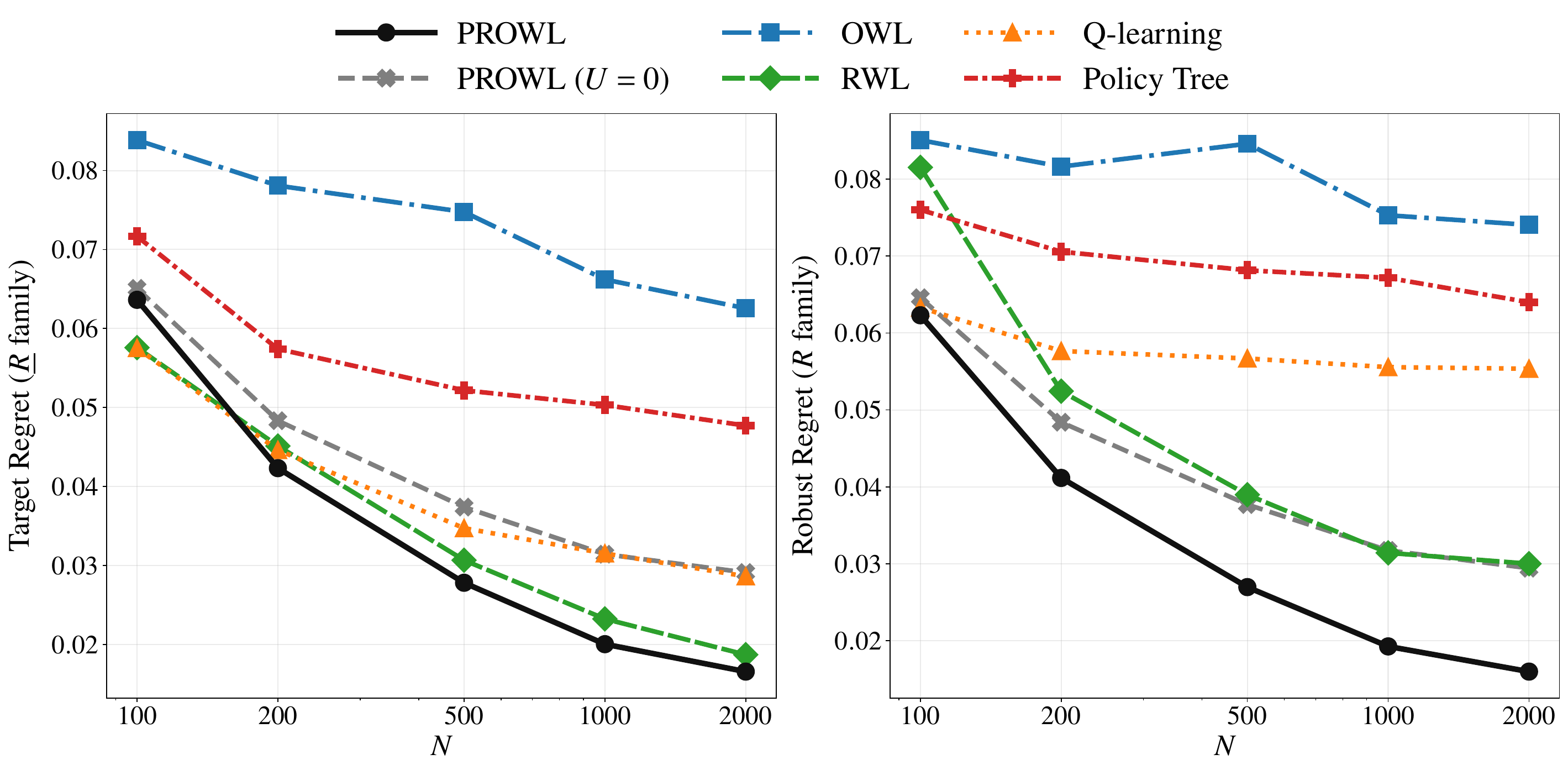}}
\caption{
Complementary performance comparison across varying sample sizes ($N$) for Scenario 2.
The left panel reports target regret against the $\uR$-family baselines, and the right panel reports robust regret against the $R$-family baselines.
}
\label{app:fig:scenario2-n-sweep-comp}
\end{center}
\vskip -0.2in
\end{figure*}

To understand how the oracle certificate structurally reshapes the decision problem, Figure~\ref{app:fig:certificate-diag} plots the proxy-target gap and the target-certified gap as functions of the uncertainty level $\rho$.
These metrics are evaluated at $N=1000$ for PROWL, PROWL ($U=0$), and the $\uR$-family of classical baselines.
In Scenario~1, the performance trajectories are nearly indistinguishable across the methods, a behavior that is theoretically consistent with the policy-invariant nature of the reward uncertainty in this setup.
Conversely, in Scenario~2, both gaps widen as $\rho$ increases, with the proxy-target gap remaining substantially larger than the target-certified gap.
This discrepancy indicates that the primary impediment to policy learning is the systematic optimism bias inherent in the proxy reward, rather than any excessive conservatism introduced by the certified reward bound.
Notably, PROWL consistently achieves the smallest, or near-smallest, proxy-target gap across the majority of the $\rho$-sweep.
This highlights its ability to strategically avoid covariate regions where the proxy reward is most misleading.

\begin{figure*}[tb]
\vskip 0.2in
\begin{center}
\centerline{\includegraphics[width=\columnwidth]{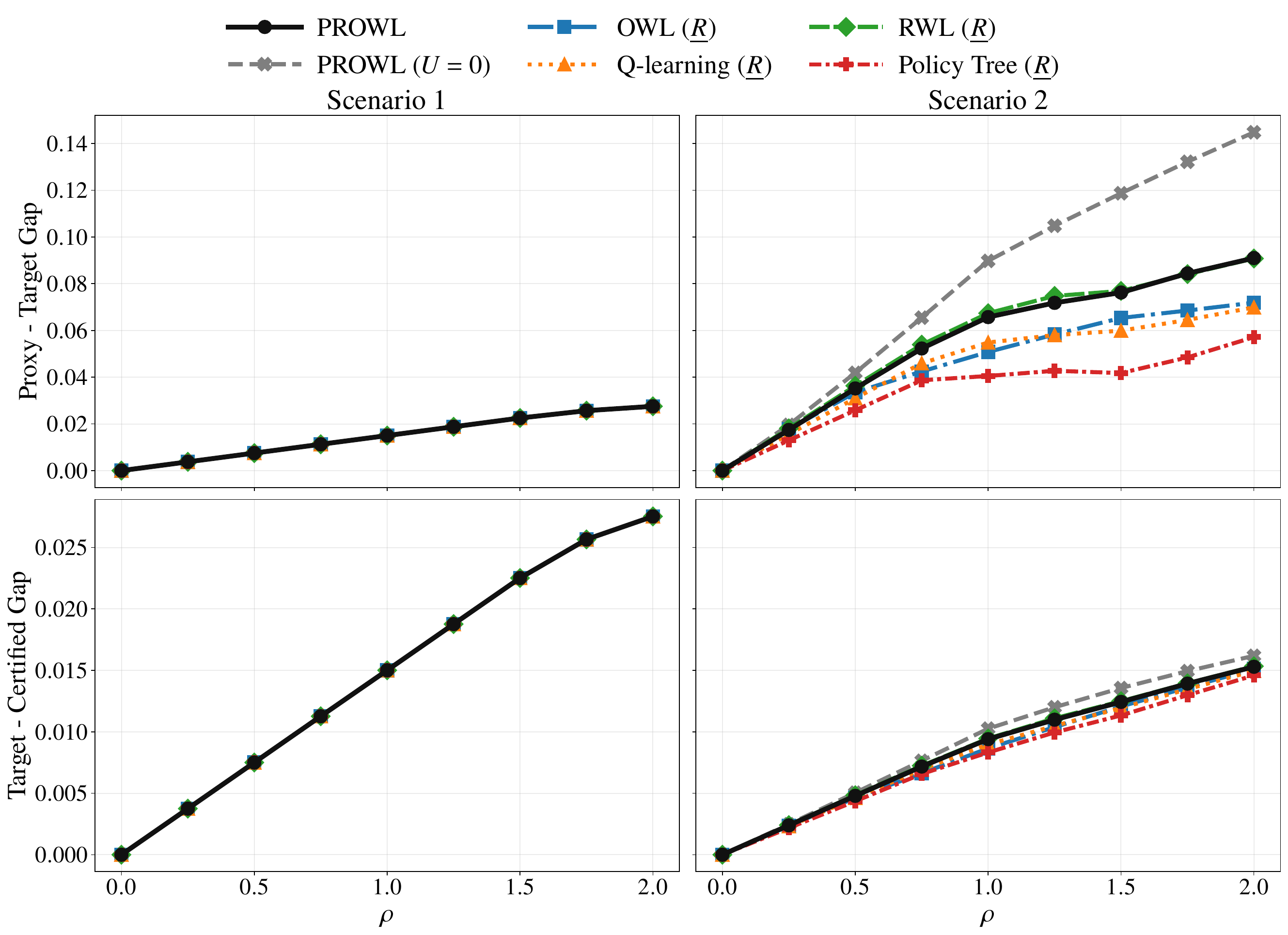}}
\caption{
Certificate-diagnostic plots.
Rows report the proxy-target gap $V(\hat{d})-V^\ast(\hat{d})$ and the target-certified gap $V^\ast(\hat{d})-V_{\uR}(\hat{d})$, while columns correspond to Scenarios~1 and 2.
Methods shown are PROWL, PROWL ($U=0$), and the classical $\uR$-family baselines.
}
\label{app:fig:certificate-diag}
\end{center}
\vskip -0.2in
\end{figure*}

Table~\ref{app:tab:certificate-dataset-diagnostics} summarizes the dataset-level diagnostics for the constructed certificates.
As expected, the average certificate magnitude increases monotonically with $\rho$ across both scenarios.
Furthermore, it is uniformly larger in Scenario~2, corroborating that the clinically heterogeneous environment necessitates a more substantial uncertainty correction.
Simultaneously, the active clipping rate ($\mathrm{Clip}$) remains effectively zero, and the validity rate ($\mathrm{Valid}$) is strictly one across all evaluated configurations.
This confirms that the certificates are both strictly valid and do not pathologically truncate the observed rewards.
Therefore, the performance discrepancies highlighted in the main text do not stem from aggressive reward clipping; rather, they reflect the varying capacities of the learning algorithms to effectively optimize under a valid, treatment-sensitive uncertainty correction.

\begin{table}[tb]
\centering
\caption{
Dataset-level certificate diagnostics across uncertainty levels $\rho$.
Here $\EE[U]$ is the logged-sample mean certificate magnitude, Clip is the logged-sample rate $\PP(U>R)$, and Valid is the independent-test validity rate that the oracle envelope upper-bounds the true proxy bias in both arms.
}
\label{app:tab:certificate-dataset-diagnostics}
\setlength{\tabcolsep}{6pt}
\begin{tabular}{c rrr rrr}
\toprule
& \multicolumn{3}{c}{Scenario 1} & \multicolumn{3}{c}{Scenario 2} \\
\cmidrule(lr){2-4} \cmidrule(lr){5-7}
$\rho$ & $\EE[U]$ & Clip & Valid & $\EE[U]$ & Clip & Valid \\
\midrule
$0$    & $0.000$ & $0.000$ & $1.000$
& $0.000$ & $0.000$ & $1.000$ \\
$0.25$ & $0.008$ & $0.000$ & $1.000$
& $0.023$ & $0.000$ & $1.000$ \\
$0.5$  & $0.015$ & $0.000$ & $1.000$
& $0.047$ & $0.000$ & $1.000$ \\
$0.75$ & $0.023$ & $0.000$ & $1.000$
& $0.070$ & $0.000$ & $1.000$ \\
$1$    & $0.030$ & $0.000$ & $1.000$
& $0.093$ & $0.000$ & $1.000$ \\
$1.25$ & $0.037$ & $0.000$ & $1.000$
& $0.107$ & $0.000$ & $1.000$ \\
$1.5$  & $0.045$ & $0.000$ & $1.000$
& $0.119$ & $0.000$ & $1.000$ \\
$1.75$ & $0.051$ & $0.000$ & $1.000$
& $0.131$ & $0.000$ & $1.000$ \\
$2$    & $0.055$ & $0.000$ & $1.000$
& $0.141$ & $0.000$ & $1.000$ \\
\bottomrule
\end{tabular}
\end{table}

Finally, Figure~\ref{app:fig:split-ablation} contrasts our default split-free PROWL implementation with an honest sample-splitting baseline, where one arm-stratified half of the data is used for nuisance estimation and the remaining half for policy optimization.
The split-free formulation achieves uniformly lower regret, with the performance gap being most pronounced in Scenario~2 under moderate sample sizes.
This gap gradually diminishes as $N$ increases and remains marginal in Scenario~1.
These observations suggest that the primary benefit of split-free learning lies in its ability to maximize finite-sample efficiency, rather than fundamentally altering the asymptotic learning target.

\begin{figure*}[tb]
\vskip 0.2in
\begin{center}
\centerline{\includegraphics[width=\columnwidth]{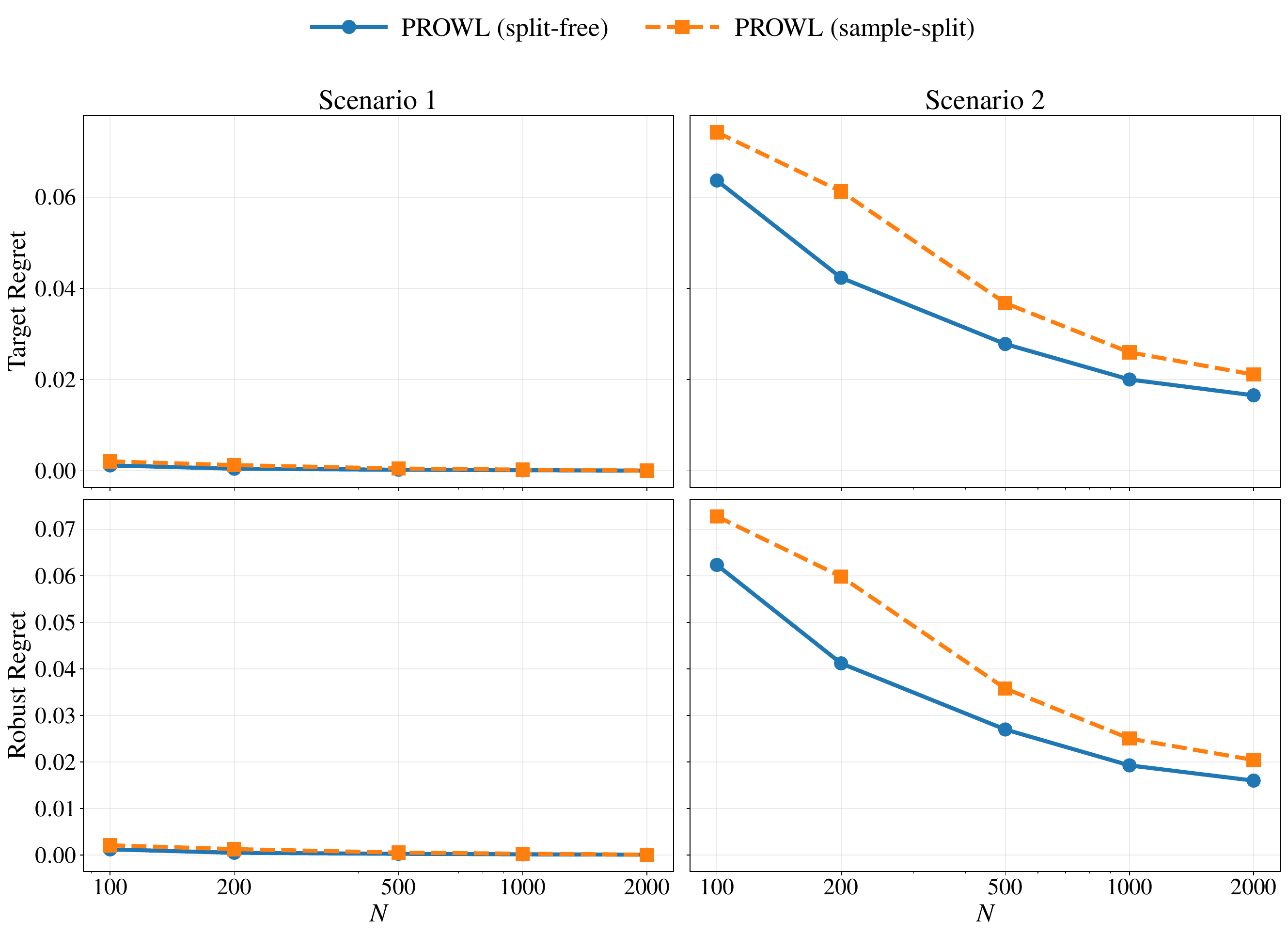}}
\caption{
Split-free ablation at $\rho=1.5$.
Rows report target regret and robust regret, while columns correspond to Scenarios~1 and 2.
PROWL (split-free) fits nuisance functions and the policy on the same sample, whereas PROWL (sample-split) uses a stratified honest half-sample split.
}
\label{app:fig:split-ablation}
\end{center}
\vskip -0.2in
\end{figure*}
\section{Detailed Actual-Data Experiments}
\label{app:sec:detail-actual-exp}

\subsection{Detailed Data Description and Implementation}
\label{app:subsec:actual-data-details}

The public ELAIA-1 dataset provides one record per randomized patient episode, comprising a mixture of baseline, process, and outcome variables.
While \citet{huang2025identification} restricted their focus to a secondary-analysis subset with available bicarbonate measurements, our policy-learning framework leverages the complete deidentified dataset and handles missing pre-randomization covariates via within-split imputation.
The primary policy covariate set consists of $48$ pre-randomization variables, categorized into:
(i) hospital and admission context;
(ii) demographics;
(iii) AKI timing and renal-function metrics;
(iv) laboratory measurements and vital signs at randomization;
(v) comorbidity indicators; and
(vi) recent medication or exposure history alongside cumulative alert burden.
We strictly exclude all post-randomization fields to prevent information leakage.
This encompasses the \texttt{*post24} variables, \texttt{consult14}, \texttt{aki\_documentation}, all clinical outcomes, event-time variables (\texttt{time\_to\_*}), length-of-stay (LOS) and cost summaries, provider-count summaries, and the highly leakage-prone \texttt{max\_stage} field.
A comprehensive variable inventory, detailing inclusion/exclusion statuses and missingness percentages, is provided in the supplementary material.

To ensure robust evaluation, all actual-data analyses employ repeated sample splitting rather than a single train/test partition.
Across $30$ independent replications, we generate $70/30$ training-to-test splits, stratified by both hospital and treatment arm.
Preprocessing pipelines are independently fitted on each training split.
Continuous covariates are median-imputed, standardized, and augmented with missingness indicators where applicable.
Binary and categorical covariates are mode-imputed, while hospital affiliation is encoded via six one-hot indicators.
Subsequently, all candidate learners are trained on the preprocessed training fold and evaluated on the strictly held-out test fold using the AIPW estimator, leveraging the known randomization propensity of $1/2$.

The primary reward specification utilizes the three-component hard clinical utility defined in Section~\ref{sec:actual-data-experiments}.
For completeness, the corresponding uncertainty set is formulated as
\begin{equation*}
    \mathcal{W}_\rho
    =
    \Biggl\{
        w\in\mathbb{R}_+^3:
        \sum_{j=1}^3 w_j=1,
        \,
        |w_j-w_{0j}|\le \rho \Delta_j
    \Biggr\}
    ,
\end{equation*}
$w_0=(0.60,0.25,0.15)^\top$, $\Delta=(0.10,0.05,0.05)^\top$ with $\rho=1$ serving as the baseline parameter in the main analysis.
The certified lower reward, $\uR_i=\min_{w\in\mathcal{W}_\rho}w^\top G_i$, is exactly computed via a straightforward three-dimensional linear program.
For the patient-centered sensitivity analysis detailed subsequently, we augment the utility vector to incorporate discharge-to-home status, defining
\begin{equation*}
    G_i^{(4)}
    =
    \Bigl(
        1-\texttt{death14}_i,\,
        1-\texttt{dialysis14}_i,\,
        1-\texttt{aki\_progression14}_i,\,
        \texttt{discharge\_to\_home}_i
    \Bigr)^\top
    ,
\end{equation*}
alongside the preference weights $w_0^{(4)}=(0.55,0.20,0.15,0.10)^\top$ and tolerance $\Delta^{(4)}=(0.10,0.05,0.05,0.05)^\top$.

Hyperparameter tuning for OWL, Q-learning, and RWL is conducted via three-fold cross-validation over the regularization grid $\{10^{-3},10^{-2},10^{-1},1\}$, optimized independently for both the nominal reward $R$ and the certified reward $\uR$.
PROWL and its uncertified ablation, PROWL $(U=0)$, calibrate the learning and temperature parameters over the grid $\eta,\gamma\in\{1/8,1/4,1/2,1,2,4,8\}$.
We fix the confidence level at $\delta=0.1$, the prior standard deviation at $5$, and the score bound at $3$.
The nuisance functions in PROWL are modeled using arm-specific ridge regressions.
The basis expansions include the standardized main effects, squared continuous covariates, and targeted interaction terms between care-unit indicators (hospital, ICU, ward) and key severity metrics, including SOFA score, baseline creatinine, initial eGFR, age, AKI duration, and cumulative alert burden.
While the main text reports outcomes based on the deterministic mean-rule deployment, the comparative efficacy of MAP and Gibbs sampling deployments is evaluated in Figure~\ref{app:fig:actual-prowl-deployment-sensitivity}.

\subsection{Additional Results of Actual-Data Experiments}
\label{app:subsec:actual-additional-results}

Table~\ref{app:tab:actual-outcome-decomposition} disaggregates the primary performance metrics into their constituent clinical outcomes.
Among the adaptive learning methods, PROWL achieves the lowest estimated risk for the primary composite outcome and the numerically lowest mortality risk, while remaining highly competitive regarding dialysis and AKI-progression risks.
Crucially, this demonstrates that PROWL's superior certified value is not artificially attained by trading off performance on one hard clinical component for another.
In contrast, the blanket Always alert policy performs uniformly worse across all hard-outcome components.
Furthermore, the OWL variants deploy substantially more aggressive alerting strategies without yielding any commensurate clinical benefit.

Figure~\ref{app:fig:hospital-hetero} highlights pronounced hospital-level heterogeneity in the intention-to-treat effect of the alerts.
Specifically, the two non-teaching hospitals exhibit positive risk differences for the primary composite outcome (favoring usual care), with Hospital 5 presenting a confidence interval entirely above zero and Hospital 6 displaying a directionally similar trend.
Conversely, effects within the teaching hospitals center closely around the null.
This discrepancy elucidates the strong empirical performance of the global Never alert rule and further underscores the necessity of site-aware, individualized treatment policies over blanket alerting directives.

Figure~\ref{app:fig:actual-rho-sweep} investigates the procedural sensitivity to the preference-uncertainty radius, $\rho$.
As $\rho$ increases, PROWL naturally adopts a more conservative stance: its alert rate decreases monotonically, yet the composite-free value remains remarkably stable.
The observed decline in certified value is solely a mechanical artifact of the increasingly stringent lower-utility criterion.
By contrast, the alerting behaviors of PROWL $(U=0)$ and Q-learning $(R)$ are virtually invariant to $\rho$, while Q-learning $(\uR)$ exhibits only marginal shifts.
This divergence confirms that explicit uncertainty certification fundamentally alters the learned policy, rather than merely adjusting the reported post-hoc lower bound.

Figure~\ref{app:fig:actual-sensitivity-without-hospital} presents an ablation study where hospital indicators are entirely excluded from the policy covariates.
While the absolute performance metrics slightly decrease across all methods, the qualitative ranking remains robustly preserved:
PROWL sustains its competitive edge in certified value and continues to consistently outperform both Q-learning baselines on the composite-free anchor.
This evidence strongly implies that our primary findings are not merely the result of the models memorizing site-specific labels.

Figure~\ref{app:fig:actual-prowl-deployment-sensitivity} compares the efficacy of three distinct deployment strategies: the mean-rule, the MAP rule, and the Gibbs sampling rule, all derived from the identically fitted PROWL posterior.
The three strategies are nearly indistinguishable in terms of certified value, composite-free value, and overall alert rate, though the MAP rule demonstrates a marginal numerical advantage.
Thus, our core empirical conclusions are strictly robust to the specific derandomization method chosen for practical deployment.

Finally, an auxiliary patient-centered sensitivity analysis, which incorporated discharge-to-home status into the utility definition (results not shown), reaffirmed the established qualitative rankings.
PROWL retained the highest certified value among all adaptive learners and further reduced the alert rate compared to PROWL $(U=0)$ and both Q-learning baselines.
This consistent performance underscores that the structural advantages of explicit reward certification persist even when evaluated under a broadened, more complex utility landscape.

\begin{table}[t]
\centering
\caption{
Component-wise decomposition of the repeated sample-split AIPW evaluation on ELAIA-1.
Entries are means with standard errors over $30$ repeated $70/30$ hospital-by-alert stratified splits.
This table complements Table~\ref{tab:actual-policy-comparison} by demonstrating that the improvements in certified value are not driven by deterioration in any single hard-outcome component.
}
\label{app:tab:actual-outcome-decomposition}
\setlength{\tabcolsep}{6pt}
\scalebox{0.7}{
\begin{tabular}{lrrrrr}
\toprule
Methods & death14 risk & dialysis14 risk & aki\_progression14 risk & composite\_outcome risk & discharge\_to\_home rate 
\\
\midrule
Never alert         & $0.086\,(0.001)$ & $0.031\,(0.000)$ &
$0.151\,(0.001)$    & $0.203\,(0.001)$ & $0.507\,(0.002)$ \\

Always alert        & $0.088\,(0.001)$ & $0.034\,(0.001)$ &
$0.162\,(0.002)$    & $0.217\,(0.002)$ & $0.490\,(0.003)$ \\

OWL ($R$)           & $0.089\,(0.001)$ & $0.034\,(0.001)$ &
$0.160\,(0.002)$    & $0.214\,(0.002)$ & $0.493\,(0.003)$ \\

Q-learning ($R$)    & $0.086\,(0.001)$ & $0.032\,(0.001)$ &
$0.156\,(0.002)$    & $0.207\,(0.002)$ & $0.494\,(0.002)$ \\

RWL ($R$)           & $0.088\,(0.001)$ & $0.031\,(0.001)$ &
$0.154\,(0.002)$    & $0.207\,(0.002)$ & $0.500\,(0.002)$ \\

OWL ($\uR$)         & $0.087\,(0.001)$ & $0.033\,(0.001)$ &
$0.160\,(0.002)$    & $0.213\,(0.002)$ & $0.493\,(0.003)$ \\

Q-learning ($\uR$)  & $0.086\,(0.001)$ & $0.032\,(0.001)$ &
$0.156\,(0.002)$    & $0.206\,(0.002)$ & $0.496\,(0.002)$ \\

RWL ($\uR$)         & $0.087\,(0.001)$ & $0.032\,(0.001)$ &
$0.154\,(0.002)$    & $0.206\,(0.002)$ & $0.501\,(0.003)$ \\

PROWL ($U=0$)       & $0.086\,(0.001)$ & $0.032\,(0.001)$ &
$0.154\,(0.002)$    & $0.206\,(0.002)$ & $0.499\,(0.003)$ \\

PROWL               & $0.086\,(0.001)$ & $0.032\,(0.001)$ &
$0.154\,(0.002)$    & $0.205\,(0.002)$ & $0.499\,(0.003)$ \\
\bottomrule
\end{tabular}
}
\end{table}

\begin{figure*}[tb]
\vskip 0.2in
\begin{center}
\centerline{\includegraphics[width=0.7\columnwidth]{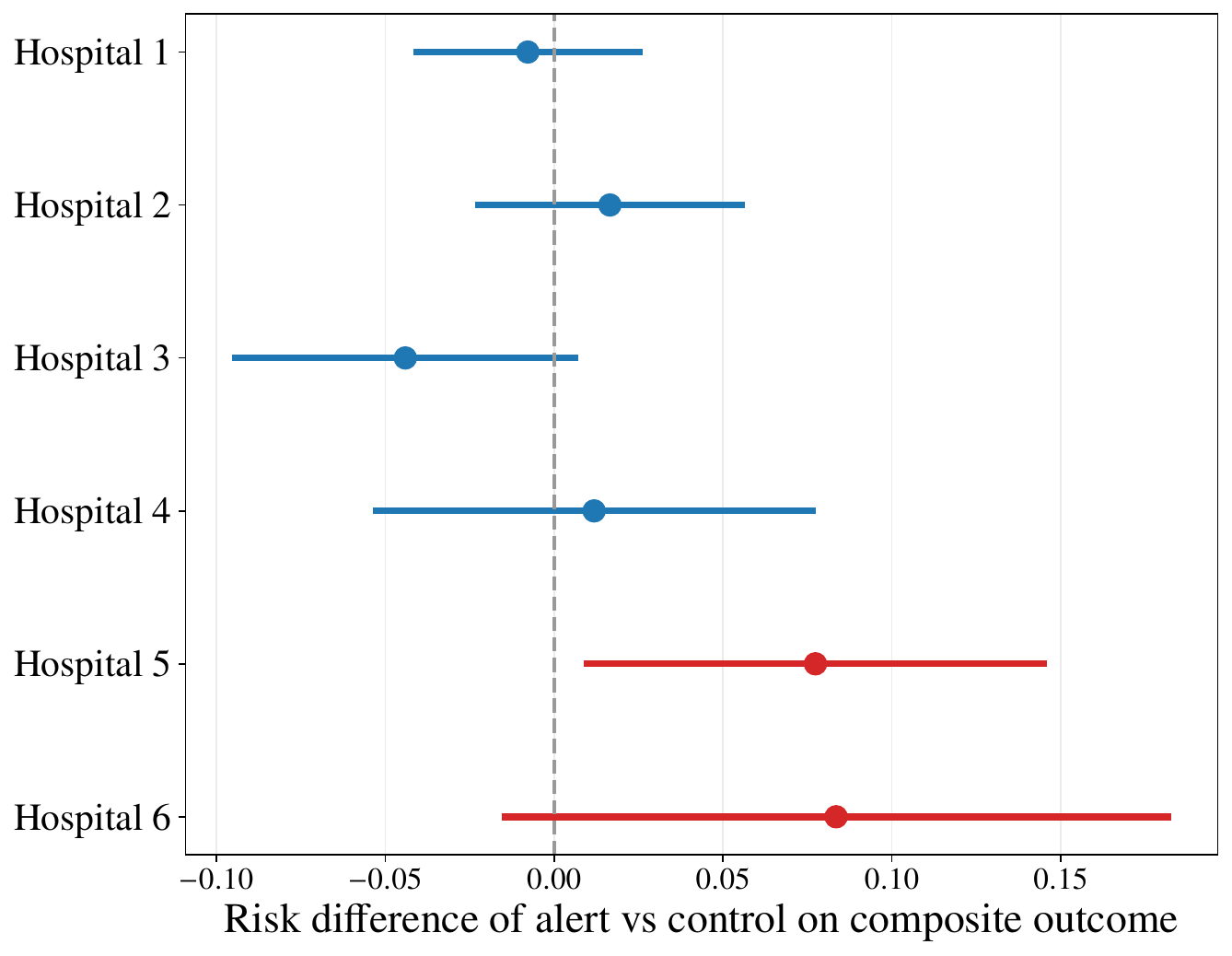}}
\caption{
Hospital-level intention-to-treat heterogeneity in the ELAIA-1 trial.
The plotted quantity is the risk difference of alert versus control on the 14-day composite outcome, accompanied by Wald-type $95\%$ confidence intervals.
Positive values favor usual care.
The two non-teaching hospitals exhibit the clearest adverse alert signal, strongly motivating the use of site-aware individualized alerting rules.
}
\label{app:fig:hospital-hetero}
\end{center}
\vskip -0.2in
\end{figure*}

\begin{figure*}[tb]
\vskip 0.2in
\begin{center}
\centerline{\includegraphics[width=\columnwidth]{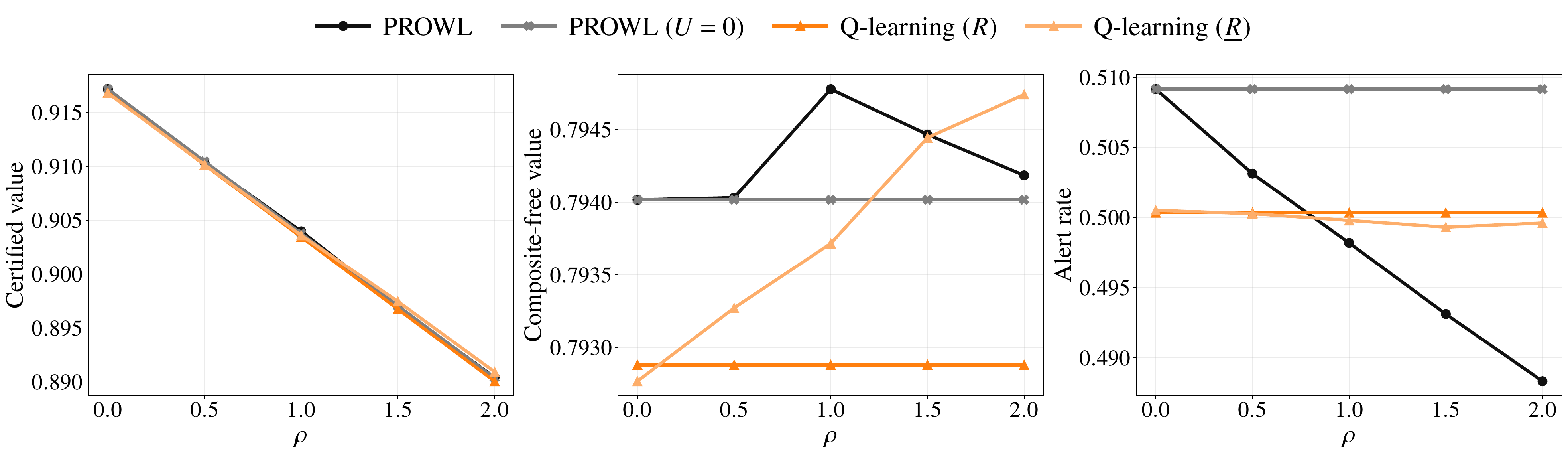}}
\caption{
Procedural sensitivity to the preference-uncertainty radius $\rho$ in the actual-data experiment.
The panels report the estimated certified value, composite-free value, and alert rate for PROWL, PROWL $(U=0)$, Q-learning $(R)$, and Q-learning $(\uR)$ across varying levels of $\rho$.
While a larger $\rho$ structurally induces a more conservative lower utility, PROWL strategically responds by modestly reducing its alerting frequency, consistently preserving the clinically anchored composite-free value.
}
\label{app:fig:actual-rho-sweep}
\end{center}
\vskip -0.2in
\end{figure*}

\begin{figure*}[tb]
\vskip 0.2in
\begin{center}
\centerline{\includegraphics[width=\columnwidth]{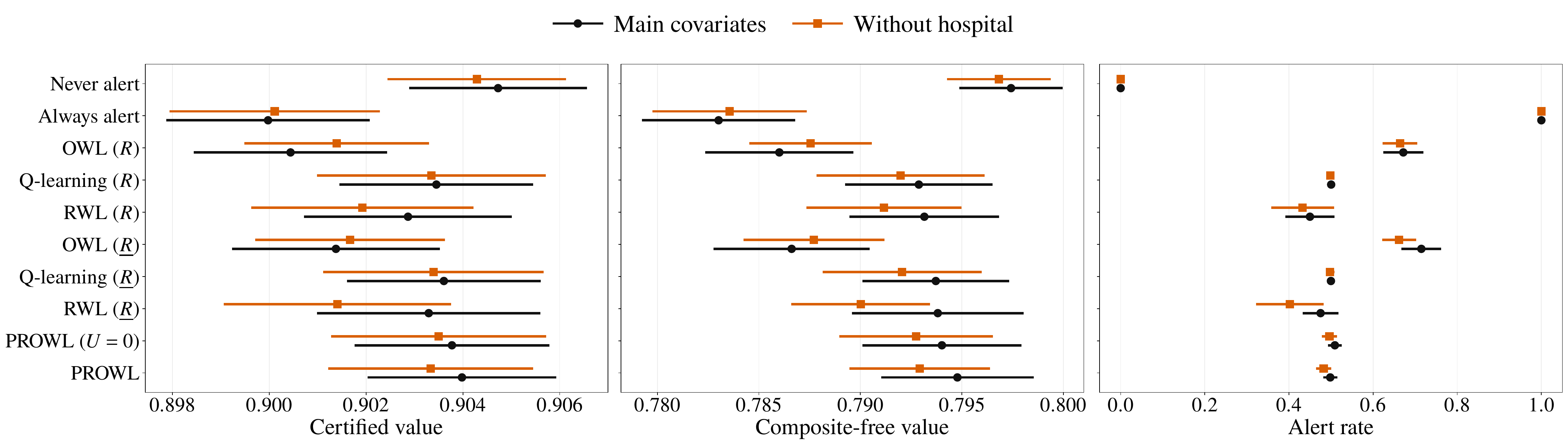}}
\caption{
Ablation sensitivity analysis excluding hospital covariates.
For each method, the points and intervals compare the baseline analysis against a refitted model that completely removes hospital indicators from the policy covariates.
Although absolute performance experiences a slight drop across all methods, the relative ordering remains robustly unchanged. This confirms that the observed gains of PROWL are not driven merely by the memorization of site-specific labels.
}
\label{app:fig:actual-sensitivity-without-hospital}
\end{center}
\vskip -0.2in
\end{figure*}

\begin{figure*}[tb]
\vskip 0.2in
\begin{center}
\centerline{\includegraphics[width=\columnwidth]{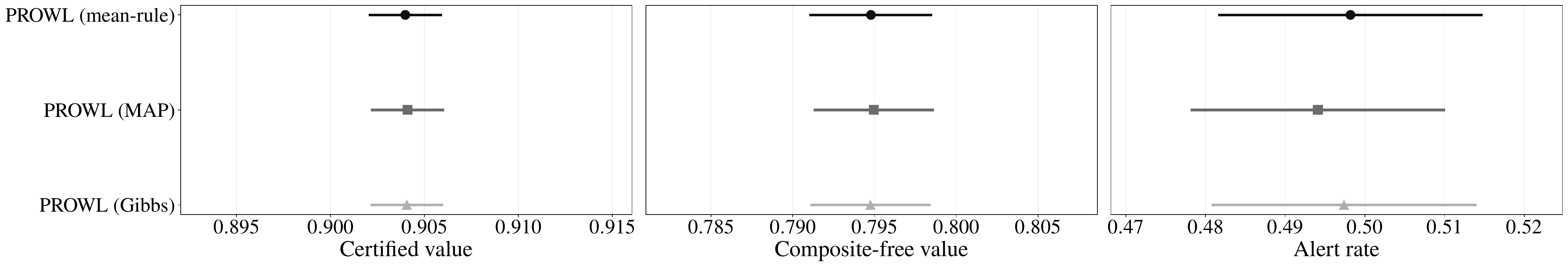}}
\caption{
Deployment derandomization sensitivity for PROWL.
This figure compares the deterministic mean-rule (utilized in the main text analysis) against MAP and Gibbs deployment strategies derived from the identically fitted posterior.
Across all three configurations, the certified value, composite-free value, and alert rates are virtually indistinguishable, underscoring the robustness of our empirical findings.
}
\label{app:fig:actual-prowl-deployment-sensitivity}
\end{center}
\vskip -0.2in
\end{figure*}

\end{document}